\documentclass[11pt,letterpaper]{article}

\usepackage[utf8]{inputenc}
\usepackage[T1]{fontenc}
\usepackage{amsmath,amssymb,amsthm}
\usepackage{xcolor}
\usepackage{xspace}
\usepackage[most]{tcolorbox}
\usepackage{deepthink}
\usepackage[nameinlink]{cleveref}
\usepackage{subcaption}
\usepackage{nicefrac}
\usepackage{multirow}
\usepackage{comment}
\usepackage{algorithm}
\usepackage{algpseudocode}
\usepackage[
  backend=biber,
  style=alphabetic,
  maxbibnames=100,
  minbibnames=100,
  maxcitenames=2,
  mincitenames=2
]{biblatex}
\newcommand{\pref}{\ensuremath{p_{\mathrm{ref}}}\xspace}

\newcommand{\corrauth}{\textsuperscript{\ddag}}

\title{MCLR: Improving Conditional Modeling via Inter-Class Likelihood-Ratio Maximization and Unifying Classifier-Free Guidance with Alignment Objectives}

\authorblock{
  {\textbf{Xiang Li}}\textsuperscript{1,}\corrauth,
  {\textbf{Yixuan Jia}}\textsuperscript{1},
  {\textbf{Xiao Li}}\textsuperscript{1},
  {\textbf{Jeffrey A. Fessler}}\textsuperscript{1},
  {\textbf{Rongrong Wang}}\textsuperscript{2},
  {\textbf{Qing Qu}}\textsuperscript{1}
}

\affiliation{
  University of Michigan\textsuperscript{1}\quad $\cdot$ \quad
  Michigan State University\textsuperscript{2}
}

\authornote{\corrauth\ Corresponding author}

\abstracttext{
\noindent Diffusion models achieve strong generative performance but often rely on classifier-free guidance (CFG), an inference-time intervention that modifies the sampling trajectory. Since standard denoising score matching (DSM) should recover the target distribution when the score is learned accurately, the practical necessity of CFG suggests deficiencies in learned conditional scores. We identify weak inter-class separation as one such limitation: samples generated from the same noise under different labels can retain similar class-agnostic structure. To address this deficiency, we introduce \textbf{MCLR}, a principled fine-tuning objective that explicitly \textbf{m}aximizes inter-\textbf{c}lass \textbf{l}ikelihood-\textbf{r}atios, thereby strengthening class-specific structures under standard sampling. Across diverse models and datasets, MCLR substantially improves guidance-free conditional generation and narrows the gap to CFG. Theoretically, we prove that the CFG-guided score is the unique minimizer of a sample-adaptively weighted, ELBO-approximated MCLR objective, interpreting CFG as an implicit inference-time contrastive alignment procedure. We further compare MCLR with related contrastive objectives, including DPO, CCA, and CFM, clarifying their connections and key differences in an otherwise fragmented literature.
}


\date{Originally released on arXiv in May 2026; revised in October 2026}
\correspondence{\href{mailto:forkobe@umich.edu}{forkobe@umich.edu}}
\resources{\href{https://github.com/Morefre/MCLR-Contrastive-Alignment-for-Improving-Conditional-Modeling}{Code Repository}}
\headerlogo{Deepthink_landscape_do_not_delete_compressed_ss.png}{https://deepthink-umich.github.io}

\begin{document}

\makeDeepthinkHeader

\begin{figure}[!ht]
    \centering
    \includegraphics[width=0.5\linewidth]{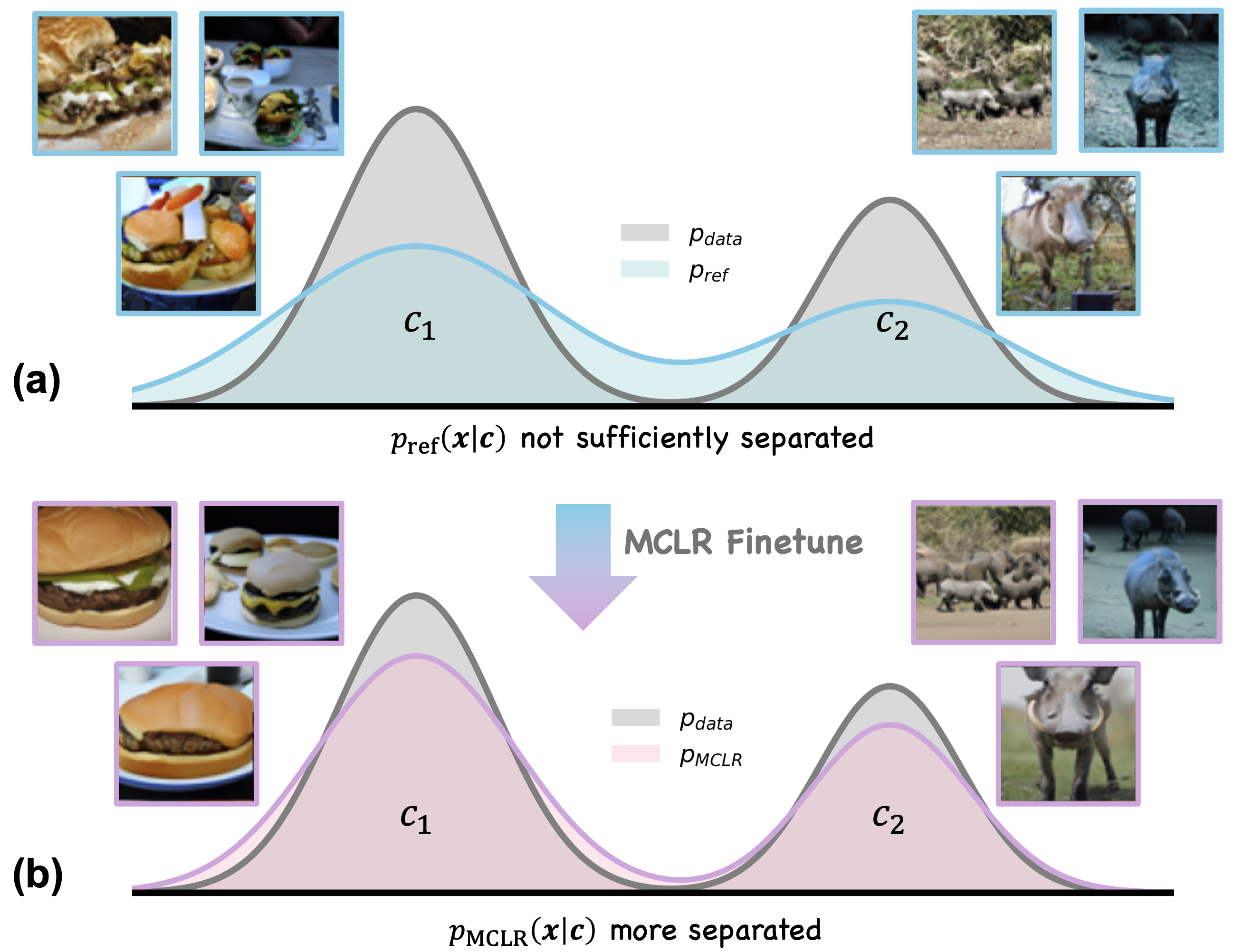}
    \captionsetup{font=small}
    \caption{\textbf{Conceptual Illustration of MCLR.}
    (a) Samples generated from two classes using the same initial noises exhibit high visual similarity despite different conditioning labels, indicating insufficient separation of the learned conditional distributions. (b) MCLR mitigates this issue by encouraging inter-class separation, resulting in generations with more distinct class-specific features.}
    \label{fig:conceptual figure}
\end{figure}

\clearpage
\tableofcontents

\section{Introduction}
Diffusion models~\cite{ho2020denoising,songscore,karras2022elucidating,lipmanflow} have become the dominant paradigm for high-fidelity generative modeling. These models generate samples by reversing a forward noising process using learned score functions, typically trained via denoising score matching~\cite{6795935}.
Although the reverse sampling process is theoretically guaranteed to recover the target distribution~\cite{anderson1982reverse} when the score functions are learned accurately, in practice, it often yields samples of noticeably inferior quality: conditional generation frequently appears visually incoherent or insufficiently faithful to the intended class or prompt~\cite{bradley2024classifier}. In fact, nearly all high-quality generation by diffusion models relies heavily on classifier-free guidance (CFG)~\cite{ho2022classifier}, an inference-time modification of the reverse sampling process that injects an additional guidance term. While CFG substantially improves sample quality, reducing FID scores by up to $75\%$ as shown in seminal works~\cite{peebles2023scalable, yurepresentation}, its empirical necessity exposes a gap between the theoretical optimality of DSM and its practical behavior. This raises a central question:
\begin{center}
\begin{tcolorbox}[
  colback=gray!18,
  colframe=black,
  arc=4mm,
  boxrule=0.8pt,
  width=0.92\linewidth
]
\centering
\vspace{0.3em}

\emph{\textbf{What deficiency in DSM-trained conditional score functions does CFG compensate for, and can this deficiency be addressed through a principled training objective to enable high-quality conditional generation under standard sampling?}}

\vspace{0.3em}
\end{tcolorbox}
\end{center}

Recent study~\cite{li2026towards} suggests that standard conditional models suffer from insufficient inter-class separation: generated samples are less distinguishable across classes than real data, indicating that class-dependent structures are not fully captured by diffusion models. Motivated by this insight, we propose \textbf{MCLR}, a principled alignment objective that explicitly prompts inter-class separability by \textbf{M}aximizing the inter-\textbf{C}lass log-\textbf{L}ikelihood \textbf{R}atio. By encouraging the model to amplify density differences between a target class and other classes, MCLR strengthens class-specific structures in the learned score function.
Empirically, models fine-tuned with MCLR exhibit CFG-like improvements under standard sampling, substantially improving guidance-free conditional generation and narrowing the gap to inference-time CFG.


Beyond the empirical benefits, we provide a theoretical result showing that the CFG-guided score coincides exactly with the optimizer of a weighted MCLR objective. This establishes a formal equivalence between CFG and alignment-based training objectives, revealing CFG as an implicit inference-time contrastive alignment algorithm.

The core principle of MCLR is to leverage inter-class contrastive signals to improve conditional modeling, which can also be realized with broader contrastive-based approaches. To assess the uniqueness of MCLR and for completeness, we analyze two alternative objectives: a conditional adaptation of DPO~\cite{rafailov2023direct}, termed Conditional Contrastive DPO (CC-DPO), and Contrastive Flow Matching (CFM)~\cite{stoica2025contrastive}. We show that CC-DPO shares the gamma-powered population optimum of CCA~\cite{chentoward}, whereas CFM induces a global mean-centered correction rather than MCLR's locally density-adaptive correction. Experiments show that MCLR consistently outperforms these training-time contrastive alternatives across diverse settings.

\paragraph{Summary of Contributions.} Our main contributions are as follows:

\begin{itemize}[leftmargin=*]
\item \textbf{A Principled Alignment Objective for Conditional Modeling.}
We propose \textbf{MCLR}, a theoretically grounded fine-tuning objective that explicitly maximizes inter-class log-likelihood ratios to improve conditional generative modeling. Across diverse models and datasets, MCLR achieves substantial improvements in sample fidelity under standard reverse sampling, consistently outperforming training-time contrastive alternatives such as CC-DPO, CCA and CFM.
\item \textbf{Theoretical Equivalence between CFG and Alignment Objectives.}
We prove that the classifier-free guidance (CFG)–induced score coincides exactly with the optimizer of a weighted MCLR objective. This establishes a formal connection between CFG and alignment-based training, providing a mechanistic interpretation of CFG as an implicit inference-time alignment algorithm.

\item \textbf{Understanding Contrastive Alternatives.}
We show that CC-DPO and CCA share the same gamma-powered population optimum, whereas CFM rescales the conditional score and introduces a global mean-centered correction rather than MCLR's locally density-adaptive correction.
\end{itemize}





\section{Preliminaries}
\newcommand{\der}{\mathrm{d}} 
\subsection{Basics of Diffusion Models}
\label{subsection: basiscs of diffusion models}
Let $p_\text{data}(\mb x)$ denote the ground-truth data distribution. Diffusion models construct a forward noising process that gradually perturbs $p_\text{data}$ into a simple prior distribution using a stochastic differential equation (SDE):
\begin{align}
    \label{forward process}
    \der {\mb x} = \mb f(\mb x,t) \,\der t+g(t) \,\der {\mb w},
\end{align}
where $\mb f(\cdot,t)$ is the drift coefficient, $g(t)$ is the diffusion coefficient, and $\mb w$ denotes the standard Brownian motion. Let $p_t(\mb x)$ be the marginal distribution of $\mb x(t)$, and $p_{0t}(\mb x_t|\mb x)$ the transition density from $\mb x(0)$ to $\mb x(t)$. For sufficiently large $T$, $p_T(\mb x)$ becomes indistinguishable from a tractable prior $\pi(\mb x)$, e.g., an isotropic Gaussian. The SDE~\eqref{forward process} admits a reverse-time probability-flow ODE~\cite{songscore}:
\begin{align}
    \der \mb x &= \bigl[\mb f(\mb x,t)-\frac{1}{2}g^2(t) \nabla_{\mb x}\log p_t(\mb x)\bigr]
    \der t.\label{reverse ODE}
\end{align}
Sampling from the reverse ODE requires the score function $\nabla_{\mb x}\log p_t(\mb x)$, which can be approximated using a deep network $s_{\mb\theta}(\mb x,t)$ trained via the denoising score matching (DSM) objective: 
\begin{align}
    \mathcal{J}_{\text{DSM}}(\mb\theta;w(\cdot))
    &:= \frac{1}{2}\int_{0}^{T}
    \mathbb{E}_{p(\mb x),\,p_{0t}(\mb x_t| \mb x)}
    \Big[w(t)
    \big\|
    \nabla_{\mb x_t}\log p_{0t}(\mb x_t| \mb x)
    -\mb s_{\mb\theta}(\mb x_t,t)
    \big\|_2^2
    \Big]\; \der t,
\label{denoising score matching}
\end{align}
where $w(t)$ is a positive weighting function. For conditional diffusion models, the score network takes a conditional embedding $\mb c$ as input, and the DSM objective naturally extends to the conditional setting by taking the expectation over class labels and class-conditional data distributions: 
\begin{align}
\mathcal{J}_{\text{DSM}}(\mb\theta,\mb c;w(\cdot)):=\frac{1}{2}\int_{0}^T\mathbb{E}_{\mb c, p(\mb x|\mb c),p_{0t}(\mb x_t|\mb x)}[w(t)\|\nabla_{\mb x_t}\log p_{0t}(\mb x_t|\mb x)-\mb s_{\mb\theta}(\mb x_t,t,\mb c)\|_2^2]\der t.
\end{align}
Accordingly, the conditional reverse ODE is:
\begin{align}
    \der \mb x = \bigl[ f(\mb x,t)-\frac{1}{2}g^2(t) \nabla_{\mb x}\log p_t(\mb x|\mb c)\bigr] \der t.
    \label{reverse conditional ODE}
\end{align}

\subsection{Evidence Lower Bound for Diffusion Models}
\label{ELBO for diffusion models}
Let $p_{\mb\theta}^\text{ode}(\mb x)$ denote the distributions induced by the reverse ODE~\eqref{reverse ODE}. Theorem 2 of~\cite{song2021maximum} shows that, under certain regularity conditions, the log-likelihood satisfies:
\begin{align}
    \underbrace{\mathbb{E}_{p(\mb x)}[\log p_{\mb\theta}^\text{ode}(\mb x)]}
    _{\text{\makebox[0pt]{Maximum Likelihood Estimation}}}
    = -\mathcal{J}_\text{DSM}(\mb\theta;g^2(\cdot)) + C, 
    \label{MLE equals DSM}
\end{align}
where $C$ is a constant independent of $\mb\theta$. \Cref{MLE equals DSM} is analogous to the evidence lower bound (ELBO) in variational autoencoders, revealing a fundamental connection between DSM and maximum likelihood estimation (MLE), which enables likelihood-based training in diffusion models~\cite{mardanivariational,wallace2024diffusion, zheng2025direct}.



\subsection{Classifier-Free Guidance}
Although the reverse ODE in~\eqref{reverse conditional ODE} is theoretically guaranteed to sample from the target conditional distribution, its practical generation quality is often unsatisfactory. In practice, high-quality conditional generation requires modifying the standard reverse process with an additional guidance term, known as classifier-free guidance (CFG)~\cite{ho2022classifier}, which leads to the perturbed reverse ODE:
\begin{align}
    \der \mb x = \bigl[\mb f(\mb x,t)-\frac{1}{2}g^2(t) (\nabla_{\mb x}\log p_t(\mb x|\mb c) 
    +\gamma (\underbrace{\nabla_{\mb x}\log p_t(\mb x|\mb c)-\nabla_{\mb x}\log p_t(\mb x)}_{\text{CFG guidance}})\bigr] \der t
    \label{CFG+ODE},
\end{align}
where $\gamma$ controls the guidance strength. Intuitively, CFG sharpens the class- or conditional-specific structure by amplifying the difference between conditional and unconditional scores. However, the CFG-perturbed reverse sampling process~\eqref{CFG+ODE} does not, in general, correspond to any known forward process~\cite{bradley2024classifier}. Despite recent progress towards understanding CFG~\cite{wutheoretical,chidambaramdoes,bradley2024classifier,pavasovic2025classifier,li2026towards,liprovable,jin2025stage,yangelucidating,ventura2026emergence}, its underlying mechanisms remain only partially understood. This post-hoc modification of the sampling procedure motivates the search for principled alternatives that reproduce CFG-like effects while preserving theoretical consistency, which may, in turn, shed light on the mechanisms of CFG itself.

\subsection{Direct Preference Optimization}
\label{subsec:dpo-prelim} Direct Preference Optimization (DPO) aligns a reference (base) model $p_\text{ref}$ by increasing the relative likelihood of preferred samples over non-preferred samples~\cite{rafailov2023direct,wallace2024diffusion}. 
Under the Bradley--Terry model~\cite{bradley1952rank},
\begin{align}
    p(\mb x_w\succ \mb x_l|\mb c)
    =
    \mathrm{Sigmoid}
    \bigl(
    r(\mb x_w|\mb c)-r(\mb x_l|\mb c)
    \bigr),
    \label{BT-Model}
\end{align}
where $\mb x_w$ and $\mb x_l$ denote preferred and non-preferred samples, and $\mathrm{Sigmoid}(x):=1/(1+\exp(-x))$. 
DPO parameterizes the reward as
$r_{\mb\theta}(\mb x|\mb c)
=
\beta\log p_{\mb\theta}(\mb x|\mb c)
-
\beta\log p_\text{ref}(\mb x|\mb c)$,
which leads to the objective
\begin{align}
    \max_{\mb\theta}\;
    \mathbb{E}_{(\mb c,\mb x_w,\mb x_l)\sim S}
    \left[
    \log\mathrm{Sigmoid}
    \left(
    \beta \log\frac{p_{\mb\theta}(\mb x_w|\mb c)}{p_\text{ref}(\mb x_w|\mb c)}
    -
    \beta \log\frac{p_{\mb\theta}(\mb x_l|\mb c)}{p_\text{ref}(\mb x_l|\mb c)}
    \right)
    \right],
    \label{DPO objective main text}
\end{align}
where $p_\text{ref}$ denotes the pretrained base (reference) model, $\beta$ controls the strength of the KL regularization between $p_{\mb\theta}$ and $p_\text{ref}$,  and $S$ represents the preference dataset (see a more detailed introduction in~\Cref{basics of dpo appendix}). We later adapt this preference-style objective to conditional generation by treating target-class samples as preferred and mismatched-class samples as non-preferred.

\section{Motivation and Method}
This section first identifies weak class specificity as a practical limitation of
standard DSM training (\cref{subsec:lack of class-specificity}). To remedy this issue, we introduce MCLR (\Cref{sec: maximum likelihood-ratio training,subsec: theoretical analysis on MCLR,Approximating log-likelihood with ELBO and the caveat}), and establish its connection to CFG (\Cref{theoretical equivalence of CFG and Alignment Algorithm main text}).
\subsection{A Practical Limitation of DSM: Weak Class Specificity}
\label{subsec:lack of class-specificity}
In theory, if the score functions are learned accurately, standard reverse diffusion sampling should produce samples from the target conditional distribution. In practice, however, conditional sampling often exhibits weak class-specific structure~\cite{li2026towards}. A common failure mode is that samples generated from the same initial noise
but different class conditions share similar global layouts,
while class-discriminative features are attenuated or missing,
as shown in~\Cref{fig:conceptual figure}(a)
and~\Cref{fig:imagenet512 qualitative teaser}. This suggests that the learned conditional distributions are insufficiently distinguishable from one another. We provide additional quantitative evidence based on pairwise
inter-class FID, together with controlled qualitative comparisons, in
Appendix~\ref{app:weak-inter-class-separation}.
\Cref{fig:conceptual figure} illustrates this conceptually: the learned class-conditional distributions of a base model $p_\text{ref}$ exhibit substantially less separation than the ground-truth data distribution $p_{\text{data}}$. This motivates modifying the standard DSM to more explicitly exploit
the information carried by the conditioning variable.
\begin{figure*}[t]
    \centering    \includegraphics[width=0.99\linewidth]{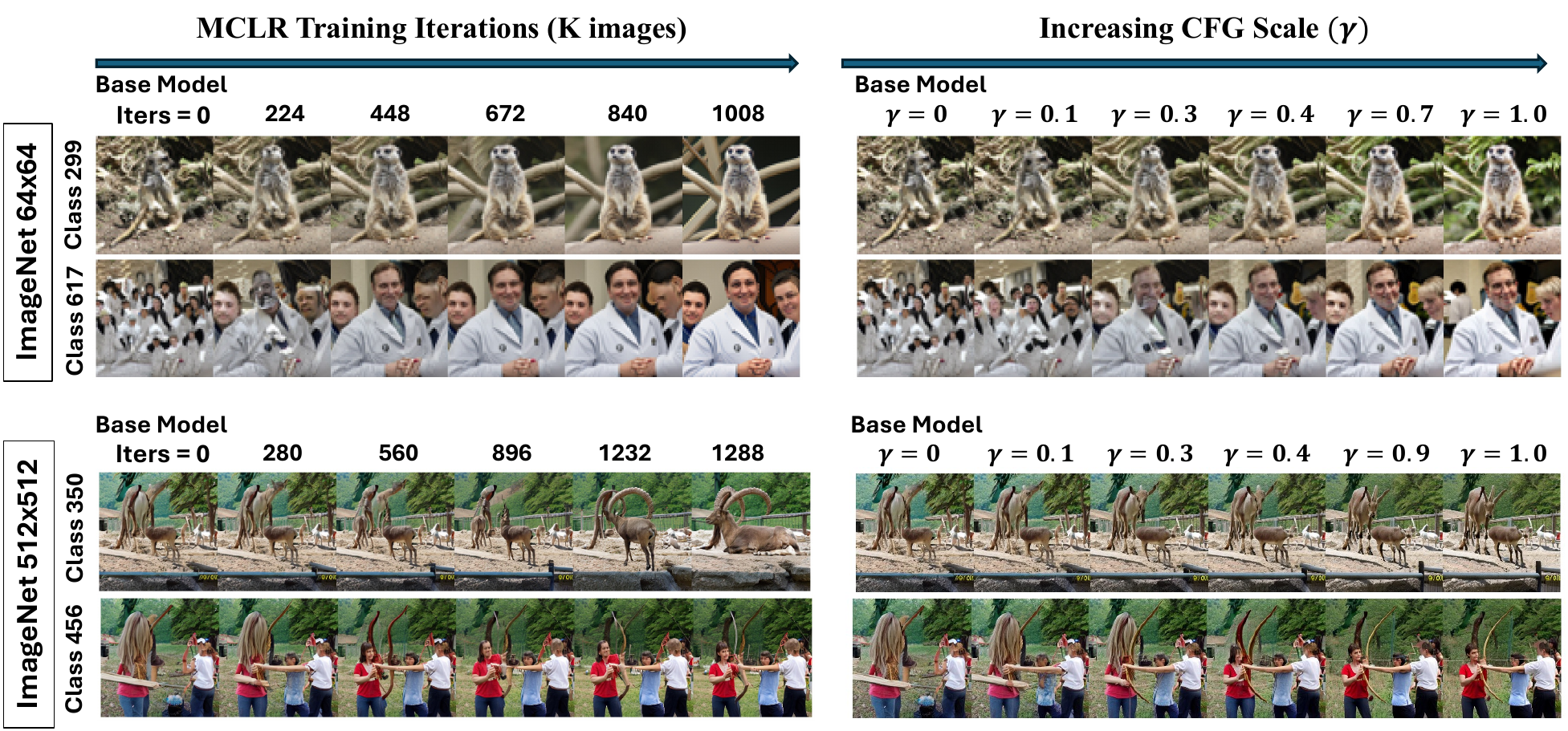}
   \caption{\textbf{Effects of MCLR vs.~CFG.}
We visualize the progressive emergence of class-specific structure under MCLR training and CFG. MCLR training induces effects analogous to increasing guidance strength in CFG, both substantially improving class-specific patterns.
For each image block, samples from different classes are generated from the same initial noises.}

\label{fig:imagenet512 qualitative teaser}
\end{figure*}
\subsection{Maximum Inter-Class Likelihood-Ratio Training}

\label{sec: maximum likelihood-ratio training}
Motivated by the preceding discussion, we propose improving conditional modeling by explicitly encouraging inter-class separation in conditional models, as shown conceptually in~\Cref{fig:conceptual figure}(b). Let $p_{\mb\theta}(\mb x|\mb c)$ denote the model's conditional distribution for class condition $\mb c$, where $\mb\theta$ is the model parameter. We consider the following objective:
\begin{align}
    \max_{\mb\theta}\mathbb{E}_{\mb c, p(\mb x|\mb c)}\log p_{\mb\theta}(\mb x|\mb c)
    +\frac{\eta}{2}\,\mathbb{E}_{\mb c, \tilde{\mb c},\mb x\sim p(\cdot|\mb c),\mb y\sim p(\mb \cdot|\tilde{\mb c})}
    \left[\log \frac{p_{\mb\theta}(\mb x|\mb c)}{p_{\mb\theta}(\mb x|\tilde{\mb c})}
    +\log \frac{p_{\mb\theta}(\mb y|\tilde{\mb c})}{p_{\mb\theta}(\mb y| \mb c)}\right],
    \label{MLE with Likelihood ratio constraint symmetric}
\end{align}
where $\mb c$ and $\tilde{\mb c}$ are two randomly sampled classes, and $\eta$ controls the regularization strength. Compared to standard MLE (equivalently, DSM under appropriate weighting) in~\eqref{MLE equals DSM}, eq.~\eqref{MLE with Likelihood ratio constraint symmetric} introduces an additional regularizer, which we refer to as \textbf{M}aximum Inter-\textbf{C}lass \textbf{L}ikelihood-\textbf{R}atio (MCLR).


Specifically, given a sample $\mb x$ (or $\mb y$) drawn from class $\mb c$ (or $\tilde{\mb c}$), MCLR explicitly encourages its corresponding log-likelihood under the true class $\log p_{\mb\theta}(\mb x|\mb c)$ (or $\log p_{\mb\theta}(\mb y|\tilde{\mb c})$) to be higher relative to its log-likelihood $\log p_{\mb\theta}(\mb x|\tilde{\mb c})$ (or $\log p_{\mb\theta}(\mb y|\mb c)$) under a mismatched class $\tilde{\mb c}$ (or $\mb c$). In doing so, MCLR drives the model to increase the inter-class likelihood ratio, thereby pushing $p_{\mb\theta}(\mb x|\mb c)$ to concentrate more probability mass in regions where the true class is favored over competing classes. Such regions typically correspond to samples with more pronounced class-specific features. 

Intuitively, MCLR encourages the model to fully exploit the label-conditioned information. Standard conditional training simply feeds the conditional label $\mb c$ alongside the data to the model, optimizes the MLE (or DSM) objective, and relies on deep networks to automatically discover the conditional structure. If labels are under-exploited, the conditional distributions can collapse and become weakly distinguishable, yielding $p_{\mb\theta}(\mb x|\mb c)\approx p_{\mb\theta}(\mb x|\tilde{\mb c})$ regardless of the true class of $\mb x$. Maximizing log-likelihood ratio $\log \frac{p_{\mb\theta}(\mb x|\mb c)}{p_{\mb\theta}(\mb x|\tilde{\mb c})}$ penalizes this collapse and encourages the model to discriminate between different conditions (classes) by leveraging the information encoded in the labels. 

Applying the linearity of expectation to~\eqref{MLE with Likelihood ratio constraint symmetric} and simplifying leads to the following equivalent form that we use throughout, unless otherwise stated: 
\begin{align}
\label{MLE with Likelihood ratio constraint}
    \max_{\mb\theta}\;
    \mathbb{E}_{\mb c,\,p(\mb x\mid\mb c)}
    \big[\log p_{\mb\theta}(\mb x|\mb c)\big]+\eta\,
    \underbrace{
    \mathbb{E}_{\mb c,\, \tilde{\mb c},\,\mb x\sim p(\cdot|\mb c),\mb y\sim p(\cdot|\tilde{\mb c})}
    \Big[
    \log \frac{p_{\mb\theta}(\mb x|\mb c)}{p_{\mb\theta}(\mb y|\mb c)}
    \Big]
    }_{\text{MCLR Regularization}}.
\end{align} 
\paragraph{Fine-tuning with MCLR.} When a pretrained (suboptimal) model $p_{\text{ref}}(\mb x)$ that lacks class specificity is available, we can fine-tune it using MCLR in combination with KL regularization:
\begin{align}
    \max_{\mb\theta}-\mathbb{E}_{\mb c}\left[D_\text{KL}\bigl(p_\text{ref}(\mb x|\mb c)\|p_{\mb\theta}(\mb x|\mb c)\bigl)\right] +\eta\,
    \mathbb{E}_{\mb c,\,\tilde{\mb c},\mb x\sim p(\cdot|\mb c),\mb y\sim p(\cdot|\tilde{\mb c})}
    \Big[
    \log \frac{p_{\mb\theta}(\mb x|\mb c)}{p_{\mb\theta}(\mb y|\mb c)}
    \Big].
    \label{KL+MCLR objective main text}
\end{align}
We next analyze the optimal solution of~\eqref{MLE with Likelihood ratio constraint} and~\eqref{KL+MCLR objective main text}, providing insight into MCLR's effects.

\subsection{Theoretical Analysis of MCLR} 
\label{subsec: theoretical analysis on MCLR}
We now analyze the optimal distribution induced by MCLR. Define:
\begin{align}
    h(\mb x|\mb c):=p(\mb x|\mb c)+\eta\, \bigl(p(\mb x|\mb c)-p(\mb x)\bigr).
    \label{sum of difference distribution}
\end{align}


\begin{theorem}
\label{MCLR theorem}
Suppose $h(\mb x|\mb c)$ is supported on a bounded set $K$, and restrict the model class to densities satisfying 
$p_{\mb\theta}(\mb x|\mb c)\geq \delta>0$ on $K$. These conditions are mainly technical: bounded support is natural for image data with bounded pixel values, and the positive density floor ensures that the log-likelihood objective is well-defined. 
Then the optimal solution to~\eqref{MLE with Likelihood ratio constraint} is:    
\begin{align}
    p_{\mb\theta^*}(\mb x|\mb c)\;=\;
    \begin{cases}
        \max \bigl\{\frac{h(\mb x|\mb c)}{Z(\mb c)},\delta\bigl\}, \;\mb x\in K, \\
        0, \;\mb x\notin K,
    \end{cases}
    \label{main text optimal solution under MCLR}
\end{align}
where $Z(\mb c)$ is the normalizing constant.
\end{theorem}
The formal assumptions and proof are provided in~\Cref{sec: proof for MCLR main}. Intuitively, the optimal conditional distribution $p_{\mb\theta^*}(\mb x|\mb c)$ induced by MCLR is proportional to $h(\mb x|\mb c)$ whenever $h(\mb x|\mb c)>Z(\mb c)\delta$, while being clipped to the floor $\delta$ elsewhere. In the limit $\delta\rightarrow 0$, the floor disappears and the optimal distribution approaches:
\begin{align}
    p_{\mb\theta^*}(\mb x|\mb c)=\frac{h^+(\mb x|\mb c)}{\int_{\mb x}h^+(\mb x|\mb c)d\mb x},
    \label{optimal MCLR solution in the limit}
\end{align}
where $h^+(\mb x|\mb c):=\max\bigl\{h(\mb x|\mb c),0\bigl\}$, such that the negative part of $h(\mb x|\mb c)$ is truncated to zero and then normalized to form a valid distribution.
One can interpret~\eqref{main text optimal solution under MCLR} and~\eqref{optimal MCLR solution in the limit} as the "sum-of-difference" distribution: MCLR reshapes $p(\mb x|\mb c)$ according to the difference between $p(\mb x|\mb c)$ and $p(\mb x)$, such that the regions where $p(\mb x|\mb c)>p(\mb x)$ (i.e., samples with strong class-specific features) are amplified, while regions where $p(\mb x|\mb c)<p(\mb x)$ (i.e., ambiguous samples lying near class boundaries) are suppressed.

For the fine-tuning objective in~\eqref{KL+MCLR objective main text}, it is easy to show that the optimal solution admits the same form as~\eqref{main text optimal solution under MCLR}, but in this case $h(\mb x|\mb c)= \pref(\mb x|\mb c)+\eta\, (p(\mb x|\mb c)-p(\mb x))$. This directly leads to the following corollary (\Cref{sec: proof for fineutning} provides the proof).
\begin{cor}
    If the base model \pref satisfies the mixture error model:
\begin{align}
\pref(\mb x|\mb c)=(1-\eta)\,p(\mb x|\mb c)+\eta \,p(\mb x),
\label{mixture error model assumption}
\end{align}
where $\eta\in[0,1]$, then fine-tuning $\pref(\mb x|\mb c)$ with MCLR regularization~\eqref{KL+MCLR objective main text} recovers the ground-truth conditional distribution $p(\mb x|\mb c)$.
\label{corollary MCLR finetuning}
\end{cor}
The mixture error model~\eqref{mixture error model assumption} posits that the base model suffers from cross-class "leakage"; specifically, the learned conditional density $\pref(\mb x|\mb c)$ is a convex combination of the ground-truth conditional distribution $p(\mb x|\mb c)$ and the unconditional counterpart $p(\mb x)$, which corresponds to a weighted average over all class-conditional distributions. MCLR counteracts the leakage by adding the contrastive density difference $p(\mb x|\mb c)-p(\mb x)$ to the base model, thereby suppressing regions dominated by competing classes and strengthening class-specific regions. 
Although the error structure of practical diffusion models is likely more complex than this simple mixture model, this model provides a useful intuition for why MCLR can mitigate insufficient class separation.

\subsection{Diffusion-Based MCLR: Approximating Log-Likelihood with ELBO}
\label{Approximating log-likelihood with ELBO and the caveat}
The objective above is written in terms of  log-likelihoods.
This is directly applicable to autoregressive models~\cite{tian2024visual},
but exact likelihood computation for diffusion models is too expensive for
training.
We therefore approximate log-likelihood differences using the diffusion
ELBO~\eqref{MLE equals DSM}.
This yields an ELBO-approximated version of MCLR~\eqref{the full elbo-approximated MCLR objective in appendix score-based form}, whose full derivation is
given in~\Cref{sec: ELBO denoiser form}.
\paragraph{Denoising Interpretation of MCLR.} According to the equivalence between score function and optimal MMSE denoiser (see~\Cref{sec: ELBO denoiser form} for details), the score function can be parameterized as: $\mb s_{\mb\theta}(\mb x,t,\mb c)=\frac{\mathcal{D}_{\mb\theta}(\mb x;\sigma(t),\mb c)-\mb x}{\sigma^2(t)}$, where $\sigma(t)$ is the standard deviation of additive noise at time $t$. With a customized noise-time distribution $p(t)$ and weighting function $w(t)$, the MCLR regularization becomes
\begin{align}
    \label{denoiser perspective}
    \mathbb{E}_{\substack{
    \mb c,\tilde{\mb c},\,t\sim p(t)\\
    \mb x\sim p(\cdot|\mb c),\,\mb x_t\sim p_{0t}(\cdot|\mb x)\\
    \mb y\sim p(\cdot|\tilde{\mb c}),\,\mb y_t\sim p_{0t}(\cdot|\mb y)
    }}
   \Bigl[
     w(t)\Bigl(
       \|\mb x-\mathcal{D}_{\mb\theta}(\mb x_t;\sigma(t),\mb c)\|_2^2
       -
       \|\mb y-\mathcal{D}_{\mb\theta}(\mb y_t;\sigma(t),\mb c)\|_2^2
     \Bigr)
   \Bigr].
\end{align}
Eq.~\eqref{denoiser perspective} provides a denoising-perspective interpretation of MCLR: it encourages the target-condition denoiser $\mathcal{D}(\cdot;\sigma(t),\mb c)$ to achieve lower reconstruction error on samples of the target condition $\mb c$ than on samples from a mismatched condition $\tilde{\mb c}$. We use this formulation in practical implementation.

\paragraph{Population-Level Analysis of Diffusion-Based MCLR.} Although likelihood-based objectives are commonly adapted to score-based diffusion models through ELBO approximations~\cite{wallace2024diffusion,zheng2025direct,xue2025advantage}, the resulting diffusion objective need not preserve the likelihood-level optimum. In particular, the population-optimal score of diffusion-based MCLR does not generally coincide with the score associated with the forward-noised version of the likelihood-based solution in~\eqref{main text optimal solution under MCLR}. We therefore directly characterize the population-optimal score field induced by the diffusion-based MCLR objective~\eqref{the full elbo-approximated MCLR objective in appendix score-based form} by treating the score function $\mb s_{\mb\theta}(\mb x_t,t,\mb c)$ as the optimization variable, without imposing any additional assumptions on the model parameterization. However, we do impose the following regularity constraint on the score function, which bounds its maximum norm:
\begin{align}
    ||\mb s_{\mb\theta}(\mb x_t,t,\mb c)||_2\leq M,
    \label{maximum norm constraint main text}
\end{align}
where $M$ is a positive constant. This constraint is required for the diffusion-based MCLR objective~\eqref{the full elbo-approximated MCLR objective in appendix score-based form} and~\eqref{the full elbo-approximated MCLR objective in appendix} to be well-posed so that a pointwise minimizer exists. Consider a $\mb x_t$ such that:
\begin{align}
    p_t(\mb x_t)\gg p_t(\mb x_t|\mb c),
\end{align}
i.e., $\mb x_t$ has a much larger density under the unconditional distribution compared to that under the conditional distribution. At such $\mb x_t$, the negative term in~\eqref{the full elbo-approximated MCLR objective in appendix score-based form} and~\eqref{the full elbo-approximated MCLR objective in appendix} dominates the training objective, and the objective value can be pushed arbitrarily small by simply increasing the norm of $\mb s_{\mb\theta}(\mb x_t,t,\mb c)$ or $\mathcal{D}_{\mb\theta}(\mb x_t,t,\mb c)$ to be arbitrarily large. Intuitively, MCLR relies on the positive term to prevent the degenerate behavior caused by the negative term, but at $\mb x_t$ where $p_t(\mb x_t)\gg p_t(\mb x_t|\mb c)$, the regularization effect of the positive term becomes too weak and the MCLR loss can be pushed arbitrarily small, while the score function or denoiser fails to learn any useful class-specific structure. Therefore, we impose a pointwise bound on the score magnitude. This constraint also reflects the implicit magnitude control induced in practice by finite network capacity and regularization. Now we are ready to present the main theorem.

Define:
\begin{align}
h_t(\mb x_t|\mb c):=p_t(\mb x_t|\mb c)+\eta\,\bigl(p_t(\mb x_t|\mb c)-p_t(\mb x_t)\bigr),
    \end{align}
which can be interpreted as a noisy version of the "sum-of-difference" contrastive function~\eqref{sum of difference distribution}. Then we have the following theorem. 
\begin{theorem}
\label{theorem:diffusion-mclr-main}
Consider the population diffusion-based MCLR objective over score
fields satisfying
$\|\mb s_{\mb\theta}(\mb x_t,t,\mb c)\|_2\leq M$.
For any $(\mb x_t,t,\mb c)$ such that
$\nabla_{\mb x_t}h_t(\mb x_t|\mb c)\neq\mb 0$, the pointwise
optimal score is
\begin{align}
\mb s_{\mb\theta^*}(\mb x_t,t,\mb c)
=
\begin{cases}
\nabla_{\mb x_t}\log h_t(\mb x_t|\mb c),
&
\begin{aligned}
&h_t(\mb x_t\mid\mb c)>0
\ \text{and}\
\|\nabla_{\mb x_t}\log h_t(\mb x_t|\mb c)\|_2
\leq M,
\end{aligned}
\\[3mm]
M\displaystyle
\frac{\nabla_{\mb x_t}h_t(\mb x_t|\mb c)}
{\|\nabla_{\mb x_t}h_t(\mb x_t|\mb c)\|_2},
& \text{otherwise}.
\end{cases}
\end{align}
\end{theorem}
Thus, the diffusion-based MCLR score is locally governed by the noise-mollified sum-of-difference function \(h_t\). In regions where \(h_t>0\) (i.e., $\frac{p_t(\mb x_t|\mb c)}{p_t(\mb x_t)}>\frac{\eta}{1+\eta}$), it follows \(\nabla\log h_t\), up to magnitude clipping determined by $M$. Note that:
\begin{align}
    \nabla_{\mb x_t}\log h_t(\mb x_t|\mb c)=\nabla_{\mb x_t}\log p_t(\mb x_t|\mb c)+\alpha_t(\eta,\mb x_t)\bigl(\nabla_{\mb x_t}\log p_t(\mb x_t|\mb c)-\nabla_{\mb x_t}\log p_t(\mb x_t)\bigr),
\end{align}
where $\alpha_t(\eta,\mb x_t)=\frac{1}{(1+\frac{1}{\eta})\frac{p_t(\mb x_t|\mb c)}{p_t(\mb x_t)}-1}$. Hence, in this region, diffusion-based MCLR takes the form of CFG with a sample-adaptive guidance strength $\alpha_t(\eta,\mb x_t)$, which is a monotonically decreasing function of the conditional-unconditional likelihood ratio $\frac{p_t(\mb x_t|\mb c)}{p_t(\mb x_t)}$. Therefore, the guidance induced by MCLR becomes weaker when the ratio is large, because the current sample already provides strong evidence for the target condition, and becomes stronger as this relative evidence decreases. 

When \(h_t\leq0\), the optimal field saturates at the norm bound and points along the steepest-ascent direction of \(h_t\), thereby potentially steering samples toward regions with stronger relative evidence for the target condition. 

The above discussion focuses on non-critical points. At critical points where $\nabla_{\mb x_t}h_t(\mb x_t|\mb c)=\mb 0$, $h_t$ provides no local directional information, so the MCLR objective may not identify a unique optimal direction. This degeneracy reflects the absence of local first-order gradient information rather than a failure of MCLR. In practice, the learned score values at the critical points are implicitly determined by the shared network parameterization and neighboring training samples. The complete characterization at critical points, the proof and synthetic simulation are deferred to~\Cref{subsection: analysis on MCLR for diffusion models}. Importantly, as shown in Corollary~\ref{corollary diffusion-based MCLR finetuning}, diffusion-based MCLR recovers the ground-truth conditional distribution under the mixture-error model~\eqref{mixture error model assumption}.
\subsection{CFG as a Contrastive Alignment Algorithm: A Mechanistic Interpretation}
\label{theoretical equivalence of CFG and Alignment Algorithm main text}
We now connect classifier-free guidance (CFG) to diffusion-based MCLR developed above.
In particular, we show that the CFG-guided score in~\eqref{CFG+ODE}
arises as the unique minimizer of a sample-adaptive weighted,
ELBO-approximated MCLR objective.
This result provides a formal interpretation of CFG as an inference-time
procedure that implicitly performs contrastive alignment between the
conditional distribution and the unconditional mixture (the proof is provided in~\Cref{connection to CFG}).

\begin{theorem}
\label{theorem: CFG equivalent to weighted mclr}
For any time sampling distribution $p(t)$ and weighting function $w(t)>0$, the CFG-guided score
\[
\mb s_\text{cfg}(\mb x_t,t,\mb c)
:=
\nabla_{\mb x_t}\log p_t(\mb x_t|\mb c)
+
\eta\,\big(
\nabla_{\mb x_t}\log p_t(\mb x_t|\mb c)
-
\nabla_{\mb x_t}\log p_t(\mb x_t)
\big)
\]
is the unique minimizer of a sample-adaptive weighted ELBO-approximated MCLR objective:
\begin{equation}
\begin{aligned}
 \mb s_\text{cfg}(\cdot)
=
\argmin_{\mb s_{\mb\theta}(\cdot)}\;
&\;\mathbb{E}_{\mb c,t\sim p(t),\,\mb x\sim p(\mb x|\mb c),\,\mb x_t\sim p_{0t}(\mb x_t|\mb x)}
\Bigl[
w(t)\|
\nabla_{\mb x_t}\log p_{0t}(\mb x_t|\mb x)
-
\mb s_{\mb\theta}(\mb x_t,t,\mb c)
\|_2^2
\Bigr]
\\
&+ \eta\,
\mathbb{E}_{\substack{\boldsymbol{c},{\tilde{\mb c}},t\sim p(t),\mb x\sim p(\cdot|\mb c),\mb y\sim p(\cdot|\tilde{\mb c})\\
p_{0t}(\boldsymbol{x}_t|\boldsymbol{x}), p_{0t}(\boldsymbol{y}_t|\boldsymbol{y})}}
\Biggl[
w(t)
\Bigl(
\|\nabla_{\boldsymbol{x}_t}\log p_{0t}(\boldsymbol{x}_t|\boldsymbol{x})
- \boldsymbol{s}_{\boldsymbol{\theta}}(\boldsymbol{x}_t,t,\boldsymbol{c})\|_2^2
\\
&\quad\quad\quad -\textcolor{red}{\frac{p_t(\mb y_t|\mb c)}{p_t(\mb y_t)}}
\|\nabla_{\boldsymbol{y}_t}\log p_{0t}(\boldsymbol{y}_t|\boldsymbol{y})
- \boldsymbol{s}_{\boldsymbol{\theta}}(\boldsymbol{y}_t,t,\boldsymbol{c})\|_2^2
\Bigr)
\Biggr],
\end{aligned}
\label{CFG equivalent MCLR form 1}
\end{equation}
where $p_t(\mb y_t):=\mathbb{E}_{\mb c}[p_t(\mb y_t|\mb c)]$.
\end{theorem}
The regularization term in~\eqref{CFG equivalent MCLR form 1} has the same contrastive structure as the standard MCLR objective in~\eqref{the full elbo-approximated MCLR objective in appendix score-based form} and its denoising form~\eqref{denoiser perspective}: it encourages the denoiser associated with the target condition $\mb c$ to achieve lower denoising error on samples from the target condition than on samples drawn from mismatched conditions. The key difference from standard MCLR is the sample-adaptive weight $\frac{p_t(\mb y_t|\mb c)}{p_t(\mb y_t)}$.
Since this weight is nonnegative, the inter-class contrastive nature of MCLR is preserved. 

Intuitively, a larger weight indicates that a noisy sample $\mb y_t$ from a mismatched condition $\tilde{\mb c}$ is more likely to be confused with samples from the target condition $\mb c$. Hence, the adaptive weight makes the CFG-equivalent objective~\eqref{CFG equivalent MCLR form 1} more selective over negative samples: 
samples $\mb y_t$ from mismatched conditions that have high relative density under the target condition $\mb c$ receive larger weights, while samples that are already easily distinguishable from the target condition (those with small $\frac{p_t(\mb y_t|\mb c)}{p_t(\mb y_t)}$) receive smaller weights. 
Thus, the adaptive weight can be interpreted as a form of hard-negative mining, emphasizing ambiguous or class-confusable samples that provide stronger contrastive signal.

This perspective clarifies the relationship between MCLR and CFG. 
Standard MCLR uses a uniform contrastive penalty over mismatched samples, whereas CFG corresponds to a sample-adaptive weighted version of MCLR. 
Therefore, CFG can be interpreted as an \textbf{inference-time} contrastive
alignment procedure, while MCLR provides a \textbf{training-time} mechanism for
partially internalizing this effect into the model. In practice, we implement the standard, uniformly weighted MCLR objective
rather than the sample-adaptive version. Nevertheless, as shown empirically, standard MCLR induces similar qualitative
effects and substantially narrows the gap between unguided sampling and CFG.
\begin{figure}[t]
    \centering
    \includegraphics[width=1\linewidth]{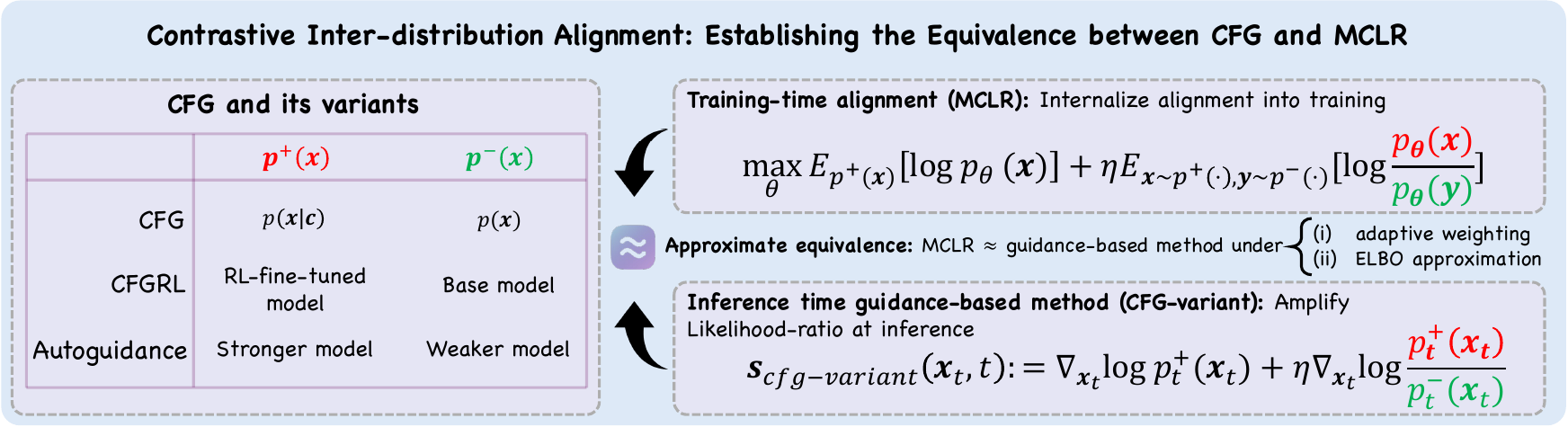}
    \caption{\textbf{A Unified Framework Connecting CFG Variants with Contrastive Alignment.}
    CFG-based methods can be interpreted as implicitly optimizing a likelihood-ratio--based contrastive alignment objective between two distributions $p^+(\mb x)$ and $p^-(\mb x)$ at inference time.}
    \label{fig:MCLR CFG Equivalence figure}
\end{figure}
\paragraph{Understanding CFG-Variants through the Alignment Lens.}Finally,~\Cref{theorem: CFG equivalent to weighted mclr} provides a unified perspective for interpreting CFG variants. By modifying the source distributions from which positive and negative samples are drawn in the contrastive objective, one recovers several existing guidance mechanisms as special cases.

For example in~\eqref{CFG equivalent MCLR form 1}, if $\mb x$ is sampled from a stronger model and $\mb y$ is sampled from a weaker model, the resulting optimal solution recovers the score of Autoguidance~\cite{karras2024guiding}. Similarly, if $\mb x$ is sampled from a reinforcement-learning fine-tuned model and $\mb y$ is sampled from the base model, the induced solution corresponds to the score used in CFGRL~\cite{frans2025diffusion}. Thus, MCLR provides a unified alignment interpretation for a broad class of guidance algorithms.
\Cref{fig:MCLR CFG Equivalence figure} summarizes this unified framework
and
\Cref{connection to CFG} provides detailed generalization and discussion.

\section{Alternative Inter-Class Contrastive Objectives}
\subsection{Adapting DPO for Improved Conditional Modeling}
\label{adapting DPO for Improved conditional modeling}
The core principle behind MCLR is to improve conditional modeling by encouraging the model to exploit class-dependent structures through inter-class contrast, which can also be instantiated via other contrastive objectives such as DPO. To study this connection, we adapt DPO to conditional generation by treating samples from the target class $\mb c$ as preferred ($\mb x_w$) and samples from other randomly selected classes as non-preferred ($\mb x_l$). This leads to the following objective: 
\begin{align}
    \min_{\mb\theta}\; 
    -\mathbb{E}_{\mb c,\tilde{\mb c},\, \mb x_w\sim p(\cdot|\mb c),\,\mb x_l\sim p(\cdot|\tilde{\mb c})}
    \log\mathrm{Sigmoid}\left(
    \beta \log\frac{p_{\mb\theta}(\mb x_w|\mb c)}{p_{\text{ref}}(\mb x_w|\mb c)}
    -\beta \log\frac{p_{\mb\theta}(\mb x_l|\mb c)}{p_{\text{ref}}(\mb x_l|\mb c)}
    \right).    
    \label{DPO objective-conditional}
\end{align}
We refer to~\eqref{DPO objective-conditional} as Conditional Contrastive DPO (CC-DPO). Similar to MCLR, CC-DPO objective also admits a closed-form solution, as stated in the following theorem (see proof in~\Cref{sec: proof of CC-DPO optimal solution}).
\begin{theorem}
\label{CC-DPO theorem 2}
Under standard regularity conditions, the optimal solution to~\eqref{DPO objective-conditional} is:
\begin{align}
    p_{\mb\theta^*}(\mb x|\mb c)=\frac{1}{\tilde{Z}(\mb c)}p_\text{ref}(\mb x|\mb c)\left(\frac{p(\mb x|\mb c)}{p(\mb x)}\right)^{\frac{1}{\beta}},
    \label{optimal CC-DPO solution}
\end{align}
where $\tilde{Z}(\mb c)=\int_{\mb x}p_\text{ref}(\mb x|\mb c)\left(\frac{p(\mb x|\mb c)}{p(\mb x)}\right)^{\frac{1}{\beta}}d\mb x$ is the normalizing constant. Moreover, this optimal solution coincides with that induced by Conditional Contrastive Alignment (CCA)~\cite{chentoward}. 
\end{theorem}
\paragraph{Comparison between MCLR and CC-DPO.}
The optimal solution~\eqref{optimal CC-DPO solution} is known as the \emph{gamma-powered} distribution, which was initially conjectured to characterize the effect of classifier-free guidance (CFG)~\cite{ho2022classifier}, but was later shown not to correspond to the true CFG dynamics~\cite{karras2024guiding, bradley2024classifier}. Although both methods amplify regions where $p(\mb x|\mb c)>p(\mb x)$
and suppress regions where $p(\mb x|\mb c)<p(\mb x)$,
they do so through fundamentally different transformations:
MCLR modifies the distribution additively through the density difference
$p(\mb x|\mb c)-p(\mb x)$,
whereas CC-DPO reweights it multiplicatively through the density ratio
$\bigl(\tfrac{p(\mb x | \mb c)}{p(\mb x)}\bigr)^\frac{1}{\beta}$.
The multiplicative form of CC-DPO can be overly aggressive. Consider a point $\tilde{\mb x}$ such that $p(\tilde{\mb x}) \approx 0$ while $p(\tilde{\mb x} |\mb c) > 0$, a situation that naturally arises when $\mb c$ are minority classes.
In this case, the ratio $\bigl(\tfrac{p(\tilde{\mb x} | \mb c)}{p(\tilde{\mb x})}\bigr)^\frac{1}{\beta}$ becomes ill-conditioned, strongly amplifying the density at $\tilde{\mb x}$, therefore potentially driving the learned conditional distribution toward degenerate or unstable solutions.
By contrast, the optimal solution~\eqref{main text optimal solution under MCLR} induced by MCLR remains well-behaved. We include CC-DPO and CCA as contrastive baselines in our experiments.
\subsection{Contrastive Flow Matching}
Contrastive Flow Matching (CFM)~\cite{stoica2025contrastive} provides another way to introduce conditional contrast into diffusion training, but it induces a fundamentally different optimal score field compared to MCLR. Specifically, we show that its optimal score is:
\begin{align}
\mb s_{\mb\theta^*}(\mb x_t,t,\mb c)
=
(1+\eta)\nabla_{\mb x_t}\log p_t(\mb x_t|\mb c)
+
\eta\frac{b_t}{a_t}(\mb x_t-\mb\mu_t),
\label{optimal CFM score main text}
\end{align}
where $\mb\mu_t:=\mathbb{E}_{\mb x_t\sim p_t(\mb x_t)}[\mb x_t]$ is the unconditional mean, $\eta$ is the CFM regularization strength and $a_t, b_t$ are determined by the noise schedule. Thus, CFM rescales the conditional score and introduces a global mean-centered affine correction determined by the displacement between the current noisy sample and the unconditional mean. Unlike CFG and MCLR, the CFM correction is not explicitly adaptive to the local conditional--unconditional density geometry. Consequently, CFM can promote global class separation, but its correction need not vanish even in regions that are already strongly associated with the target condition. We provide the complete theoretical analysis in~\Cref{sec:comparison-with-cfm}. Under matched fine-tuning settings, MCLR consistently outperforms CFM (see~\Cref{empirical comparison between MCLR and CFM}).

\section{Related Work}
Our work contributes to recent efforts~\cite{chentoward,chenvisual,tang2025diffusion} that seek to improve guidance-free conditional generation by modifying
the training objective rather than applying classifier-free guidance (CFG)
at inference time.
The closest related method is Conditional Contrastive Alignment (CCA)~\cite{chentoward}, which learns the gamma-powered distribution via Noise Contrastive Estimation~\cite{gutmann2010noise}.
We show that MCLR outperforms CCA in diffusion models while achieving comparable performance in autoregressive settings, and further reveal a previously unexplored equivalence between CC-DPO and CCA. 

Another related direction is Guidance-Free Training (GFT)~\cite{chenvisual,tang2025diffusion}, which aims to reproduce CFG-induced score functions through modified DSM objectives. While GFT mimics the functional form of CFG during training, MCLR approaches the problem from a likelihood-ratio
perspective and reveals the contrastive structure implicit in CFG,
offering a complementary mechanistic interpretation. 

Several additional works~\cite{yan2024training,lee2025aligning,kadkhodaie2024feature,yun2025no} employ class-wise contrastive objectives to improve conditional generation, but lack a formal characterization of the underlying theoretical properties. 

Finally, Direct Discriminative Optimization (DDO)~\cite{zheng2025direct} contrasts real and synthetic samples rather than class-conditional distributions, and is included as a baseline in our experiments.
Additional discussion of related work is deferred to~\Cref{More related works}.
\section{Experimental Results}
\label{experiment results main text}
\textbf{Experiment Setup.} In our experiments, we focus on fine-tuning pretrained models. For diffusion models, we fine-tune pretrained EDM2 models~\cite{karras2024analyzing} on ImageNet-64$\times$64 and ImageNet-512$\times$512, and SiT (with REPA) model~\cite{yurepresentation} on ImageNet-256$\times$256. For visual autoregressive models, we fine-tune VAR-d24~\cite{tian2024visual} on ImageNet-256$\times$256. We compare against CFG (without interval) and several training-time
baselines, including CC-DPO, CCA, CFM and DDO.

The generative performance is evaluated with \text{Fr\'echet Distance} (FD)~\cite{heusel2017gans}, Precision and Recall~\cite{kynkaanniemi2019improved}, and Inception score (IS)~\cite{salimans2016improved}. For diffusion models, we primarily report FD computed with DINOv2 features (FD\textsubscript{DINOv2}). While FID (Inception-based FD) is widely used in the diffusion literature, we find that for strong pretrained models such as EDM2 which already achieves strong FID (1.5-2), it can be insensitive to perceptually meaningful improvement: although both CFG and MCLR lead to visually pronounced quality improvements, the FID score often does not improve and may even degrade. In contrast, FD\textsubscript{DINOv2} consistently captures these improvements, aligning with prior findings that it correlates more strongly with human evaluations~\cite{stein2023exposing}. For base models that are less strong including VAR and SiT, both FID and FD\textsubscript{DINOv2} improve consistently. Additional implementation details, ablations, and qualitative results deferred to~\Cref{additional results appendix,ablation study appendix}. 
\begin{figure*}[t]
\centering\includegraphics[width=1\linewidth]{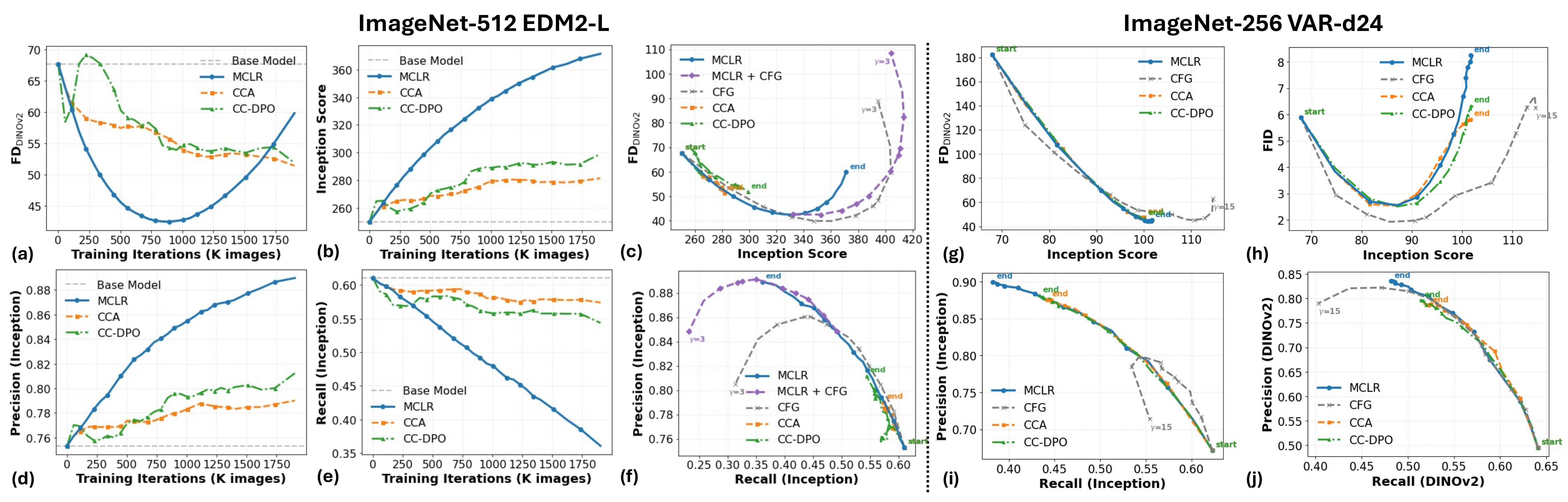}%
\caption{\textbf{Quantitative Results for EDM2-L and VAR-d24.}
(a), (b), (d), and (e) show the evolution of FD, Inception Score, Precision, and Recall, respectively, as functions of training iterations.
(c), (g), and (h) illustrate the FD--IS trade-offs, while (f), (i), and (j) depict the Precision--Recall trade-offs.
For EDM2 models, we evaluate classifier-free guidance (CFG) scales $\gamma \in \{0.1, 0.2, 0.3, 0.4, 0.5, 0.7, 0.9, 1, 1.5, 2.0, 3.0\}$. For Var-d24 model, $\gamma \in \{0.5, 0.8, 1.1, 1.5, 1.7, 2.0, 2.5, 3.0, 4.0, 5.0, 7.0, 10.0, 15.0\}$.
\textbf{Start} denotes the performance of the base model, while \textbf{End} denotes the model obtained after a fixed finetuning duration.}
    \label{fig:imagenet512 quant}
\end{figure*}

\paragraph{Progressive Class Separation and Fidelity--Diversity Trade-off.} We first analyze the training dynamics induced by MCLR and their effect on conditional generation. Qualitatively, as shown in~\Cref{fig:imagenet512 qualitative teaser,fig:class_207_qua,fig:class_343_qua,fig:class_513_qua,fig:class_515_qua,fig:class_617_qua,fig:additional_64_qua_1,fig:additional_64_qua_2,fig:additional_64_qua_3,fig:additional_512_qua_1,fig:additional_512_qua_2,fig:additional_512_qua_3,fig:additional_512_qua_4}, at early stages of training, images generated from the same initial noises share similar global structures across different class conditions, indicating weak class-conditional modeling by the base model. As training proceeds, generated images gradually develop distinct class-specific structures, reflecting increasing inter-class separation and improved conditional modeling. 

Quantitatively, as shown in~\Cref{fig:imagenet512 quant}(b,d), both Inception Score (IS) and Precision increase steadily during training, indicating that the generated samples become more class-discriminative and visually faithful.
However, excessive training reduces within-class diversity, as reflected by a decrease in Recall (\Cref{fig:imagenet512 quant}(e)). As a consequence, FD\textsubscript{DINOv2} initially improves as conditional modeling strengthens, but later deteriorates when diversity decreases, as shown in~\Cref{fig:imagenet512 quant}(a), revealing a characteristic fidelity--diversity trade-off.

This behavior closely mirrors the effect of increasing the guidance scale in CFG, where stronger guidance improves class fidelity at the cost of diversity.
Similar trade-offs are observed for other contrastive alignment objectives such as CC-DPO and CCA. Following prior works, we therefore report results at the checkpoint achieving the best FD\textsubscript{DINOv2} score in~\Cref{tab:results} for each algorithm.

\paragraph{MCLR outperforms training-time baselines.}
\Cref{tab:results} shows that MCLR consistently achieves the best FD\textsubscript{DINOv2} score among training-time baselines. On EDM2-S, MCLR improves FD-DINOv2 from $95.20$ to $52.69$, outperforming DDO, CCA, and CC-DPO. 
On EDM2-L, MCLR reaches $42.50$ FD\textsubscript{DINOv2}, substantially outperforming the strongest training-time baseline. 
On VAR-d24 and SiT (with REPA), MCLR achieves the best FD\textsubscript{DINOv2} among all compared methods, including CFG.

Moreover, for EDM models, as shown in~\Cref{fig:imagenet512 quant} (c,f), MCLR traverses a significantly wider fidelity--diversity trade-off region steadily with a faster training speed. In contrast, CCA and CC-DPO converge early at suboptimal local minima and exhibit zigzag optimization trajectories, as evidenced by slowly improving or stalled learning curves. A controlled comparison with CFM under matched fine-tuning settings leads to a similar conclusion: MCLR achieves larger and more sustained improvements across different DSM regularization strengths (see~\Cref{empirical comparison between MCLR and CFM}).

\paragraph{MCLR Substantially Narrows the Gap to CFG.}
CFG still achieves the best FD\textsubscript{DINOv2} on EDM2 models, but MCLR substantially narrows the gap under standard guidance-free sampling. 
For example, on EDM2-L, MCLR obtains $42.50$ FD-DINOv2 compared with CFG's $39.86$, while using half the NFE. When evaluated using Precision--Recall, MCLR exhibits competitive performance relative to CFG. Specifically, as shown in~\Cref{fig:imagenet512 quant}(f), MCLR matches CFG in the high-recall regime corresponding to early training stages, and attains a substantially higher best-case precision in the high-precision regime at later training stages, where CFG begins to produce images with oversaturated colors. Moreover, applying CFG on top of an MCLR-fine-tuned model further narrows the performance gap between the two methods, yielding higher best-case Inception Score and Precision on both EDM2-S and EDM2-L models (see~\Cref{fig:imagenet512 quant}(c,f) and~\Cref{fig:ImageNet64 EDM2 quant}(c,f,i)). 

Qualitatively, MCLR and CFG produce highly similar visual effects, both substantially enhancing class-specific structures in the generated images, as shown in~\Cref{fig:imagenet512 qualitative teaser,fig:class_207_qua,fig:class_343_qua,fig:class_513_qua,fig:class_515_qua,fig:class_617_qua,fig:additional_64_qua_1,fig:additional_64_qua_2,fig:additional_64_qua_3,fig:additional_512_qua_1,fig:additional_512_qua_2,fig:additional_512_qua_3,fig:additional_512_qua_4}. These observations are consistent with our theoretical analysis in~\Cref{theoretical equivalence of CFG and Alignment Algorithm main text}, which interprets CFG as an implicit contrastive alignment method.
\begin{table}[t]
  \caption{
  Quantitative results for MCLR, CFG and training-time baselines on ImageNet. For each algorithm, the metrics are reported at the model checkpoint (or CFG scale) achieving the \textcolor{blue}{best-case} \textcolor{blue}{FD\textsubscript{DINOv2}}.
  Precision and Recall are computed using Inception features. The best and second-best results are highlighted in \textbf{bold} and \underline{underline}, respectively.}
  \label{tab:results}
  \centering
  \footnotesize
  \setlength{\tabcolsep}{4pt}

  \begin{subtable}[t]{0.45\linewidth}
    \centering
    \caption{EDM2 (ImageNet-$64\times 64$)}
    \label{tab:results_in64}
    \resizebox{\linewidth}{!}{%
    \begin{tabular}{lccccc}
      \toprule
      Method & NFE
      & FD$_{\text{DINOv2}}\!\downarrow$
      & Prec.$\uparrow$
      & Rec.$\uparrow$
      & IS$\uparrow$ \\
      \midrule
      EDM2-S               & 63 & 95.20  & 0.705 & \underline{0.614} & 60.43 \\
      +CFG                 &126 & \textbf{43.75}  & \textbf{0.800} & 0.565 & \textbf{127.40} \\
      +DDO                 & 63 & 72.97  & 0.689 & \textbf{0.642} & 65.84 \\
      +CCA                 & 63 & 62.36  & 0.762 & 0.557 & 76.13 \\
      +CC-DPO              & 63 & 60.98  & \underline{0.784} & 0.536 & 86.11 \\
      \textbf{+MCLR (Ours)}& 63 & \underline{52.69} & \textbf{0.800} & 0.505 & \underline{90.68} \\
      \bottomrule
    \end{tabular}}
  \end{subtable}%
  \hfill
  \begin{subtable}[t]{0.45\linewidth}
    \centering
    \caption{EDM2 (ImageNet-$512\times 512$)}
    \label{tab:results_in512}
    \resizebox{\linewidth}{!}{%
    \begin{tabular}{lccccc}
      \toprule
      Method & NFE
      & FD$_{\text{DINOv2}}\!\downarrow$
      & Prec.$\uparrow$
      & Rec.$\uparrow$
      & IS$\uparrow$ \\
      \midrule
      EDM2-L               & 63 & 67.70  & 0.753 & \underline{0.610} & 250.07 \\
      +CFG                 &126 & \textbf{39.86}  & \underline{0.844} & 0.512 & \textbf{360.30} \\
      +DDO                 & 63 & 49.47  & 0.737 & \textbf{0.652} & 268.72 \\
      +CCA                 & 63 & 51.45  & 0.790 & 0.574 & 281.45 \\
      +CC-DPO              & 63 & 51.92  & 0.812 & 0.544 & 298.89 \\
      \textbf{+MCLR (Ours)}& 63 & \underline{42.50}  & \textbf{0.849} & 0.492 & \underline{332.02} \\
      \bottomrule
    \end{tabular}}
  \end{subtable}
  \begin{subtable}[t]{0.49\linewidth}
    \centering
    \caption{VAR (ImageNet-$256\times 256$)}
  \label{tab:results_in256}
    \resizebox{\linewidth}{!}{%
  \begin{tabular}{lccccc}
    \toprule
    Method & NFE
    &FD$_{\text{DINOv2}}\!\downarrow$
    & Prec.$\uparrow$
    & Rec.$\uparrow$
    & IS$\uparrow$ \\
    \midrule
    VAR-d24              & 10 & 182.12 & 0.672 & \textbf{0.623} & 67.92 \\
    +CFG                 & 20 & \underline{45.08}  & 0.798 & \underline{0.542} & \textbf{110.70} \\
    +CCA                 & 10 &46.82  & 0.873 & 0.448 & 98.92 \\
    +CC-DPO              & 10 &46.63  & \underline{0.881} & 0.433 & 100.12 \\
    \textbf{+MCLR (Ours)}& 10 & \textbf{44.31}  & \textbf{0.893} & 0.404 & \underline{100.84} \\
    \bottomrule
  \end{tabular}}
  \end{subtable}%
  \hfill 
  \begin{subtable}[t]{0.49\linewidth}
    \centering
    \caption{SiT-XL/2+REPA (ImageNet-$256\times 256$)}
    \label{tab:results_in256 SiT}
    \resizebox{\linewidth}{!}{%
    \begin{tabular}{lccccc}
      \toprule
      Method & NFE
      & FD$_{\text{DINOv2}}\!\downarrow$
      & Prec.$\uparrow$
      & Rec.$\uparrow$
      & IS$\uparrow$ \\
      \midrule
      SiT-XL/2+REPA               & 50 & 184.36  & 0.603 & \textbf{0.683} & 142.24 \\
      +CFG                 &100 & 50.23  & \underline{0.833} & 0.506 & \textbf{385.87} \\
      +CCA                 & 50 & 60.20  &0.769 &\underline{0.551}  &254.56  \\
      +CC-DPO              & 50 & \underline{47.20}  &\textbf{0.839}  & 0.451  & \underline{329.83} \\
      \textbf{+MCLR (Ours)}& 50 & \textbf{45.96} & {0.825} & 0.488 & 311.00 \\
      \bottomrule
    \end{tabular}}
  \end{subtable}%
  \hfill
\end{table}
\paragraph{Equivalence between CCA and CC-DPO.}
Finally, we observe that CC-DPO consistently matches or outperforms CCA (see~\Cref{subsec:hyperparameter for CCA}).
This empirical observation aligns with our theoretical analysis in~\Cref{adapting DPO for Improved conditional modeling}, which establishes the equivalence between these two objectives.

\section{Limitations and Conclusions}
This work introduced MCLR, a training-time objective that substantially improves the unguided base model
and narrows the gap to inference-time CFG. Beyond empirical improvements, our analysis reveals a close connection
between MCLR and CFG, identifying the implicit inter-class contrastive alignment mechanism of CFG. 

\paragraph{Limitations and Future Directions.}
As a training-time method, MCLR produces a fixed model and cannot dynamically adjust the fidelity--diversity trade-off at inference time as CFG does. CFG also typically achieves stronger best-case FD and Inception Score. Bridging training-time objectives and inference-time guidance remains an important future direction. More fundamentally, these observations raise a broader question:

\begin{center}
\begin{tcolorbox}[
  colback=gray!18,
  colframe=black,
  arc=4mm,
  boxrule=0.8pt,
  width=0.92\linewidth
]
\centering
\vspace{0.3em}

\emph{\textbf{Should alignment be performed at training time, or is inference-time guidance a more effective paradigm?}}

\vspace{0.3em}
\end{tcolorbox}
\end{center}
Recent studies~\cite{frans2025diffusion,jin2025inference,cheng2025diffusion,jiang2026rethinking} suggest that inference-time guidance can, in certain settings, outperform training-based alignment algorithms. These observations point toward a promising research direction: developing alignment algorithms that operate directly at inference time. Exploring this possibility may lead to a broader shift from traditional training-time alignment toward inference-time alignment paradigms.

Lastly, the lack of inter-class separation is only one practical limitation
of the standard DSM; systematically identifying and addressing additional limitations remains an
important future direction.

\clearpage
\section*{Data Availability Statement}
The code and instructions for reproducing the experimental results will
be made available in our
\href{https://github.com/Morefre/MCLR-Contrastive-Alignment-for-Improving-Conditional-Modeling}
{code repository}.

\section*{Acknowledgment}
We acknowledge funding support from NSF CCF-2212066, NSF CCF-
2212326, NSF IIS 2402950, and ONR N000142512339. This research used the Delta advanced computing and data resource which is supported by the National Science Foundation (award OAC 2005572) and the State of Illinois. Delta is a joint effort of the University of Illinois Urbana-Champaign and its National Center for Supercomputing Applications\cite{boerner2023access}.







\printbibliography

\appendix
\newpage
\begin{center}
{\LARGE \bf Appendices}
\end{center}
\par\noindent\rule{\textwidth}{1pt}
\section{Weak Inter-Class Separation in Diffusion Models}
\label{app:weak-inter-class-separation}
In~\Cref{subsec:lack of class-specificity}, we argue that conditional diffusion models trained with denoising score matching (DSM) may suffer from weak inter-class separation, yielding class-conditional distributions that
are less separated than their real-data counterparts. In this section, we provide quantitative evidence for this phenomenon by comparing pairwise Fr\'echet Inception Distances (FID) among real class-conditional distributions with the corresponding distances among generated class-conditional distributions. Specifically, for two classes $\mb c_1$ and $\mb c_2$, let $\mathcal{D}_{\mb c_1}$ and $\mathcal{D}_{\mb c_2}$ denote the corresponding real training data sets, and let $\mathcal{G}_{\mb c_1}$ and $\mathcal{G}_{\mb c_2}$ denote the sample sets generated by a conditional diffusion model. We compute
$\operatorname{FID}(\mathcal{D}_{\mb c_1},\mathcal{D}_{\mb c_2})$ and
$\operatorname{FID}(\mathcal{G}_{\mb c_1},\mathcal{G}_{\mb c_2})$. Note that FID is a measure of distances between two distributions in the InceptionV3 feature space, and a larger FID means the two measured distributions are more separated. Therefore, if $\text{FID}(\mathcal{G}_{\mb c_1},\mathcal{G}_{\mb c_2})<\text{FID}(\mathcal{D}_{\mb c_1},\mathcal{D}_{\mb c_2})$, the generated conditional distributions for $\mb c_1$ and $\mb c_2$ are less separated than their real-data counterparts. 

We perform this evaluation on eight selected ImageNet classes following~\cite{kynkaanniemi2019improved}: four identified as relatively easy to model---great Pyrenees (GP), Egyptian cat (EC), broccoli (BR) and lemon (LE)---and four relatively
difficult classes---cornet/trumpet (CT), park bench (PB), bubble (BU), and baseball
player (BP). As shown in~\cite{kynkaanniemi2019improved}, the easy classes are largely textural or are common in the datasets while the hard classes often contain precise global structure or are underrepresented in the dataset. We generate samples using SiT-XL/2 with REPA and a 50-step
Euler ODE sampler, without classifier-free guidance (CFG). For each class, we use all available training images and generate an
equal number of samples for computing FID.

\Cref{tab:inter_class_fid} reports the pairwise inter-class FID. Across all 28 evaluated class pairs, the generated pairwise FID is lower
than its real-data counterpart. This suggests that the generated class-conditional distributions are systematically less separated than the corresponding real-data distributions. We next examine whether CFG
and MCLR can recover this missing separation.~\Cref{tab:interclass_fid_mean} 
reports the mean pairwise FID over all $28$ class pairs.
Both methods increase inter-class separation.~\Cref{fig:weak inter-class separation qualitative result} qualitatively demonstrate the weak class separation issue, where initial random noise is fixed across all class conditions and generation methods. 


\begin{table}[h]
\centering
\caption{\textbf{Pairwise inter-class FID for real and generated samples.}
Each entry reports
$\operatorname{FID}(\mathcal{D}_{\mb c_i},\mathcal{D}_{\mb c_j}) /
 \operatorname{FID}(\mathcal{G}_{\mb c_i},\mathcal{G}_{\mb c_j})$.
Only the upper triangle is shown. Across all 28 class pairs, the generated
pairwise FID is lower than the corresponding real-data pairwise FID,
indicating weaker inter-class separation in the generated distributions. The larger value in each entry is shown in bold.}
\label{tab:inter_class_fid}
\begingroup
\setlength{\tabcolsep}{4pt}
\renewcommand{\arraystretch}{1.1}
\resizebox{0.75\linewidth}{!}{%
\begin{tabular}{@{}l|*{8}{c}@{}}
   & GP & EC & CT & PB & BR & LE & BU & BP \\
\hline
GP & -- & \textbf{227}/186 & \textbf{234}/201 & \textbf{294}/213 & \textbf{328}/309
   & \textbf{249}/232 & \textbf{324}/262 & \textbf{284}/248 \\
EC &    & -- & \textbf{233}/196 & \textbf{286}/208 & \textbf{315}/312
   & \textbf{254}/240 & \textbf{307}/242 & \textbf{275}/247 \\
CT &    &    & -- & \textbf{281}/190 & \textbf{303}/287
   & \textbf{238}/217 & \textbf{286}/198 & \textbf{265}/212 \\
PB &    &    &    & -- & \textbf{380}/327
   & \textbf{299}/243 & \textbf{347}/239 & \textbf{309}/223 \\
BR &    &    &    &    & --
   & \textbf{310}/264 & \textbf{382}/346 & \textbf{356}/328 \\
LE &    &    &    &    &
   & -- & \textbf{321}/274 & \textbf{306}/273 \\
BU &    &    &    &    &
   &    & -- & \textbf{354}/276 \\
BP &    &    &    &    &
   &    &    & -- \\
\end{tabular}%
}
\endgroup
\end{table}

\begin{table}[htbp]
    \centering
    \caption{\textbf{Mean pairwise inter-class FID over the 28 class pairs.}
    The real training data serve as the reference for inter-class separation. The absolute
    gap measures deviation from this reference; lower is better.
    Here, $\gamma$ denotes the CFG scale, with $\gamma=0$
    corresponding to unguided conditional sampling.}
    \label{tab:interclass_fid_mean}
    \begin{tabular}{llcc}
        \toprule
        Source / Model & Sampling
        & Mean Pairwise FID
        & Absolute Gap $\downarrow$ \\
        \midrule
        Real training data
        & -- & 298.1 & -- \\
        SiT-XL/2 + REPA
        & Unguided ($\gamma=0$) & 249.8 & 48.3 \\
        SiT-XL/2 + REPA
        & CFG ($\gamma=1$) & 310.5 & 12.4 \\
        SiT-XL/2 + REPA + MCLR
        & Unguided ($\gamma=0$) & 287.7 & \textbf{10.4} \\
        \bottomrule
    \end{tabular}
\end{table}

\begin{figure}[t]
    \centering
    \includegraphics[width=0.9\linewidth]{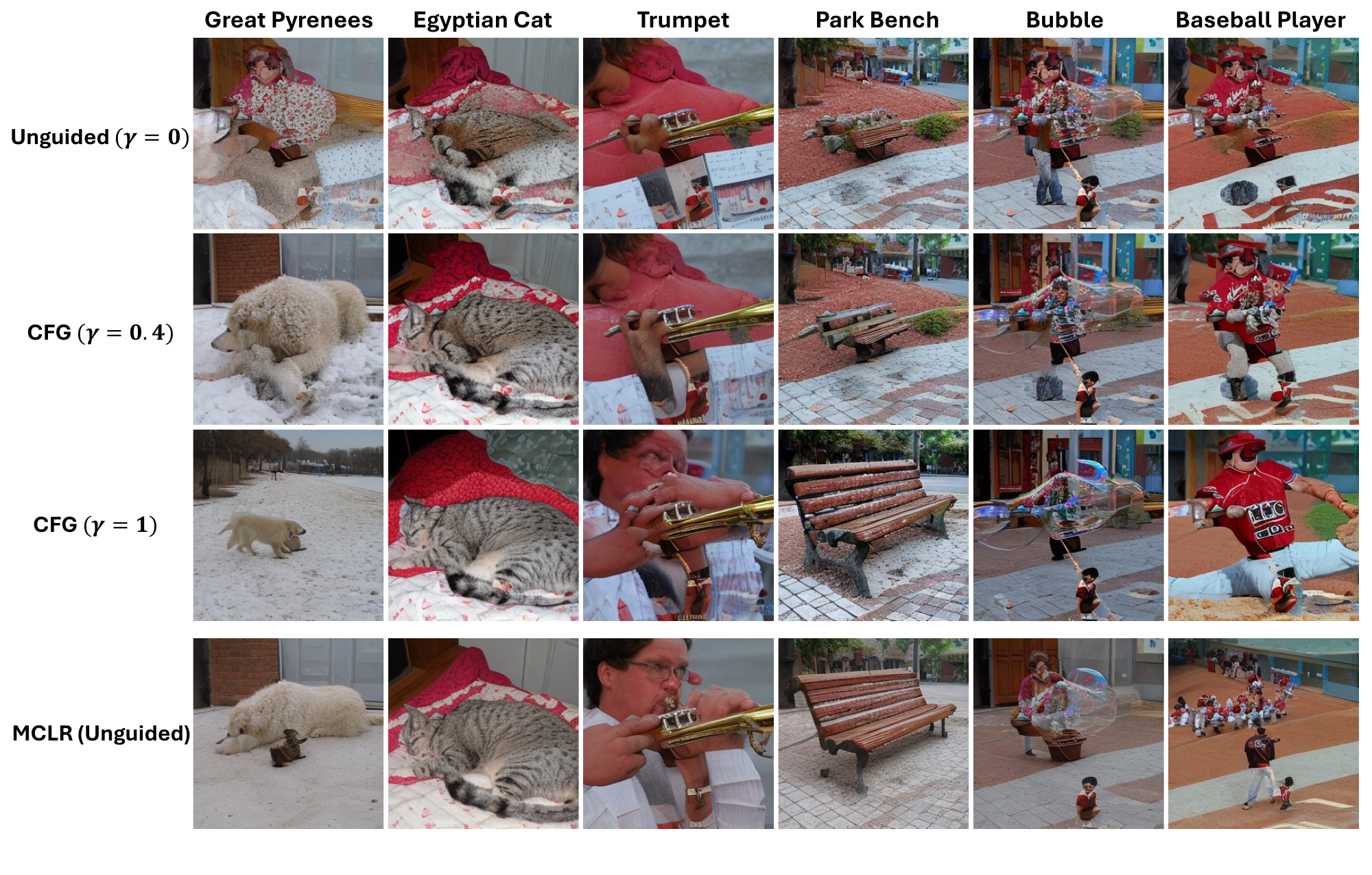}
    \caption{\textbf{Qualitative Comparison of Inter-Class Separation.}
    All panels share the same initial noise; columns vary the class
condition, while rows vary the generation method. Under unguided
sampling, the baseline model preserves similar class-agnostic
structures across different labels. Increasing the CFG strength
progressively enhances class-specific structure. MCLR produces a
similar effect under unguided sampling, without applying CFG at
inference time.}
    \label{fig:weak inter-class separation qualitative result}
\end{figure}

Lastly, we note that similar evidence of weak inter-class separation was previously reported by~\cite{li2026towards} using EDM on ImageNet-$64\times 64$~\cite{karras2022elucidating}. Our results extend this observation to the more
recent transformer-based SiT-XL/2 architecture with REPA at
$256\times256$ resolution, suggesting that the phenomenon persists
across model architectures and image resolutions.

\section{Theoretical Analysis of MCLR}
\subsection{Equivalent Forms of MCLR Objective}
By the linearity of expectation, it is easy to show that the MCLR objective~\eqref{MLE with Likelihood ratio constraint symmetric} admits the following two equivalent forms:
\begin{align}
\label{MLE with Likelihood ratio constraint in appendix}
    \max_{\mb\theta}\;
    \mathbb{E}_{\mb c,\,p(\mb x\mid\mb c)}
    \big[\log p_{\mb\theta}(\mb x|\mb c)\big]+\eta\,
    \underbrace{
    \mathbb{E}_{\mb c,\, \tilde{\mb c},\,\mb x\sim p(\cdot|\mb c),\mb y\sim p(\cdot|\tilde{\mb c})}
    \Big[
    \log \frac{p_{\mb\theta}(\mb x|\mb c)}{p_{\mb\theta}(\mb y|\mb c)}
    \Big]
    }_{\text{MCLR Regularization (Form I)}},
\end{align} 
and 
\begin{align}
\label{MLE with Likelihood ratio constraint in the appendix !!!}
    \max_{\mb\theta}\;
    \mathbb{E}_{\mb c,\,p(\mb x\mid\mb c)}
    \big[\log p_{\mb\theta}(\mb x|\mb c)\big]+\eta\,
    \underbrace{
    \mathbb{E}_{\mb c,\,\tilde{\mb c},\,\mb x\sim p(\mb x|\mb c)}
    \Big[
    \log \frac{p_{\mb\theta}(\mb x|\mb c)}{p_{\mb\theta}(\mb x|\tilde{\mb c})}
    \Big]
    }_{\text{MCLR Regularization (Form II)}}.
\end{align}
While the main text presents MCLR in Form I, the appendix derives most results from Form II. Because these two formulations are equivalent, we use them interchangeably throughout the analysis.

\subsection{Distributional Analysis of MCLR}
\label{sec: proof for MCLR main}
In this section, we provide the proof for~\Cref{MCLR theorem}. We first restate the assumptions and theorem. 
\begin{assum}
    \label{MCLR assumption 1}
    The function $h(\mb x|\mb c)$ has compact support; that is, 
\begin{align}
\operatorname{supp} h(\mb x|\mb c)\;\subseteq\;K, \qquad |K|<\infty,
\end{align}
where $|K|$ denotes the volume of set $K$.
\end{assum}

\begin{assum}
\label{MCLR assumption 2}
For $\forall \mb x\in K$, 
\begin{align}
p_{\mb\theta}(\mb x|\mb c)\geq \delta>0, \qquad \delta<\frac{1}{|K|}. 
\end{align}
\end{assum}
Assumption~\ref{MCLR assumption 1} is mild in practice, as image data typically occupies a bounded pixel range. Assumption~\ref{MCLR assumption 2} is a regularity assumption that ensures the log-likelihood is well-defined, avoiding the singularity that occurs when evaluating $\log p_{\mb\theta}(\mb x|\mb c)$ at zero density. Since $p_{\mb\theta}(\mb x|\mb c)$ integrates to one over $K$, we necessarily have $\delta\leq\frac{1}{|K|}$. In the following, we further assume $\delta<\frac{1}{|K|}$. 



\noindent\textbf{Theorem~\ref{MCLR theorem}.}
\textit{Under the two assumptions, the optimal solution to~\eqref{MLE with Likelihood ratio constraint} is:    
\begin{align}
    p_{\mb\theta^*}(\mb x|\mb c)\;=\;
    \begin{cases}
        \max \bigl\{\frac{h(\mb x|\mb c)}{Z(\mb c)},\delta\bigl\}, \;\mb x\in K, \\
        0, \;\mb x\notin K,
    \end{cases}
\end{align}
where $Z(\mb c)$ is the normalizing constant.}

\begin{proof}
Without loss of generality, we consider discrete classes in this proof; the theorem extends straightforwardly to continuous classes. Suppose there are $M$ classes $\{\mb c_i\}_{i=1}^M$, each class has prior probability $p(\mb c_i)$ and $\sum_{i=1}^Mp(\mb c_i)=1$, then the training objective~\eqref{MLE with Likelihood ratio constraint} takes the following form:

\begin{align}
 \max_{\mb\theta}\mathcal{L}(\mb\theta):=\mathbb{E}_{\mb c, p(\mb x|\mb c)}\left[\log p_{\mb\theta}(\mb x|\mb c)\right]+\eta\mathbb{E}_{\mb c, \tilde{\mb c},\mb x\sim p(\mb x|\mb c)}\left[\log \frac{p_{\mb\theta}(\mb x|\mb c)}{p_{\mb\theta}(\mb x|\tilde{\mb c})}\right]\label{MCLR form II starting point}\\
 =\max_{\mb\theta}\sum_{i=1}^Mp(\mb c_i)\mathbb{E}_{p(\mb x|\mb c_i)}\left[\log p_{\mb\theta}(\mb x|\mb c_i)\right]+\eta\sum_{i=1}^M p(\mb c_i)\sum_{j=1}^M p(\mb c_j)\mathbb{E}_{\mb x\sim p(\mb x|\mb c_i)}\left[\log p_{\mb\theta}(\mb x|\mb c_i)\right]\\ -\eta\sum_{i=1}^M p(\mb c_i)\sum_{j=1}^M p(\mb c_j)\mathbb{E}_{\mb x\sim p(\mb x|\mb c_i)}\left[\log p_{\mb\theta}(\mb x|\mb c_j)\right].
\end{align}
Note that we can decompose overall objective $\mathcal{L}(\mb\theta)$ as:
\begin{align}
    \mathcal{L}(\mb\theta) = \sum_{k=1}^M \mathcal{L}_{k}(\mb\theta),
\end{align}
where $\mathcal{L}_k(\mb\theta)$ is the amount contributed by $p_{\mb\theta}(\mb x|\mb c_k)$:
\begin{align}
    \mathcal{L}_k(\mb\theta)=p(\mb c_k)\mathbb{E}_{p(\mb x|\mb c_k)}\left[\log p_{\mb\theta}(\mb x|\mb c_k)\right] + \eta p(\mb c_k)\mathbb{E}_{p(\mb x|\mb c_k)}\left[\log p_{\mb\theta}(\mb x|\mb c_k)\right] \\
    -\eta p(\mb c_k)\sum_{i=1}^M p(\mb c_i)\mathbb{E}_{p(\mb x|\mb c_i)}\left[\log p_{\mb\theta}(\mb x|\mb c_k)\right]\\
    =p(\mb c_k)\mathbb{E}_{p(\mb x|\mb c_k)}\left[\log p_{\mb\theta}(\mb x|\mb c_k)\right] + \eta p(\mb c_k)\mathbb{E}_{p(\mb x|\mb c_k)}\left[\log p_{\mb\theta}(\mb x|\mb c_k)\right]
    -\eta p(\mb c_k)\mathbb{E}_{p(\mb x)}\left[\log p_{\mb\theta}(\mb x|\mb c_k)\right].
\end{align}
Note that we can optimize $L_k(\mb \theta)$ for each $k\in\{1,\ldots,M\}$ to get the optimal conditional distribution $p(\mb x|\mb c_k)$ independently:
\begin{align}
    \max_{\mb\theta} \mathcal{L}_k(\mb\theta)\Leftrightarrow \max_{p_{\mb\theta}(\cdot|\mb c_k)} \mathbb{E}_{p(\mb x|\mb c_k)}\left[\log p_{\mb\theta}(\mb x|\mb c_k)\right] + \eta \mathbb{E}_{p(\mb x|\mb c_k)}\left[\log p_{\mb\theta}(\mb x|\mb c_k)\right]
    -\eta \mathbb{E}_{p(\mb x)}\left[\log p_{\mb\theta}(\mb x|\mb c_k)\right]\label{independently solve for kth conditional distribution}\\
    =\max_{p_{\mb\theta}(\cdot|\mb c_k)}\int_{K}\log p_{\mb\theta}(\mb x|\mb c_k)\left(p(\mb x|\mb c_k)+\eta(p(\mb x|\mb c_k)-p(\mb x))\right)\der\mb x.
    \label{pseudo KL for train from scratch}
\end{align}
In what follows, we drop the subscript $k$, so that the optimization problem becomes:
\begin{align}
    \max_{p_{\mb\theta}(\cdot|\mb c)}\int_{K}\log p_{\mb\theta}(\mb x|\mb c)\left(p(\mb x|\mb c)+\eta(p(\mb x|\mb c)-p(\mb x))\right)\der\mb x\\
    =\max_{p_{\mb\theta}(\cdot|\mb c)}\int_{K}\log p_{\mb\theta}(\mb x|\mb c)h(\mb x|\mb c)\der\mb x
    \label{pseudo KL}
\end{align}
Note that~\eqref{pseudo KL} shares the similar form as KL divergence, but $h(\mb x|\mb c)$ is not a valid probability distribution, since there could exists $\mb x$ such that $h(\mb x|\mb c)<0$. In this case, setting $p_{\mb\theta}(\mb x|\mb c)=0$ makes~\eqref{pseudo KL} approaches $+\infty$, hence the optimization problem does not have an attainable optimum. To make the optimization problem well-posed, we impose Assumption~\ref{MCLR assumption 2}. 

Define the sets
\begin{align}
    K_{+}(\mb c) = \{\mb x\in K: h(\mb x|\mb c)>0\},\qquad K_{-}(\mb c)=\{\mb x\in K: h(\mb x|\mb c)\leq 0\}.
\end{align}
Under Assumption~\ref{MCLR assumption 2}, it is straightforward to show the optimal distribution $p_{\mb\theta^*}(\mb x|\mb c)$ must attain the lower bound $\delta$ for $\forall\mb x\in K_{-}(\mb c)$; otherwise increasing the density will decrease the objective~\eqref{pseudo KL}. Furthermore for $\mb x\in K^c$, where $h(\mb x|\mb c)=0$, the optimal density must satisfy $p_{\mb\theta^*}(\mb x|\mb c)=0$ for almost every $\mb x\in K^c$; otherwise, probability mass could be shifted from $K^c$ to $K_+(\mb c)$to further increase the objective~\eqref{pseudo KL}. Therefore, the remaining optimization concerns the density on $K^+(\mb c)$, which is the optimal solution to the following constrained optimization problem:
\begin{align}
    \min_{p_{\mb\theta}(\cdot|\mb c)}\int_{K_+(\mb c)}&-\log p_{\mb\theta}(\mb x|\mb c)h(\mb x|\mb c)\der\mb x \\
    \text{s.t.}\;& -p_{\mb\theta}(\mb x|\mb c)+\delta\leq 0, \;\forall \mb x\in K_+(\mb c)\\
    &\int_{K_+(\mb c)}p_{\mb\theta}(\mb x|\mb c)\der\mb x-m(\mb c)=0,
\end{align}
where
\begin{align}
    m(\mb c) = 1-\int_{K_-(\mb c)}\delta \der\mb x=1-\delta|K_{-}(\mb c)|.
\end{align}
Since $\delta<\frac{1}{|K|}$ by assumption, we have $0<m(\mb c)\leq 1$.

Note that this optimization problem is convex in $p_{\mb\theta}(\mb x|\mb c)$. Treating $p_{\mb\theta}(\mb x|\mb c)$ for each $\mb x$ as optimization variables, the constraints are affien and Slater's condition holds. Therefore strong duality applies, and the optimal solution can be characterized by the KKT conditions~\cite{boyd2004convex}.

Define the Lagrangian as:
\begin{align}
    \mathcal{L}(p_{\mb\theta}(\cdot|\mb c),\lambda,u(\cdot))=\int_{K_+(\mb c)}&-\log p_{\mb\theta}(\mb x|\mb c)h(\mb x|\mb c)\der\mb x+\lambda\bigl(\int_{K_+(\mb c)}p_{\mb\theta}(\mb x|\mb c)\der\mb x-m(\mb c)\bigl)\\&+\int_{K_+(\mb c)}u(\mb x)(-p_{\mb\theta}(\mb x|\mb c)+\delta)\der\mb x,
\end{align}
where $\lambda$ and $\mu(\cdot)$ are the dual variables. 

For each $\mb x\in K_+(\mb c)$, treating $p_{\mb\theta^*}(\mb x|\mb c)$ as a pointwise optimization variable, let $p_{\mb\theta^*}(\mb x|\mb c)$, $\lambda^*$ and $u^*(\mb x)$  be the corresponding optimal primal and dual variables. Applying stationary condition of the KKT conditions, we have:
\begin{align}
    \nabla_{p_{\mb\theta^*}(\mb x|\mb c)}\mathcal{L}(p_{\mb\theta^*}(\mb x|\mb c),\lambda^*, u^*(\mb x))=0\\
    \Rightarrow -\frac{h(\mb x|\mb c)}{p_{\mb\theta^*}(\mb x|\mb c)}+\lambda^*-u^*(\mb x)=0.
    \label{KKT stationary condition}
\end{align}

Based on~\eqref{KKT stationary condition}, we consider the following two cases:
\begin{itemize}[leftmargin=*]
    \item Suppose $u^*(\mb x)=0$, we have $p_{\mb\theta^*}(\mb x|\mb c)=\frac{h(\mb x|\mb c)}{\lambda^*}$.
    \item Suppose $u^*(\mb x)\neq 0$, by complementary slackness, we have $p_{\mb\theta^*}(\mb x|\mb c)=\delta$ and $u^*(\mb x)=\lambda^*-\frac{h(\mb x|\mb c)}{\delta}$. Applying dual feasibility $u^*(\mb x)\geq 0$, we have $h(\mb x|\mb c)\leq \lambda^*\delta$.
\end{itemize}
Note that the above two cases can be combined as:
\begin{align}
    p_{\mb\theta^*}(\mb x|\mb c)=\max\bigl\{\frac{h(\mb x|\mb c)}{\lambda^*},\delta\bigl\},\;\forall \mb x\in K_+(\mb c).
    \label{optimal solution for K}
\end{align}

Next, we prove $\lambda^*$ exists, i.e.,~\eqref{optimal solution for K} is normalizable. By applying the primal feasibility, we have:
\begin{align}
&\int_{K_+(\mb c)}p_{\mb\theta^*}(\mb x|\mb c)\der\mb x=m(\mb c)\\
    &\Rightarrow \int_{K_+(\mb c)\cap\{h(\mb x|\mb c)\leq\lambda^*\delta\}}\delta \der\mb x+\int_{K_+(\mb c)\cap\{h(\mb x|\mb c)>\lambda^*\delta)\}}\frac{h(\mb x|\mb c)}{\lambda^*}\der\mb x=m(\mb c).
\label{dual feasibility}
\end{align}
Define:
\begin{align}
    A(\lambda):=\int_{K_+(\mb c)\cap\{h(\mb x|\mb c)\leq\lambda\delta)\}}\delta \der\mb x+\int_{K_+(\mb c)\cap\{h(\mb x|\mb c)>\lambda\delta\}}\frac{h(\mb x|\mb c)}{\lambda}\der\mb x.
\end{align}
Suppose $0<\lambda_1<\lambda_2$, we have:
\begin{align}
    A(\lambda_2)-A(\lambda_1)=\int_{K_+(\mb c)\cap\{\lambda_1\delta\leq h(\mb x|\mb c)\leq\lambda_2\delta\}}\delta \der\mb x-\int_{K_+(\mb c)\cap\{\lambda_1\delta\leq h(\mb x|\mb c)\leq\lambda_2\delta\}}\frac{h(\mb x|\mb c)}{\lambda_1}\der\mb x\\
    +\int_{K_+(\mb c)\cap\{h(\mb x|\mb c)>\lambda_2\delta\}}(\frac{h(\mb x|\mb c)}{\lambda_2}-\frac{h(\mb x|\mb c)}{\lambda_1})\der\mb x\\
    =\int_{K_+(\mb c)\cap\{\lambda_1\delta\leq h(\mb x|\mb c)\leq\lambda_2\delta\}}(\delta-\frac{h(\mb x|\mb c)}{\lambda_1})\der\mb x+\int_{K_+(\mb c)\cap\{h(\mb x)>\lambda_2\delta\}}(\frac{h(\mb x|\mb c)}{\lambda_2}-\frac{h(\mb x|\mb c)}{\lambda_1})\der\mb x\\
    <0,
\end{align}
which implies $A(\lambda)$ is a monotonically decreasing function of $\lambda>0$. By the assumption $\delta<\frac{1}{|K|}$, we have $\delta |K_-(\mb c)|+\delta |K_+(\mb c)|<1$, which implies $\delta|K_+(\mb c)|<m(\mb c)$.

Since:
\begin{align}
    \lim_{\lambda\rightarrow 0}A(\lambda)=+\infty,\;\;\lim_{\lambda\rightarrow+\infty}A(\lambda)=\delta|K_+(\mb c)|,
\end{align}
as long as $\delta|K_+(\mb c)|<m(\mb c)$, by the intermediate value theorem, there exists a finite dual optimal point $\lambda^*>0$ such that $A(\lambda^*)=m(\mb c)$, i.e., the primal feasibility~\eqref{dual feasibility} holds. Hence, $p_{\mb\theta^*}(\mb x)$ in~\eqref{optimal solution for K} is a valid, normalizable probability distribution. Since the exact value of $\lambda^*$ is dependent on the specific condition $\mb c$, we replace it with the notation $Z(\mb c)$, which leads to the final result:
\begin{align}
    p_{\mb\theta^*}(\mb x|\mb c)\;=\;
    \begin{cases}
        \max \bigl\{\frac{h(\mb x|\mb c)}{Z(\mb c)},\delta\bigl\}, \;\mb x\in K_+(\mb c), \\
        \delta,\;\mb x \in K_-(\mb c),\\
        0,\; \mb x\notin K.
    \end{cases}
\end{align}
, which can be further simplified as:
\begin{align}
    p_{\mb\theta^*}(\mb x|\mb c)\;=\;
    \begin{cases}
        \max \bigl\{\frac{h(\mb x|\mb c)}{\lambda^*},\delta\bigl\}, \;\mb x\in K, \\
        0, \;\mb x\notin K.
    \end{cases}
    \label{optimal solution for MCLR appendix}
\end{align}

In the limit $\delta\rightarrow 0$, the optimal distribution approaches:
\begin{align}
    p_{\mb\theta^*}(\mb x|\mb c)=\frac{h^+(\mb x|\mb c)}{\int_{\mb x}h^+(\mb x|\mb c)\der\mb x},\qquad h^+(\mb x|\mb c):=\max\bigl\{h(\mb x|\mb c),0\bigl\}
\end{align}
which simply zeros out the negative part of $h(\mb x|\mb c)$ and renormalizes it as a valid distribution. This completes the proof.
\end{proof}

\subsection{Distributional Analysis of Fine-tuning with MCLR}
\label{sec: proof for fineutning}
Given a base model $p_\text{ref}(\mb x)$ that lacks class specificity, we may fine-tune it using MCLR combined with KL regularization:
\begin{align}
    \max_{\mb\theta}-\mathbb{E}_{\mb c}\left[D_\text{KL}\bigl(p_\text{ref}(\mb x|\mb c)||p_{\mb\theta}(\mb x|\mb c)\bigl)\right]+\eta\mathbb{E}_{\mb c,\tilde{\mb c},\mb x\sim p(\mb x|\mb c)}\left[\log \frac{p_{\mb\theta}(\mb x|\mb c)}{p_{\mb\theta}(\mb x|\tilde{\mb c})}\right].
    \label{KL+MCLR objective appendix}
\end{align}
Similar to~\eqref{independently solve for kth conditional distribution}, we can get the optimal conditional distribution $p(\mb x|\mb c_k)$ for each $k\in\{1,\ldots,M\}$ by solving the following optimization problem:
\begin{align}
   \argmax_{p_{\mb\theta}(\cdot|\mb c_k)} -D_\text{KL}\bigl(p_\text{ref}(\mb x|\mb c_k)||p_{\mb\theta}(\mb x|\mb c_k)\bigl) + \eta \mathbb{E}_{p(\mb x|\mb c_k)}\left[\log p_{\mb\theta}(\mb x|\mb c_k)\right]
    -\eta \mathbb{E}_{p(\mb x)}\left[\log p_{\mb\theta}(\mb x|\mb c_k)\right]\\
    =\argmax_{p_{\mb\theta}(\cdot|\mb c_k)}\int_{K}\log p_{\mb\theta}(\mb x|\mb c_k)\left(p_\text{ref}(\mb x|\mb c_k)+\eta(p(\mb x|\mb c_k)-p(\mb x))\right)\der\mb x.
    \label{pseudo KL problem for fine-tuning setting}
\end{align}
Note that optimization problem~\eqref{pseudo KL problem for fine-tuning setting} shares the same structure as~\eqref{pseudo KL for train from scratch}, hence by letting $h(\mb x|\mb c):= p_\text{ref}(\mb x|\mb c)+\eta (p(\mb x|\mb c)-p(\mb x))$ and under the same compact-support Assumption~\ref{MCLR assumption 1}, we get the same optimal solution as in stated in Theorem~\ref{MCLR theorem}. 

Importantly, in the fine-tuning setting, under a mixture error model, MCLR recovers the ground-truth conditional distribution, as stated in the following corollary.

\noindent\textbf{Corollary~\ref{corollary MCLR finetuning}.} \textit{If the base model satisfies the mixture error model:
\begin{align}
p_\text{ref}(\mb x|\mb c)=(1-\eta)p(\mb x|\mb c)+\eta p(\mb x),
\label{convex error model appendix}
\end{align}
then fine-tuning $p_\text{ref}(\mb x|\mb c)$ with MCLR objective~\eqref{KL+MCLR objective main text} recovers the ground-truth conditional distribution $p(\mb x|\mb c)$.} 
\begin{proof}
    Under the mixture-error model~\eqref{convex error model appendix}, the optimization problem~\eqref{pseudo KL problem for fine-tuning setting} becomes:
\begin{align}
    \argmax_{p_{\mb\theta}(\cdot|\mb c_k)}\int_{K}\log p_{\mb\theta}(\mb x|\mb c_k)p(\mb x|\mb c_k)d\mb x=\argmin_{p_{\mb\theta}(\cdot|\mb c_k)}D_\text{KL}(p(\mb x|\mb c_k)||p_{\mb\theta}(\mb x|\mb c_k)),
    \label{pseudo KL problem for fine-tuning setting under mixture-error model}
\end{align}
which has optimal solution $p_{\mb\theta}(\mb x|\mb c_k)=p(\mb x|\mb c_k)$. Note that in this case, the proof does not depend on Assumption~\ref{MCLR assumption 1} and Assumption~\ref{MCLR assumption 2}. This completes the proof.
\end{proof}
\subsection{MCLR for Diffusion Models}
\label{subsection: analysis on MCLR for diffusion models}
Thus far, we have studied MCLR at the likelihood level. Although diffusion
models induce likelihoods through their reverse processes, directly evaluating
and optimizing these likelihoods is computationally prohibitive during training. We therefore replace each expected negative-log-likelihood term in MCLR with
its diffusion ELBO surrogate~\cite{song2021maximum}, resulting in the following score-based objective:
\begin{equation}
\begin{aligned}
    &\min_{\mb s_{\mb \theta}(\cdot)}\mathbb{E}_{\mb c,t\sim p(t),\mb x\sim p(\cdot|\mb c),\mb x_t\sim p_{0t}(\cdot|\mb x)}\bigl[w(t)||\nabla_{\mb x_t}\log p_{0t}(\mb x_t|\mb x)-\mb s_{\mb\theta}(\mb x_t,t,\mb c)||_2^2\bigr]
    \\
   +&\eta \mathbb{E}_{\substack{
    \mb c,\tilde{\mb c},\,t\sim p(t)\\
    \mb x\sim p(\cdot|\mb c),\,\mb x_t\sim p_{0t}(\cdot|\mb x)\\
    \mb y\sim p(\cdot|\tilde{\mb c}),\,\mb y_t\sim p_{0t}(\cdot|\mb y)
    }}
   \Bigl[
     w(t)\Bigl(
       \|\nabla_{\mb x_t}\log p_{0t}(\mb x_t|\mb x)-\mb s_{\mb\theta}(\mb x_t,t,\mb c)\|_2^2 \\
      &\qquad\qquad\qquad\qquad\qquad\qquad\qquad
 -
       \|\nabla_{\mb y_t}\log p_{0t}(\mb y_t|\mb y)-\mb s_{\mb\theta}(\mb y_t,t,\mb c)\|_2^2
     \Bigr)
   \Bigr],
\end{aligned}
\label{the full elbo-approximated MCLR objective in appendix score-based form}
\end{equation}
and the equivalent denoising form: \begin{equation}
\begin{aligned}
    &\min_{\mathcal{D}_{\mb\theta}(\cdot)}\mathbb{E}_{\mb c,t\sim p(t),\mb x\sim p(\cdot|\mb c),\mb x_t\sim p_{0t}(\cdot|\mb x)}\bigl[w(t)||\mb x-\mathcal{D}_{\mb\theta}(\mb x_t;\sigma(t),\mb c)||_2^2\bigr]
    \\
   +&\eta \mathbb{E}_{\substack{
    \mb c,\tilde{\mb c},\,t\sim p(t)\\
    \mb x\sim p(\cdot|\mb c),\,\mb x_t\sim p_{0t}(\cdot|\mb x)\\
    \mb y\sim p(\cdot|\tilde{\mb c}),\,\mb y_t\sim p_{0t}(\cdot|\mb y)
    }}
   \Bigl[
     w(t)\Bigl(
       \|\mb x-\mathcal{D}_{\mb\theta}(\mb x_t;\sigma(t),\mb c)\|_2^2
       -
       \|\mb y-\mathcal{D}_{\mb\theta}(\mb y_t;\sigma(t),\mb c)\|_2^2
     \Bigr)
   \Bigr],
\end{aligned}
\label{the full elbo-approximated MCLR objective in appendix}
\end{equation}
where a practical time-sampling distribution $p(t)$ and weighting function $w(t)$ are used (see
\Cref{sec: ELBO denoiser form} for the detailed derivation).

It is common in the diffusion-model literature to adapt likelihood-based
objectives to score-based models through ELBO-based substitutions
~\cite{wallace2024diffusion,zheng2025direct,xue2025advantage}.
Importantly, however,~\eqref{the full elbo-approximated MCLR objective in appendix score-based form} and~\eqref{the full elbo-approximated MCLR objective in appendix}
should be viewed as a diffusion training surrogate induced by likelihood-based
MCLR, rather than, in general, an equivalent reformulation of the original
likelihood-based MCLR objective.

First, the likelihood-weighted score-matching objective provides a one-sided
upper bound on each expected negative log-likelihood. Because MCLR involves
the difference between negative-log-likelihood terms, substituting their
individual upper bounds does not itself yield a bound on the complete MCLR
objective. Second, equivalence requires the diffusion likelihood bound to be
tight. A sufficient condition given by~\cite{song2021maximum} is that there
exists a valid conditional distribution $q_0(\mb x\mid\mb c)$ whose forward
marginals $q_t(\mb x_t\mid\mb c)$ satisfy
\begin{align}
    \mb s_{\mb\theta}(\mb x_t,t,\mb c)
    \equiv
    \nabla_{\mb x_t}\log q_t(\mb x_t\mid\mb c),
    \qquad t\in[0,T],
\end{align}
together with the terminal condition
$q_T(\cdot\mid\mb c)=\pi$.
In particular, each time-indexed vector field must be conservative and induce a normalizable density, while the resulting family of distributions must be consistent
with the prescribed forward process~\eqref{forward process} through the
Fokker--Planck--Kolmogorov equation~\cite{sarkka2019applied}.
These cross-time consistency conditions are not explicitly enforced by the ELBO-approximated training objective. Moreover, when $p(t)$ and $w(t)$ differ from the time distribution and
likelihood weighting prescribed by the ELBO, the practical objective introduces
an additional departure from likelihood-based MCLR. Consequently, rather than inferring its behavior from likelihood-level MCLR,
we analyze the diffusion objective in
~\eqref{the full elbo-approximated MCLR objective in appendix score-based form} and~\eqref{the full elbo-approximated MCLR objective in appendix}
directly. Intuitively, the contrastive term encourages the denoiser conditioned
on $\mb c$ to achieve a lower denoising error on samples from the target
condition $\mb c$ than on samples from a mismatched condition
$\tilde{\mb c}$. Features shared across conditions tend to benefit both
denoising terms similarly and therefore provide limited contrastive advantage.
The MCLR regularization term instead encourages the model to exploit
structures specific to the target condition. We formalize this intuition below, first by optimizing over unconstrained score functions and then under constrained model parameterizations. Before presenting the main theorems, we introduce the following lemma, which will be useful in their proofs.
\begin{lemma}
\label{lemma1}
    Let $p(\mb x)$ be the clean data distribution, $p_t(\mb x)$ be the marginal distribution of $\mb x(t)$, and $p_{0t}(\mb x_t|\mb x)$ be the transition density from $\mb x(0)$ to $\mb x(t)$ as defined in~\Cref{subsection: basiscs of diffusion models}, then
    \begin{align}
        \nabla_{\mb x_t}\log p_t(\mb x_t)=\mathbb{E}_{p(\mb x|\mb x_t)}\bigl[\nabla_{\mb x_t}\log p_{0t}(\mb x_t|\mb x)\bigr].
    \end{align}
\end{lemma}
\begin{proof}
\begin{align}
    \nabla_{\mb x_t}\log p_t(\mb x_t)&=\frac{\nabla_{\mb x_t}p_t(\mb x_t)}{p_t(\mb x_t)}\\
    &=\frac{\nabla_{\mb x_t}\int_{\mb x}p(\mb x)p_{0t}(\mb x_t|\mb x)\der\mb x}{p_t(\mb x_t)}\\
    &=\int_{\mb x}{\frac{p(\mb x)p_{0t}(\mb x_t|\mb x)\nabla_{\mb x_t}p_{0t}(\mb x_t|\mb x)}{p_t(\mb x_t)p_{0t}(\mb x_t|\mb x)}}\der\mb x\\
    &=\int_{\mb x}p(\mb x|\mb x_t)\nabla_{\mb x_t}\log p_{0t}(\mb x_t|\mb x)\der\mb x\\
    &=\mathbb{E}_{p(\mb x|\mb x_t)}\bigl[\nabla_{\mb x_t}\log p_{0t}(\mb x_t|\mb x)\bigr].
\end{align}
This completes the proof. 
\end{proof}

\subsubsection{Parameterization-Free Analysis under a Pointwise Norm Bound}
\label{sec: parameterization-free analysis for diffusion MCLR}
We first analyze~\eqref{the full elbo-approximated MCLR objective in appendix score-based form} by treating the score function $\mb s_{\mb\theta}(\mb x_t,t,\mb c)$ as the optimization variable, without imposing any additional assumptions on the model parameterization. However, we do impose the following regularity constraint on the score function, which bounds its maximum norm:
\begin{align}
    ||\mb s_{\mb\theta}(\mb x_t,t,\mb c)||_2\leq M,
    \label{maximum norm constraint}
\end{align}
where $M$ is a positive constant. This constraint is required for the diffusion-based MCLR objective~\eqref{the full elbo-approximated MCLR objective in appendix score-based form} and~\eqref{the full elbo-approximated MCLR objective in appendix} to be well-posed so that a pointwise minimizer exists. Consider a $\mb x_t$ such that:
\begin{align}
    p_t(\mb x_t)\gg p_t(\mb x_t|\mb c),
\end{align}
i.e., $\mb x_t$ has a much larger density under the unconditional distribution compared to that under the conditional distribution. At such $\mb x_t$, the negative term in~\eqref{the full elbo-approximated MCLR objective in appendix score-based form} and~\eqref{the full elbo-approximated MCLR objective in appendix} dominates the training objective, and the objective value can be pushed arbitrarily small by simply increasing the norm of $\mb s_{\mb\theta}(\mb x_t,t,\mb c)$ or $\mathcal{D}_{\mb\theta}(\mb x_t,t,\mb c)$ to be arbitrarily large. Intuitively, MCLR relies on the positive term to prevent the degenerate behavior caused by the negative term, but at $\mb x_t$ where $p_t(\mb x_t)\gg p_t(\mb x_t|\mb c)$, the regularization effect of the positive term becomes too weak and the MCLR loss can be pushed arbitrarily small, while the score function or denoiser fails to learn any useful class-specific structure. Therefore, we impose a pointwise bound on the score magnitude. This constraint also reflects the implicit magnitude control induced in practice by finite network capacity and regularization. Next, we will present the main theorem.

\noindent\textbf{Theorem~\ref{theorem:diffusion-mclr-main}
(full statement).}
\begingroup
\itshape
Recall the signed sum-of-difference function introduced
in~\eqref{sum of difference distribution}:
\begin{align}
    h(\mb x|\mb c)
    :=
    p(\mb x|\mb c)
    +
    \eta\bigl(
    p(\mb x|\mb c)-p(\mb x)
    \bigr).
\end{align}
Define its noise-mollified counterpart induced by the diffusion
forward process~\eqref{forward process} as
\begin{align}
    h_t(\mb x_t|\mb c)
    &:=
    \int_{\mb x}
    h(\mb x|\mb c)
    p_{0t}(\mb x_t|\mb x)
    \der\mb x
    \\
    &=
    p_t(\mb x_t|\mb c)
    +
    \eta\bigl(
    p_t(\mb x_t|\mb c)-p_t(\mb x_t)
    \bigr).
\end{align}
Assume that $p_t(\mb x_t|\mb c)$ and $p_t(\mb x_t)$ are
continuously differentiable and strictly positive. Consider the
population diffusion-based MCLR objective
in~\eqref{the full elbo-approximated MCLR objective in appendix score-based form}
over the bounded-score function class
\begin{align}
    \mathcal S_M
    :=
    \left\{
    \mb s:
    \|\mb s(\mb x_t,t,\mb c)\|_2\leq M
    \text{ for all }(\mb x_t,t,\mb c)
    \right\},
    \qquad M>0.
    \label{point wise norm bound}
\end{align}

For any fixed $(\mb x_t,t,\mb c)$ satisfying
$\nabla_{\mb x_t}h_t(\mb x_t|\mb c)\neq\mb 0$, the pointwise
optimal score is
\begin{align}
    \mb s_{\mb\theta^*}(\mb x_t,t,\mb c)
    =
    \begin{cases}
        \displaystyle
        \nabla_{\mb x_t}\log h_t(\mb x_t|\mb c),
        &
        \begin{aligned}
        \text{if }\
h_t(\mb x_t\mid\mb c)>0
\ \text{and}\
\|\nabla_{\mb x_t}\log h_t(\mb x_t|\mb c)\|_2
\leq M,
        \end{aligned}
        \\[4mm]
        \displaystyle
        M
        \frac{\nabla_{\mb x_t}h_t(\mb x_t|\mb c)}
        {\|\nabla_{\mb x_t}h_t(\mb x_t|\mb c)\|_2},
        &
        \text{otherwise}.
    \end{cases}
    \label{optimal solution to the diffusion-based MCLR}
\end{align}

At a critical point satisfying
$\nabla_{\mb x_t}h_t(\mb x_t|\mb c)=\mb 0$, let
$\mathcal A^*(\mb x_t,t,\mb c)$ denote the set of pointwise
minimizers. Then
\begin{align}
    \mathcal A^*(\mb x_t,t,\mb c)
    =
    \begin{cases}
        \{\mb 0\},
        &
        \text{if }h_t(\mb x_t|\mb c)>0,
        \\[1mm]
        \{\mb s:\|\mb s\|_2\leq M\},
        &
        \text{if }h_t(\mb x_t|\mb c)=0,
        \\[1mm]
        \{\mb s:\|\mb s\|_2=M\},
        &
        \text{if }h_t(\mb x_t|\mb c)<0.
    \end{cases}
    \label{optimal solution to the diffusion-based MCLR at critical points}
\end{align}
Thus, at a nonpositive critical point, the pointwise MCLR
objective does not identify a unique score direction.

\endgroup

Before providing the proof, we first discuss the intuition behind
the optimal score field induced by diffusion-based MCLR. Recall that
the optimal distribution induced by likelihood-based MCLR, shown in~\eqref{optimal MCLR solution in the limit}, is
determined by the sum-of-difference function:
\begin{align}
    h(\mb x|\mb c)
    =
    p(\mb x|\mb c)
    +
    \eta\bigl(
    p(\mb x|\mb c)-p(\mb x)
    \bigr).
\end{align}
Consequently, likelihood-based MCLR amplifies regions where the
conditional density is large relative to the unconditional density,
while suppressing regions where the unconditional density dominates.

Although diffusion-based MCLR is not equivalent to likelihood-based
MCLR, \Cref{theorem:diffusion-mclr-main} shows
that its optimal score field is governed by the corresponding
noise-smoothed function:
\begin{align}
    h_t(\mb x_t|\mb c)
    =
    p_t(\mb x_t|\mb c)
    +
    \eta\bigl(
    p_t(\mb x_t|\mb c)-p_t(\mb x_t)
    \bigr).
\end{align}
Therefore, the two formulations share the same underlying mechanism:
both use the difference between the conditional and unconditional
densities to emphasize structures that are characteristic of the
target condition. We next explain this connection through the local
geometry of the MCLR-induced score field.

We first consider a noncritical point $(\mb x_t,t,\mb c)$ satisfying $\nabla_{\mb x_t}h_t(\mb x_t|\mb c)\neq \mb 0$. According to~\eqref{optimal solution to the diffusion-based MCLR}, the behavior of the optimal score field is determined by the sign of $h_t(\mb x_t|\mb c)$ and the norm of $\nabla_{\mb x_t}\log h_t(\mb x_t|\mb c)$. 

\paragraph{Regime 1: $h_t(\mb x_t|\mb c)>0$.} In this setting, if $||\nabla_{\mb x_t}\log h_t(\mb x_t|\mb c)||_2\leq M$, the optimal score at $\mb x_t$ is:
\begin{align}
    \mb s_{\mb\theta^*}(\mb x_t,t,\mb c)&=\nabla_{\mb x_t}\log h_t(\mb x_t|\mb c)\\ 
    &= \nabla_{\mb x_t}\log \bigl(p_t(\mb x_t|\mb c)+\eta(p_t(\mb x_t|\mb c)-p_t(\mb x_t))\bigr).
\end{align}

Since $h_t(\mb x_t|\mb c)>0$,
$\nabla_{\mb x_t}\log h_t(\mb x_t|\mb c)$ has the same
direction as $\nabla_{\mb x_t}h_t(\mb x_t|\mb c)$.
Therefore, the MCLR-induced score field locally points in the
steepest-ascent direction of $h_t$.

The difference term
$\eta(p_t(\mb x_t|\mb c)-p_t(\mb x_t))$ increases the value of
$h_t$ relative to $p_t(\mb x_t|\mb c)$ in regions where
$p_t(\mb x_t|\mb c)>p_t(\mb x_t)$, while decreasing it in regions
where $p_t(\mb x_t|\mb c)<p_t(\mb x_t)$. Consequently, the geometry
of the MCLR-induced score field favors regions in which the target
conditional density has a larger advantage over the unconditional
density. Such regions are more characteristic of the target
condition and are therefore more likely to contain strong
class-specific structures.


Note that as $h_t(\mb x_t|\mb c)\to0^+$, the norm
\begin{align}
    \|
    \nabla_{\mb x_t}\log h_t(\mb x_t|\mb c)
    \|_2
    =
    \frac{
    \|\nabla_{\mb x_t}h_t(\mb x_t|\mb c)\|_2
    }{
    h_t(\mb x_t|\mb c)
    }
\end{align}
can become arbitrarily large unless
$\nabla_{\mb x_t}h_t(\mb x_t|\mb c)$ approaches zero at a
comparable rate, and this large-norm behavior could be an ideal feature. In particular, samples at the zero-level boundary, where $h_t(\mb x_t|\mb c)=0$, satisfy:
\begin{align}
    p_t(\mb x_t)=(1+\frac{1}{\eta})p_t(\mb x_t|\mb c)>p_t(\mb x_t|\mb c).
    \label{boundary density condition}
\end{align}
Thus the boundary lies in a region where the unconditional density dominates the target-conditional density, which is not ideal as our goal is to generate samples with strong class-specific structures. MCLR without the maximum-norm constraint avoids this nonideal boundary by strongly pushing samples near the boundary toward interior regions where $h_t(\mb x_t|\mb c)$ has higher values, with a score norm potentially approaching $+\infty$. However, since practical score functions parameterized by deep networks are not likely to have an arbitrarily large norm, we add the maximum-norm constraint~\eqref{point wise norm bound}, with which, the optimal score is still in the same direction as $\nabla_{\mb x_t}\log h_t(\mb x_t|\mb c)$, but being clipped to have a finite norm $M$ if $||\nabla_{\mb x_t}\log h_t(\mb x_t|\mb c)||_2>M$, yielding:
\begin{align}
    \mb s_{\mb\theta^*}(\mb x_t,t,\mb c)=M
\frac{\nabla_{\mb x_t}h_t(\mb x_t|\mb c)}
{\|\nabla_{\mb x_t}h_t(\mb x_t|\mb c)\|_2},
\end{align}
The constraint therefore bounds the score magnitude while
preserving its steepest-ascent direction. Consequently, near a regular zero-level boundary, the bounded MCLR score
field has a local inward-pointing effect toward the region where
$h_t(\mb x_t|\mb c)>0$.

\paragraph{Regime 2: $h_t(\mb x_t|\mb c)\leq0$}. In this setting, we have:
\begin{align}
    p_t(\mb x_t)\geq(1+\frac{1}{\eta})p_t(\mb x_t|\mb c)>p_t(\mb x_t|\mb c).
    \label{boundary density condition 2}
\end{align}
Thus, the unconditional density dominates the target-conditional density by at least the factor $1+\frac{1}{\eta}$, and such points are nonideal since they are less characteristic of the target condition. At such $\mb x_t$, the optimal MCLR-induced score is:
\begin{align}
        \mb s_{\mb\theta^*}(\mb x_t,t,\mb c)=M
\frac{\nabla_{\mb x_t}h_t(\mb x_t|\mb c)}
{\|\nabla_{\mb x_t}h_t(\mb x_t|\mb c)\|_2},
\end{align}
which has a norm $M$ and points toward the steepest-ascent direction that increases $h_t(\mb x_t|\mb c)$. Thus, the MCLR-induced score locally points away from regions where the unconditional density strongly dominates the target-conditional density and toward regions where $h_t(\mb x_t|\mb c)$ has larger values. In particular, near a
regular zero-level boundary, this direction points toward the region
where $h_t(\mb x_t|\mb c)>0$. Therefore, Regimes~1 and~2 exhibit a consistent local geometry: on both sides of the boundary, the MCLR-induced score favors
directions that increase the contrastive function $h_t$, while the
norm constraint prevents its magnitude from becoming unbounded.

\paragraph{Behavior at Critical Points.}
Lastly, we consider critical points satisfying
$\nabla_{\mb x_t}h_t(\mb x_t|\mb c)=\mb 0$. At such points, the
pointwise minimizers are characterized by
\eqref{optimal solution to the diffusion-based MCLR at critical points}
and depend on the sign of $h_t(\mb x_t|\mb c)$.

When $h_t(\mb x_t|\mb c)>0$, the unique pointwise minimizer is
$\mb s_{\mb\theta^*}(\mb x_t,t,\mb c)=\mb 0$, which is consistent
with the optimal score $\nabla_{\mb x_t}\log h_t(\mb x_t|\mb c)$ in~\eqref{optimal solution to the diffusion-based MCLR}. When
$h_t(\mb x_t|\mb c)=0$, every vector satisfying
$\|\mb s\|_2\leq M$ is pointwise optimal. When
$h_t(\mb x_t|\mb c)<0$, every vector satisfying
$\|\mb s\|_2=M$ is pointwise optimal.

This degeneracy at nonpositive critical points is a natural
consequence of the local geometry underlying MCLR. The direction of
the MCLR-induced score is determined by the local gradient of $h_t$. When this gradient vanishes, $h_t$ provides no first-order information
for selecting a preferred direction. Thus, the nonuniqueness is not
a failure of the MCLR formulation, but reflects the absence of local
directional information at such critical points. In practical
implementations, the score at these points is implicitly determined
by the shared neural-network parameterization and the surrounding
training samples.

\paragraph{Main Implications.} Overall, \Cref{theorem:diffusion-mclr-main}
reveals two complementary mechanisms at noncritical points. Within
the positive region of $h_t$, MCLR augments the standard conditional
score with a positively weighted conditional--unconditional
log-likelihood-ratio gradient, thereby providing an adaptive local
guidance toward stronger evidence for the target condition. Near or
outside the zero-level boundary, the optimal score instead uses the
maximum permitted magnitude in the direction that most rapidly
increases $h_t$. Together, these two regimes provide a unified local
mechanism through which diffusion-based MCLR favors class-specific
structures and suppresses regions dominated by the unconditional
distribution. We emphasize that this discussion characterizes the local pointwise
geometry of the induced score field. A global characterization of
the resulting reverse sampling trajectories and terminal distribution
is left for future work.

Now the intuition is clear, and we are ready to provide the proof.
\begin{proof} First, note optimization problem~\eqref{the full elbo-approximated MCLR objective in appendix score-based form} is equivalent to:
\begin{align}
    \min_{\mb s_{\mb\theta}(\cdot)} \;& \;(1+\eta)\,\mathbb{E}_{\mb c, t\sim p(t),\,\mb x\sim p(\mb x|\mb c),\,\mb x_t\sim p_{0t}(\mb x_t|\mb x)}\bigl[w(t)||\nabla_{\mb x_t}\log p_{0t}(\mb x_t|\mb x)-\mb s_{\mb\theta}(\mb x_t,t,\mb c)||_2^2\bigr] \label{first term for MCLR}
    \\
    &-\eta \, \mathbb{E}_{\substack{\mb c,\tilde{\mb c},t\sim p(t)\\
\mb x\sim p(\mb x|\tilde{\mb c}),\,\mb x_t\sim p_{0t}(\mb x_t|\mb x)}}
\Bigl[
w(t)
\|
\nabla_{\mb x_t}\log p_{0t}(\mb x_t|\mb x)
-
\mb s_{\mb\theta}(\mb x_t,t,\mb c)
\|_2^2
\Bigr].\label{second term for MCLR}
\end{align}
Note that~\eqref{second term for MCLR} can be further simplified as:
\begin{align}
    -\eta \, \mathbb{E}_{\mb c,t\sim p(t),\,\mb x\sim p(\mb x),\, \mb x_t\sim p_{0t}(\mb x_t|\mb x)}
\Bigl[
w(t)
\|
\nabla_{\mb x_t}\log p_{0t}(\mb x_t|\mb x)
-
\mb s_{\mb\theta}(\mb x_t,t,\mb c)
\|_2^2
\Bigr].
\end{align}
Hence~\eqref{the full elbo-approximated MCLR objective in appendix score-based form} is equivalent to
\begin{equation}
\begin{aligned}
    \min_{\mb s_{\mb\theta}(\cdot)} \;& \;(1+\eta)\,\mathbb{E}_{\mb c, t\sim p(t),\,\mb x\sim p(\mb x|\mb c),\,\mb x_t\sim p_{0t}(\mb x_t|\mb x)}\bigl[w(t)||\nabla_{\mb x_t}\log p_{0t}(\mb x_t|\mb x)-\mb s_{\mb\theta}(\mb x_t,t,\mb c)||_2^2\bigr] 
    \\ 
    &-\eta \, \mathbb{E}_{\mb c,t\sim p(t),\,\mb x\sim p(\mb x),\, \mb x_t\sim p_{0t}(\mb x_t|\mb x)}
\Bigl[
w(t)
\|
\nabla_{\mb x_t}\log p_{0t}(\mb x_t|\mb x)
-
\mb s_{\mb\theta}(\mb x_t,t,\mb c)
\|_2^2
\Bigr].\label{simplified form 2 for MCLR}
\end{aligned} 
\end{equation}

Since the objective decomposes across $t$ and $\mb c$, the optimization can be performed independently for each pair $(t,\mb c)$. Hence, we can fix a specific pair of $t$ and $\mb c$, and optimize the corresponding score function $\mb s_{\mb \theta}(\mb x,t,\mb c)$ independently:
\begin{equation}
    \begin{aligned}
            \min_{\mb s_{\mb\theta}(\cdot,t,\mb c)} \;& \;(1+\eta)\,\mathbb{E}_{\mb x\sim p(\mb x|\mb c),\,\mb x_t\sim p_{0t}(\mb x_t|\mb x)}\bigl[||\nabla_{\mb x_t}\log p_{0t}(\mb x_t|\mb x)-\mb s_{\mb\theta}(\mb x_t,t,\mb c)||_2^2\bigr] 
    \\
    &-\eta \, \mathbb{E}_{\mb x\sim p(\mb x),\, \mb x_t\sim p_{0t}(\mb x_t|\mb x)}
\Bigl[
\|
\nabla_{\mb x_t}\log p_{0t}(\mb x_t|\mb x)
-
\mb s_{\mb\theta}(\mb x_t,t,\mb c)
\|_2^2
\Bigr]\label{simplified form 3 independent loss for MCLR}
\end{aligned}
\end{equation},

\begin{equation}
    \begin{aligned}
        \Leftrightarrow 
    \min_{\mb s_{\mb\theta}(\cdot,t,\mb c)} \;& \;(1+\eta)\,\mathbb{E}_{\mb x_t\sim p_t(\mb x_t|\mb c),\,\mb x\sim p(\mb x|\mb x_t,\mb c)}\bigl[||\nabla_{\mb x_t}\log p_{0t}(\mb x_t|\mb x)-\mb s_{\mb\theta}(\mb x_t,t,\mb c)||_2^2\bigr] 
    \\ 
    &-\eta \, \mathbb{E}_{\mb x_t\sim p_t(\mb x_t),\, \mb x\sim p(\mb x|\mb x_t)}
\Bigl[
\|
\nabla_{\mb x_t}\log p_{0t}(\mb x_t|\mb x)
-
\mb s_{\mb\theta}(\mb x_t,t,\mb c)
\|_2^2
\Bigr],\label{simplified form 4 independent loss for MCLR}
    \end{aligned}
\end{equation}
where from~\eqref{simplified form 3 independent loss for MCLR} to~\eqref{simplified form 4 independent loss for MCLR} we use the Bayes rule: 
\begin{align}
    p(\mb x|\mb c)p_{0t}(\mb x_t|\mb x)=p(\mb x,\mb x_t|\mb c)=p_t(\mb x_t|\mb c)p(\mb x|\mb x_t,\mb c),
\end{align}
and\begin{align}
        p(\mb x)p_{0t}(\mb x_t|\mb x)=p(\mb x,\mb x_t)=p_t(\mb x_t)p(\mb x|\mb x_t).
\end{align}
Note that:
\begin{equation}
\begin{aligned}
    \mathbb{E}_{\mb x\sim p(\mb x|\mb x_t,\mb c)}&\bigl[||\nabla_{\mb x_t}\log p_{0t}(\mb x_t|\mb x)-\mb s_{\mb\theta}(\mb x_t,t,\mb c)||_2^2\bigr]\\
    &=||\mb s_{\mb\theta}(\mb x_t,t,\mb c)||_2^2-2\mb s_{\mb\theta}(\mb x_t,t,\mb c)^T\mathbb{E}_{\mb x\sim p(\mb x|\mb x_t,\mb c)}\bigl[\nabla_{\mb x_t}\log p_{0t}(\mb x_t|\mb x)\bigr]+C_1\\
    &=||\mb s_{\mb\theta}(\mb x_t,t,\mb c)||_2^2-2\mb s_{\mb\theta}(\mb x_t,t,\mb c)^T\nabla_{\mb x_t}\log p_t(\mb x_t|\mb c)+C_1,\label{substep for cond for MCLR}
\end{aligned}
\end{equation}
where $C_1$ is a constant independent of $\mb \theta$, and the second equality follows by applying~\Cref{lemma1} to the conditional distribution $p(\mb x|\mb c)$. Similarly, we have
\begin{equation}
\begin{aligned}
    &\mathbb{E}_{\mb x\sim p(\mb x|\mb x_t)}\bigl[||\nabla_{\mb x_t}\log p_{0t}(\mb x_t|\mb x)-\mb s_{\mb\theta}(\mb x_t,t,\mb c)||_2^2\bigr]  \\
    =\; &||\mb s_{\mb\theta}(\mb x_t,t,\mb c)||_2^2-2\mb s_{\mb\theta}(\mb x_t,t,\mb c)^T\nabla_{\mb x_t}\log p_t(\mb x_t)+C_2,\label{substep for uncond for MCLR}
\end{aligned}
\end{equation},
where $C_2$ is a constant independent of $\mb\theta$.

Substituting~\eqref{substep for cond for MCLR} and~\eqref{substep for uncond for MCLR} into~\eqref{simplified form 4 independent loss for MCLR}, the optimization problem becomes equivalent to:
\begin{equation}
\begin{aligned}
    \min_{\mb s_{\mb\theta}(\cdot,t,\mb c)}\int_{\mb x_t}\bigl((1+\eta)p_t(\mb x_t|\mb c)-\eta p_t(\mb x_t)\bigl)||\mb s_{\mb\theta}(\mb x_t,t,\mb c)||_2^2\\
    -2\mb s_{\mb\theta}(\mb x_t,t,\mb c)^T\bigl((1+\eta)p_t(\mb x_t|\mb c)\nabla_{\mb x_t}\log p_t(\mb x_t|\mb c)-\eta p_t(\mb x_t)\nabla_{\mb x_t}\log p_t(\mb x_t)\bigr)\der \mb x_t,
\end{aligned}
\label{simplified and deriative form of MCLR}
\end{equation}
Note~\eqref{simplified and deriative form of MCLR} is minimized when every integrand is minimized, hence we fix a $\mb x_t$ and considering the following optimization problem:
\begin{equation}
\begin{aligned}
        \min_{\mb s_{\mb\theta}(\mb x_t,t,\mb c)}\bigl((1+\eta)p_t(\mb x_t|\mb c)-\eta p_t(\mb x_t)\bigl)||\mb s_{\mb\theta}(\mb x_t,t,\mb c)||_2^2\\
    -2\mb s_{\mb\theta}(\mb x_t,t,\mb c)^T\bigl((1+\eta)p_t(\mb x_t|\mb c)\nabla_{\mb x_t}\log p_t(\mb x_t|\mb c)-\eta p_t(\mb x_t)\nabla_{\mb x_t}\log p_t(\mb x_t)\bigr).
\end{aligned}
\label{simplified and deriative form of MCLR for independent xt}
\end{equation}
Note that~\eqref{simplified and deriative form of MCLR for independent xt} is a quadratic function of the optimization variable $\mb s_{\mb\theta}(\mb x_t,t,\mb c)$ and the optimal solution will depend on the sign of $(1+\eta)p_t(\mb x_t|\mb c)-\eta p_t(\mb x_t)$. Therefore, we have the following three different regimes.

\textbf{Regime 1.} Suppose:
\begin{align}
    (1+\eta)p_t(\mb x_t|\mb c)-\eta p_t(\mb x_t)>0 \Leftrightarrow h_t(\mb x_t|\mb c)>0,
\end{align}
then~\eqref{simplified and deriative form of MCLR for independent xt} is convex w.r.t. the optimization variable $\mb s_{\mb\theta}(\mb x_t,t,\mb c)$, hence the optimal solution is its stationary point:
\begin{align}
    \mb s_{\mb\theta^*}(\mb x_t,t,\mb c) &= \frac{(1+\eta)p_t(\mb x_t|\mb c)\nabla_{\mb x_t}\log p_t(\mb x_t|\mb c)-\eta p_t(\mb x_t)\nabla_{\mb x_t}\log p_t(\mb x_t)}{(1+\eta)p_t(\mb x_t|\mb c)-\eta p_t(\mb x_t)}
    \label{raw expression for the MCLR induced score function for regime 1}
    \\
    &= \nabla_{\mb x_t}\log\bigl((1+\eta)p_t(\mb x_t|\mb c)-\eta p_t(\mb x_t)\bigr)\\
    &=\nabla_{\mb x_t}\log h_t(\mb x_t|\mb c)
    \label{MCLR induced score function for regime 1}
\end{align}
Note that when $h_t(\mb x_t|\mb c)\approx 0$, the norm of~\eqref{MCLR induced score function for regime 1} can become arbitrarily large since $h_t(\mb x_t|\mb c)$ appears in the denominator of~\eqref{raw expression for the MCLR induced score function for regime 1}, which makes the score function ill-defined. Therefore, the maximum magnitude constraint~\eqref{maximum norm constraint} is required as a regularization. Under the constraint, by completing the square,~\eqref{simplified and deriative form of MCLR for independent xt} becomes equivalent to:
\begin{equation}
\begin{aligned}
        \min_{\mb s_{\mb\theta}(\mb x_t,t,\mb c)}||\mb s_{\mb\theta}(\mb x_t,t,\mb c)-\frac{(1+\eta)p_t(\mb x_t|\mb c)\nabla_{\mb x_t}\log p_t(\mb x_t|\mb c)-\eta p_t(\mb x_t)\nabla_{\mb x_t}\log p_t(\mb x_t)}{(1+\eta)p_t(\mb x_t|\mb c)-\eta p_t(\mb x_t)}||_2^2, \\
    s.t. \qquad  ||\mb s_{\mb\theta}(\mb x_t,t,\mb c)||_2\leq M,
    \\ 
    \Leftrightarrow \min_{\mb s_{\mb\theta}(\mb x_t,t,\mb c)} ||\mb s_{\mb\theta}(\mb x_t,t,\mb c)-\nabla_{\mb x_t}\log h_t(\mb x_t|\mb c)||_2^2,\qquad \\
    s.t. \qquad  ||\mb s_{\mb\theta}(\mb x_t,t,\mb c)||_2\leq M.
\end{aligned}
\label{simplified and deriative form of MCLR for independent xt with constraint}
\end{equation}
Because of Cauchy–Schwarz inequality, it is easy to see that the optimal solution to~\eqref{simplified and deriative form of MCLR for independent xt with constraint} must be of the same direction as $\nabla_{\mb x_t}\log h_t(\mb x_t|\mb c)$ but with a norm of $M$ if $||\nabla_{\mb x_t}\log h_t(\mb x_t|\mb c)||_2> M$:
\begin{align}
    \mb s_{\mb\theta^*}(\mb x_t,t,\mb c)&=M\frac{\nabla_{\mb x_t}\log h_t(\mb x_t|\mb c)}{||\nabla_{\mb x_t}\log h_t(\mb x_t|\mb c)||_2}\\
    &=M\frac{\nabla_{\mb x_t}h_t(\mb x_t|\mb c)}{||\nabla_{\mb x_t}h_t(\mb x_t|\mb c)||_2}.
\end{align}
On the other hand, if $||\nabla_{\mb x_t}\log h_t(\mb x_t|\mb c)||_2\leq M$, the optimal score function induced by MCLR is:
\begin{align}
    \mb s_{\mb\theta^*}(\mb x_t,t,\mb c)=\nabla_{\mb x_t}\log h_t(\mb x_t|\mb c).
\end{align}

Next, we consider the setting where $h_t(\mb x_t|\mb c)=0$.

\textbf{Regime 2.} Suppose:
\begin{align}
    (1+\eta)p_t(\mb x_t|\mb c)-\eta p_t(\mb x_t)=0 \Leftrightarrow h_t(\mb x_t|\mb c)=0.
\end{align}
In this case, together with the maximum magnitude constraint~\eqref{maximum norm constraint},~\eqref{simplified and deriative form of MCLR for independent xt} becomes a constrained linear optimization problem:
\begin{equation}
\begin{aligned}
    \min_{\mb s_{\mb\theta}(\mb x_t,t,\mb c)} -\mb s_{\mb\theta}(\mb x_t,t,\mb c)^T \nabla_{\mb x_t}h_t(\mb x_t|\mb c)\\
    s.t. \qquad  ||\mb s_{\mb\theta}(\mb x_t,t,\mb c)||_2\leq M
    \label{linear optimizaiton objective for regime 2}.
\end{aligned}
\end{equation}
Note that when $\nabla_{\mb x_t}h_t(\mb x_t|\mb c)\neq \mb 0$, the optimal solution to~\eqref{linear optimizaiton objective for regime 2} must be of the same direction as the linear coefficient $\nabla_{\mb x_t}h_t(\mb x_t|\mb c)$ and with norm $M$:
\begin{align}
    \mb s_{\mb\theta^*}(\mb x_t,t,\mb c)=M\frac{\nabla_{\mb x_t}h_t(\mb x_t|\mb c)}{||\nabla_{\mb x_t}h_t(\mb x_t|\mb c)||_2}.
\end{align}

On the other hand, if $\nabla_{\mb x_t}h_t(\mb x_t|\mb c)=\mb 0$, i.e., at critical point of $h_t(\mb x_t|\mb c)$, the optimal solution is not identifiable, since in this case $\mb s_{\mb\theta}(\mb x_t,t,\mb c)$ has no effect on the training objective, and its specific value can be any vector in the set $\{\mb s:||\mb s||_2\leq M\}$ such that the maximum magnitude constraint is satisfied, and its specific value will depend on the implicit regularization of the score network. Next, we consider the setting where $h_t(\mb x_t|\mb c)<0$.

\textbf{Regime 3.} Suppose:
\begin{align}
    (1+\eta)p_t(\mb x_t|\mb c)-\eta p_t(\mb x_t)<0 \Leftrightarrow h_t(\mb x_t|\mb c)<0.
\end{align}
In this case, together with the maximum magnitude constraint~\eqref{maximum norm constraint},~\eqref{simplified and deriative form of MCLR for independent xt} becomes equivalent to the following optimization problem:
\begin{equation}
\begin{aligned}
        \max_{\mb s_{\mb\theta}(\mb x_t,t,\mb c)}||\mb s_{\mb\theta}(\mb x_t,t,\mb c)-\frac{(1+\eta)p_t(\mb x_t|\mb c)\nabla_{\mb x_t}\log p_t(\mb x_t|\mb c)-\eta p_t(\mb x_t)\nabla_{\mb x_t}\log p_t(\mb x_t)}{(1+\eta)p_t(\mb x_t|\mb c)-\eta p_t(\mb x_t)}||_2^2\\
        s.t. \qquad  ||\mb s_{\mb\theta}(\mb x_t,t,\mb c)||_2\leq M,
        \\
    \Leftrightarrow \max_{\mb s_{\mb\theta}(\mb x_t,t,\mb c)} ||\mb s_{\mb\theta}(\mb x_t,t,\mb c)-\frac{\nabla_{\mb x_t}h_t(\mb x_t|\mb c)}{h_t(\mb x_t|\mb c)}||_2^2 \\
    s.t. \qquad  ||\mb s_{\mb\theta}(\mb x_t,t,\mb c)||_2\leq M
    \label{MCLR objective in Regime 3}.
\end{aligned}
\end{equation}
Note that without the maximum magnitude constraint,~\eqref{MCLR objective in Regime 3} will become unbounded above, as one can push the quadratic objective arbitrarily large by simply increasing the norm of $\mb s_{\mb\theta}(\mb x_t,t,\mb c)$. In this case MCLR becomes an ill-posed objective. The core reason behind this is that in this case the density of $\mb x_t$ under the conditional distribution is too small compared to that under the unconditional distribution, such that the negative term in MCLR dominates the overall objective, and the objective can be pushed arbitrarily small by simply increasing the norm of the score function without actually learning class-specific features. This issue is addressed with the maximum magnitude constraint, and by Cauchy–Schwarz inequality, the optimal solution to~\eqref{MCLR objective in Regime 3} must be of opposite direction as $\frac{\nabla_{\mb x_t}h_t(\mb x_t|\mb c)}{h_t(\mb x_t|\mb c)}$. Since $h_t(\mb x_t|\mb c)<0$, the optimal solution must be of the same direction as $\nabla_{\mb x_t}h_t(\mb x_t|\mb c)$:
\begin{align}
    \mb s_{\mb\theta^*}(\mb x_t,t,\mb c)=M\frac{\nabla_{\mb x_t}h_t(\mb x_t|\mb c)}{||\nabla_{\mb x_t}h_t(\mb x_t|\mb c)||_2}.
\end{align}

If $\nabla_{\mb x_t}h_t(\mb x_t|\mb c)= \mb 0$, the pointwise objective~\eqref{MCLR objective in Regime 3} becomes:
\begin{align}
    \max_{\mb s_{\mb\theta}(\mb x_t,t,\mb c)} ||\mb s_{\mb\theta}(\mb x_t,t,\mb c)||_2^2 \\
    s.t. \qquad  ||\mb s_{\mb\theta}(\mb x_t,t,\mb c)||_2\leq M.
\end{align}
Therefore, the set of minimizers is:
\begin{align}
    \left\{
\mb s:
\|\mb s\|_2=M
\right\}.
\end{align}

Combining the results from all three regimes completes the proof.
\end{proof}

\subsubsection{Extension to Fine-Tuning from a Pretrained Score Model}
\label{extension to fine-tuning from a pretrained score model with diffusion MCLR}
\Cref{theorem:diffusion-mclr-main} can be naturally extended to the fine-tuning setting where a pretrained base diffusion model denoted by $\nabla_{\mb x_t}\log p_{\text{ref},t}(\mb x_t|\mb c)$ is available, in which case the diffusion-based MCLR objective becomes:
\begin{equation}
\begin{aligned}
    &\min_{\mb s_{\mb \theta}(\cdot)}\mathbb{E}_{\mb c,t\sim p(t),\mb x\sim p_\text{ref}(\cdot|\mb c),\mb x_t\sim p_{0t}(\cdot|\mb x)}\bigl[w(t)||\nabla_{\mb x_t}\log p_{\text{ref},t}(\mb x_t|\mb c)-\mb s_{\mb\theta}(\mb x_t,t,\mb c)||_2^2\bigr]
    \\
   +&\eta \mathbb{E}_{\substack{
    \mb c,\tilde{\mb c},\,t\sim p(t)\\
    \mb x\sim p(\cdot|\mb c),\,\mb x_t\sim p_{0t}(\cdot|\mb x)\\
    \mb y\sim p(\cdot|\tilde{\mb c}),\,\mb y_t\sim p_{0t}(\cdot|\mb y)
    }}
   \Bigl[
     w(t)\Bigl(
       \|\nabla_{\mb x_t}\log p_{0t}(\mb x_t|\mb x)-\mb s_{\mb\theta}(\mb x_t,t,\mb c)\|_2^2 \\
      &\qquad\qquad\qquad\qquad\qquad\qquad\qquad
 -
       \|\nabla_{\mb y_t}\log p_{0t}(\mb y_t|\mb y)-\mb s_{\mb\theta}(\mb y_t,t,\mb c)\|_2^2
     \Bigr)
   \Bigr].
\end{aligned}
\label{the full elbo-approximated MCLR objective in appendix score-based form for finetuning}
\end{equation}
In this setting, the optimal solution to~\eqref{the full elbo-approximated MCLR objective in appendix score-based form for finetuning} shares exactly the same structure as that of~\Cref{theorem:diffusion-mclr-main}, except that the contrastive function $h_t(\mb x_t|\mb c)$ becomes:
\begin{align}
    h_t(\mb x_t|\mb c)=p_\text{ref,t}(\mb x_t|\mb c)+\eta (p_t(\mb x_t|\mb c)-p_t(\mb x_t)).
\end{align}
Consequently, the optimal score field induced by MCLR in the fine-tuning setting pushes samples toward regions with higher value under $p_{\text{ref},t}(\mb x_t|\mb c)+\eta (p_t(\mb x_t|\mb c)-p_t(\mb x_t))$. We omit the full proof since it largely follows the same structure as the proof of~\Cref{theorem:diffusion-mclr-main}. In particular, the pointwise optimization problem in~\eqref{simplified and deriative form of MCLR for independent xt} becomes:
\begin{equation}
\begin{aligned}
        \min_{\mb s_{\mb\theta}(\mb x_t,t,\mb c)}\bigl(p_{\text{ref},t}(\mb x_t|\mb c)+
        \eta p_t(\mb x_t|\mb c)-\eta p_t(\mb x_t)\bigl)||\mb s_{\mb\theta}(\mb x_t,t,\mb c)||_2^2\\
    -2\mb s_{\mb\theta}(\mb x_t,t,\mb c)^T\bigl(p_{\text{ref},t}(\mb x_t|\mb c)\nabla_{\mb x_t}\log p_\text{ref,t}(\mb x_t|\mb c)+\\ \eta p_t(\mb x_t|\mb c)\nabla_{\mb x_t}\log p_t(\mb x_t|\mb c)-\eta p_t(\mb x_t)\nabla_{\mb x_t}\log p_t(\mb x_t)\bigr),
\end{aligned}
\label{pointwise quadratic form for diffusion-based mclr fine-tuning}
\end{equation}
and the same analysis applies.

Importantly, if the base model satisfies the mixture error model assumption~\eqref{mixture error model assumption}, then fine-tuning with diffusion-based MCLR objective recovers the correct conditional distribution, as stated by the following corollary:
\begin{cor}
    If the base model \pref satisfies the mixture error model:
\begin{align}
\pref(\mb x|\mb c)=(1-\eta)\,p(\mb x|\mb c)+\eta \,p(\mb x),
\end{align}
where $\eta\in[0,1]$, then fine-tuning $\pref(\mb x|\mb c)$ with diffusion-based MCLR objective~\eqref{the full elbo-approximated MCLR objective in appendix score-based form for finetuning} recovers the ground-truth conditional
score
\begin{align}
\mb s_{\mb\theta^*}(\mb x_t,t,\mb c)
=
\nabla_{\mb x_t}\log p_t(\mb x_t|\mb c),
\qquad t\in(0,T].
\end{align}
Consequently, exact reverse sampling recovers
$p(\mb x|\mb c)$ in the limit $t\downarrow0$.
Moreover, in this case the pointwise norm bound is not required.
\label{corollary diffusion-based MCLR finetuning}
\end{cor}
\begin{proof}
Due to the linearity of forward noising process, we have:
\begin{align}
    p_{\text{ref},t}(\mb x_t|\mb c)=(1-\eta)\,p_{t}(\mb x_t|\mb c)+\eta \,p_t(\mb x_t).
    \label{ref model after forward process}
\end{align}

Substituting~\eqref{ref model after forward process} into~\eqref{pointwise quadratic form for diffusion-based mclr fine-tuning}, the pointwise optimization problem becomes:
\begin{equation}
\begin{aligned}
        \min_{\mb s_{\mb\theta}(\mb x_t,t,\mb c)}p_t(\mb x_t|\mb c)||\mb s_{\mb\theta}(\mb x_t,t,\mb c)||_2^2\\
    -2\mb s_{\mb\theta}(\mb x_t,t,\mb c)^T\bigl(\nabla_{\mb x_t}p_{\text{ref},t}(\mb x_t|\mb c)+\eta p_t(\mb x_t|\mb c)\nabla_{\mb x_t}\log p_t(\mb x_t|\mb c)-\eta p_t(\mb x_t)\nabla_{\mb x_t}\log p_t(\mb x_t)\bigr),\\
    \Leftrightarrow
        \min_{\mb s_{\mb\theta}(\mb x_t,t,\mb c)}p_t(\mb x_t|\mb c)||\mb s_{\mb\theta}(\mb x_t,t,\mb c)||_2^2-2\mb s_{\mb\theta}(\mb x_t,t,\mb c)^T\nabla_{\mb x_t}p_t(\mb x_t|\mb c),  \\  
    \Leftrightarrow
        \min_{\mb s_{\mb\theta}(\mb x_t,t,\mb c)}||\mb s_{\mb\theta}(\mb x_t,t,\mb c)||_2^2-2\mb s_{\mb\theta}(\mb x_t,t,\mb c)^T\frac{\nabla_{\mb x_t}p_t(\mb x_t|\mb c)}{p_t(\mb x_t|\mb c)},\\
    \Leftrightarrow \min_{\mb s_{\mb\theta}(\mb x_t,t,\mb c)}||\mb s_{\mb\theta}(\mb x_t,t,\mb c)||_2^2-2\mb s_{\mb\theta}(\mb x_t,t,\mb c)^T\nabla_{\mb x_t}\log p_t(\mb x_t|\mb c),
   \\ \Leftrightarrow \min_{\mb s_{\mb\theta}(\mb x_t,t,\mb c)}||\mb s_{\mb\theta}(\mb x_t,t,\mb c)-\nabla_{\mb x_t}\log p_t(\mb x_t|\mb c)||_2^2.
\end{aligned}
\label{simplified and deriative form of MCLR for independent xt for fine-tuning setting}
\end{equation}
Therefore the unique optimal solution is
\begin{align}
\mb s_{\mb\theta^*}(\mb x_t,t,\mb c)
=
\nabla_{\mb x_t}\log p_t(\mb x_t|\mb c)
\end{align}
for every $t\in(0,T]$. Consequently, reverse sampling with this score recovers
$p(\mb x|\mb c)$ as $t\downarrow0$. This completes the proof.
\end{proof}
Note that although likelihood-based MCLR and diffusion-based MCLR are not
equivalent in general, this corollary shows that both of them correct the same
mixture-error mode at their respective population optima.

\subsubsection{Connection to Classifier-Free Guidance}
\label{connection to CFG}
We now connect classifier-free guidance (CFG) to diffusion-based MCLR from two perspectives. The first perspective is that the score-field induced by diffusion-based MCLR can be interpreted as CFG with a sample-adaptive weight, while the second perspective is that the CFG-guided score is the unique minimizer of a sample-adaptive weighted ELBO-approximated MCLR objective.

\paragraph{Perspective 1: MCLR as CFG with a Sample-Adaptive Weight.} We first restrict attention to Regime~1, where
$h_t(\mb x_t|\mb c)>0$, and consider the unconstrained optimum, or
equivalently, the constrained optimum at points satisfying
$\|\nabla_{\mb x_t}\log h_t(\mb x_t|\mb c)\|_2\leq M$ in~\Cref{theorem:diffusion-mclr-main}. In this case, the MCLR-induced score is:
\begin{align}
    \nabla_{\mb x_t}\log h_t(\mb x_t|\mb c)&=\nabla_{\mb x_t}\log \bigl(p_t(\mb x_t|\mb c)+\eta(p_t(\mb x_t|\mb c)-p_t(\mb x_t))\bigr)\\
    &= \frac{(1+\eta)p_t(\mb x_t|\mb c)\nabla_{\mb x_t}\log p_t(\mb x_t|\mb c)-\eta p_t(\mb x_t)\nabla_{\mb x_t}\log p_t(\mb x_t)}{(1+\eta)p_t(\mb x_t|\mb c)-\eta p_t(\mb x_t)}\\
    &= \nabla_{\mb x_t}\log p_t(\mb x_t|\mb c)+\frac{\eta p_t(\mb x_t)}{(1+\eta)p_t(\mb x_t|\mb c)-\eta p_t(\mb x_t)}\bigl(\nabla_{\mb x_t}\log p_t(\mb x_t|\mb c)-\nabla_{\mb x_t}\log p_t(\mb x_t)\bigr)\\
    &=  \nabla_{\mb x_t}\log p_t(\mb x_t|\mb c)+\frac{1}{(1+\frac{1}{\eta})\frac{p_t(\mb x_t|\mb c)}{p_t(\mb x_t)}-1}\bigl(\nabla_{\mb x_t}\log p_t(\mb x_t|\mb c)-\nabla_{\mb x_t}\log p_t(\mb x_t)\bigr)\\
    &= \nabla_{\mb x_t}\log p_t(\mb x_t|\mb c)+\alpha_t(\eta,\mb x_t)\bigl(\nabla_{\mb x_t}\log p_t(\mb x_t|\mb c)-\nabla_{\mb x_t}\log p_t(\mb x_t)\bigr),
    \label{guidance form of diffusion-based MCLR with sample-adaptive weights}
\end{align}
where we denote $\frac{1}{(1+\frac{1}{\eta})\frac{p_t(\mb x_t|\mb c)}{p_t(\mb x_t)}-1}$ as $\alpha_t(\eta,\mb x_t)$ for notational simplicity. Since we restrict our discussion in Regime 1 where $h_t(\mb x_t|\mb c)>0$, we have $\alpha_t(\eta,\mb x_t)>0$, implying the MCLR-induced score fields add the conditional-unconditional score difference $\nabla_{\mb x_t}\log p_t(\mb x_t|\mb c)-\nabla_{\mb x_t}\log p_t(\mb x_t)$, weighted by $\alpha_t(\eta,\mb x_t)$, to the standard conditional score $\nabla_{\mb x_t}\log p_t(\mb x_t|\mb c)$, which shares a similar mechanism as classifier-free guidance (CFG). However, unlike in standard CFG where a constant guidance strength is used, here the guidance term is weighted by a sample-adaptive weight $\alpha_t(\eta,\mb x_t)$ that monotonically decreases with the likelihood-ratio $\frac{p_t(\mb x_t|\mb c)}{p_t(\mb x_t)}$. This makes sense as $\mb x_t$ with high likelihood-ratio implies that such $\mb x_t$ is less likely to be confused with a mismatched condition, i.e., this particular sample is already well-aligned with the target condition, therefore there is no need to further apply the guidance. On the other hand, a sample $\mb x_t$ with low likelihood-ratio implies such $\mb x_t$ lies in a region providing weak relative evidence for the target condition, hence the MCLR-induced score adds a large guidance to steer the generation trajectory away from such $\mb x_t$.

\paragraph{Perspective 2: CFG as MCLR with Sample-Adaptive Negative Weighting.} It turns out the CFG-guided score is the optimal solution to a sample-adaptive weighted MCLR as stated by the following theorem.

\noindent\textbf{Theorem~\ref{theorem: CFG equivalent to weighted mclr}}.
\textit{For any time sampling distribution $p(t)$ and weighting function $w(t)>0$, the CFG-guided score
\[
\mb s_\text{cfg}(\mb x_t,t,\mb c)
:=
\nabla_{\mb x_t}\log p_t(\mb x_t|\mb c)
+
\eta\,\big(
\nabla_{\mb x_t}\log p_t(\mb x_t|\mb c)
-
\nabla_{\mb x_t}\log p_t(\mb x_t)
\big)
\]
is the unique minimizer of a sample-adaptive weighted ELBO-approximated MCLR objective:
\begin{equation}
\begin{aligned}
 \mb s_\text{cfg}(\cdot)
=
\argmin_{\mb s_{\mb\theta}(\cdot)}\;
&\;\mathbb{E}_{\mb c,t\sim p(t),\,\mb x\sim p(\mb x|\mb c),\,\mb x_t\sim p_{0t}(\mb x_t|\mb x)}
\Bigl[
w(t)\|
\nabla_{\mb x_t}\log p_{0t}(\mb x_t|\mb x)
-
\mb s_{\mb\theta}(\mb x_t,t,\mb c)
\|_2^2
\Bigr]
\\
&+ \eta\,
\mathbb{E}_{\substack{\boldsymbol{c},{\tilde{\mb c}},t\sim p(t),\mb x\sim p(\cdot|\mb c),\mb y\sim p(\cdot|\tilde{\mb c})\\
p_{0t}(\boldsymbol{x}_t|\boldsymbol{x}), p_{0t}(\boldsymbol{y}_t|\boldsymbol{y})}}
\Biggl[
w(t)
\Bigl(
\|\nabla_{\boldsymbol{x}_t}\log p_{0t}(\boldsymbol{x}_t\mid\boldsymbol{x})
- \boldsymbol{s}_{\boldsymbol{\theta}}(\boldsymbol{x}_t,t,\boldsymbol{c})\|_2^2
\\
&\quad\quad\quad -\textcolor{red}{\frac{p_t(\mb y_t|\mb c)}{p_t(\mb y_t)}}
\|\nabla_{\boldsymbol{y}_t}\log p_{0t}(\boldsymbol{y}_t\mid\boldsymbol{y})
- \boldsymbol{s}_{\boldsymbol{\theta}}(\boldsymbol{y}_t,t,\boldsymbol{c})\|_2^2
\Bigr)
\Biggr],
\end{aligned}
\label{cfg equivalent optimization problem appendix}
\end{equation}}
where $p_t(\mb y_t):=\mathbb{E}_{\mb c}[p_t(\mb y_t|\mb c)]$.
\begin{proof}
First, note that optimization problem~\eqref{cfg equivalent optimization problem appendix} is equivalent to

\begin{equation}
\begin{aligned}
    \min_{\mb s_{\mb\theta}(\cdot)} \;& \;(1+\eta)\,\mathbb{E}_{\mb c, t\sim p(t),\,\mb x\sim p(\mb x|\mb c),\,\mb x_t\sim p_{0t}(\mb x_t|\mb x)}\bigl[w(t)||\nabla_{\mb x_t}\log p_{0t}(\mb x_t|\mb x)-\mb s_{\mb\theta}(\mb x_t,t,\mb c)||_2^2\bigr] 
    \\ 
    &-\eta \, \mathbb{E}_{\mb c,t\sim p(t),\,\mb x\sim p(\mb x),\, \mb x_t\sim p_{0t}(\mb x_t|\mb x)}
\Bigl[
w(t)
\frac{p_t(\mb x_t|\mb c)}{p_t(\mb x_t)}
\|
\nabla_{\mb x_t}\log p_{0t}(\mb x_t|\mb x)
-
\mb s_{\mb\theta}(\mb x_t,t,\mb c)
\|_2^2
\Bigr].\label{simplified form 2}
\end{aligned} 
\end{equation}

In order to solve~\eqref{simplified form 2}, since the objective decomposes across $t$ and $\mb c$, the optimization can be performed independently for each pair $(t,\mb c)$. Hence, we can fix a specific pair of $t$ and $\mb c$, and optimize the corresponding score $\mb s_{\mb \theta}(\mb x,t,\mb c)$ independently:
\begin{equation}
    \begin{aligned}
            \min_{\mb s_{\mb\theta}(\cdot,t,\mb c)} \;& \;(1+\eta)\,\mathbb{E}_{\mb x\sim p(\mb x|\mb c),\,\mb x_t\sim p_{0t}(\mb x_t|\mb x)}\bigl[||\nabla_{\mb x_t}\log p_{0t}(\mb x_t|\mb x)-\mb s_{\mb\theta}(\mb x_t,t,\mb c)||_2^2\bigr] 
    \\
    &-\eta \, \mathbb{E}_{\mb x\sim p(\mb x),\, \mb x_t\sim p_{0t}(\mb x_t|\mb x)}
\Bigl[
\frac{p_t(\mb x_t|\mb c)}{p_t(\mb x_t)}
\|
\nabla_{\mb x_t}\log p_{0t}(\mb x_t|\mb x)
-
\mb s_{\mb\theta}(\mb x_t,t,\mb c)
\|_2^2
\Bigr],\label{simplified form 3 independent loss}
\end{aligned}
\end{equation}

\begin{equation}
    \begin{aligned}
        \Leftrightarrow 
    \min_{\mb s_{\mb\theta}(\cdot,t,\mb c)} \;& \;(1+\eta)\,\mathbb{E}_{\mb x_t\sim p_t(\mb x_t|\mb c),\,\mb x\sim p(\mb x|\mb x_t,\mb c)}\bigl[||\nabla_{\mb x_t}\log p_{0t}(\mb x_t|\mb x)-\mb s_{\mb\theta}(\mb x_t,t,\mb c)||_2^2\bigr] 
    \\ 
    &-\eta \, \mathbb{E}_{\mb x_t\sim p_t(\mb x_t),\, \mb x\sim p(\mb x|\mb x_t)}
\Bigl[
\frac{p_t(\mb x_t|\mb c)}{p_t(\mb x_t)}
\|
\nabla_{\mb x_t}\log p_{0t}(\mb x_t|\mb x)
-
\mb s_{\mb\theta}(\mb x_t,t,\mb c)
\|_2^2
\Bigr]\label{simplified form 4 independent loss}
    \end{aligned}
\end{equation}

\begin{equation}
    \begin{aligned}
        \Leftrightarrow
    \min_{\mb s_{\mb\theta}(\cdot,t,\mb c)} \;& \;(1+\eta)\,\mathbb{E}_{\mb x_t\sim p_t(\mb x_t|\mb c),\,\mb x\sim p(\mb x|\mb x_t,\mb c)}\bigl[||\nabla_{\mb x_t}\log p_{0t}(\mb x_t|\mb x)-\mb s_{\mb\theta}(\mb x_t,t,\mb c)||_2^2\bigr] 
    \\ 
    &-\eta \, \mathbb{E}_{\mb x_t\sim p_t(\mb x_t|\mb c),\, \mb x\sim p(\mb x|\mb x_t)}
\Bigl[
\|
\nabla_{\mb x_t}\log p_{0t}(\mb x_t|\mb x)
-
\mb s_{\mb\theta}(\mb x_t,t,\mb c)
\|_2^2
\Bigr].\label{simplified form 5 independent loss}
    \end{aligned}
\end{equation}
From~\eqref{simplified form 3 independent loss} to~\eqref{simplified form 4 independent loss} we use the Bayes rule: $p(\mb x|\mb c)p_{0t}(\mb x_t|\mb x)=p(\mb x,\mb x_t|\mb c)=p_t(\mb x_t|\mb c)p(\mb x|\mb x_t,\mb c)$. From~\eqref{simplified form 4 independent loss} to~\eqref{simplified form 5 independent loss} we use the identity:  $p_t(\mb x_t)p(\mb x|\mb x_t)\frac{p_t(\mb x_t|\mb c)}{p_t(\mb x_t)}=p_t(\mb x_t|\mb c)p(\mb x|\mb x_t)$.

Note that:
\begin{equation}
\begin{aligned}
    \mathbb{E}_{\mb x\sim p(\mb x|\mb x_t,\mb c)}&\bigl[||\nabla_{\mb x_t}\log p_{0t}(\mb x_t|\mb x)-\mb s_{\mb\theta}(\mb x_t,t,\mb c)||_2^2\bigr]\\
    &=||\mb s_{\mb\theta}(\mb x_t,t,\mb c)||_2^2-2\mb s_{\mb\theta}(\mb x_t,t,\mb c)^T\mathbb{E}_{\mb x\sim p(\mb x|\mb x_t,\mb c)}\bigl[\nabla_{\mb x_t}\log p_{0t}(\mb x_t|\mb x)\bigr]+C_1\\
    &=||\mb s_{\mb\theta}(\mb x_t,t,\mb c)||_2^2-2\mb s_{\mb\theta}(\mb x_t,t,\mb c)^T\nabla_{\mb x_t}\log p_t(\mb x_t|\mb c)+C_1,\label{substep for cond}
\end{aligned}
\end{equation}
where $C_1$ is a constant independent of $\mb \theta$, and the second equality follows from~\Cref{lemma1}. Similarly, we have
\begin{equation}
\begin{aligned}
    &\mathbb{E}_{\mb x\sim p(\mb x|\mb x_t)}\bigl[||\nabla_{\mb x_t}\log p_{0t}(\mb x_t|\mb x)-\mb s_{\mb\theta}(\mb x_t,t,\mb c)||_2^2\bigr]  \\
    =\; &||\mb s_{\mb\theta}(\mb x_t,t,\mb c)||_2^2-2\mb s_{\mb\theta}(\mb x_t,t,\mb c)^T\nabla_{\mb x_t}\log p_t(\mb x_t)+C_2,\label{substep for uncond}
\end{aligned}
\end{equation},
where $C_2$ is a constant independent of $\mb\theta$.

Substituting~\eqref{substep for cond} and~\eqref{substep for uncond} into~\eqref{simplified form 5 independent loss}, and letting $\bm v_1 = (1+\eta)\nabla_{\mb x_t}\log p_t(\mb x_t|\mb c)-\eta\nabla_{\mb x_t}\log p_t(\mb x_t)$, the optimization problem becomes equivalent to:
    \begin{align}
        \min_{\mb s_{\mb\theta}(\cdot,t,\mb c)}\;\;\mathbb{E}_{\mb x_t\sim p_t(\mb x_t|\mb c)}\bigl[||\mb s_{\mb\theta}(\mb x_t,t,\mb c)||_2^2-2\mb s_{\mb\theta}(\mb x_t,t,\mb c)^T \bm v_1]\\
        \Leftrightarrow \min_{\mb s_{\mb\theta}(\cdot,t,\mb c)}\;\;\mathbb{E}_{\mb x_t\sim p_t(\mb x_t|\mb c)}\bigl[||\mb s_{\mb\theta}(\mb x_t,t,\mb c)-\mb s_\text{cfg}(\mb x_t,t,\mb c)||_2^2\bigr] + C_3,
    \end{align}
where $C_3$ is a constant independent of $\mb\theta$. Therefore, the optimal solution is:
\begin{align}
    \mb s_{\mb\theta^*}(\mb x_t,t,\mb c) &= \mb s_\text{cfg}(\mb x_t,t,\mb c)\\
    &=\nabla_{\mb x_t}\log p_t(\mb x_t|\mb c)+\eta\big(\nabla_{\mb x_t}\log p_t(\mb x_t|\mb c)-\nabla_{\mb x_t}\log p_t(\mb x_t)\big).
\end{align}
This completes the proof.
\end{proof}

Note that $\mb y_t$ is sampled from the unconditional distribution, hence a larger weight $\frac{p_t(\mb y_t|\mb c)}{p_t(\mb y_t)}$ indicates that it is more
likely to be confused with samples from the target condition $\mb c$. Hence, the adaptive weight makes the CFG-equivalent objective~\eqref{CFG equivalent MCLR form 1} more selective over negative samples: 
samples $\mb y_t$ from mismatched conditions that have high relative density under the target condition $\mb c$ receive larger weights, while samples that are already easily distinguishable from the target condition (those with small $\frac{p_t(\mb y_t|\mb c)}{p_t(\mb y_t)}$) receive smaller weights. 
Thus, the adaptive weight can be interpreted as a form of hard-negative mining, emphasizing ambiguous or class-confusable samples that provide stronger contrastive signal.

\paragraph{Extensions: CFG Variants under the Alignment Framework.}
The equivalence between CFG and weighted MCLR provides a unified perspective for interpreting CFG variants. This perspective is formalized in the following corollary.
\begin{cor}
\label{corrolary: CFG variants}
Consider two distributions $p^+(\mb x)$ and $p^-(\mb x)$. For any time sampling distribution $p(t)$ and weighting function $w(t)>0$, a generalized CFG-style score:
\begin{align}
    \mb s_\text{cfg-variant}(\mb x_t,t):=\nabla_{\mb x_t}\log p_t^+(\mb x_t) + \eta\,(\nabla_{\mb x_t}\log p_t^+(\mb x_t)-\nabla_{\mb x_t}\log p_t^-(\mb x_t))
\end{align}
is the unique minimizer of the following MCLR-style optimization problem:
\begin{equation}
\begin{aligned}
    \min_{\mb s_{\mb\theta}(\cdot)} \;& \;(1+\eta)\,\mathbb{E}_{t\sim p(t),\,\mb x\sim p^+(\mb x),\,\mb x_t\sim p_{0t}(\mb x_t|\mb x)}\bigl[w(t)||\nabla_{\mb x_t}\log p_{0t}(\mb x_t|\mb x)-\mb s_{\mb\theta}(\mb x_t,t)||_2^2\bigr] 
    \\ 
    &-\eta \, \mathbb{E}_{t\sim p(t),\,\mb x\sim p^-(\mb x),\,\mb x_t\sim p_{0t}(\mb x_t|\mb x)}
\Bigl[
w(t)
\frac{p_t^+(\mb x_t)}{p_t^-(\mb x_t)}
\|
\nabla_{\mb x_t}\log p_{0t}(\mb x_t|\mb x)
-
\mb s_{\mb\theta}(\mb x_t,t)
\|_2^2
\Bigr]\label{Optimization problem that unifies CFG variants}
\end{aligned}
\end{equation}
\end{cor}
The proof is omitted, as it follows directly by adapting the proof of~\Cref{theorem: CFG equivalent to weighted mclr}. The optimization problem~\eqref{Optimization problem that unifies CFG variants} shares the same structure as the  ELBO-approximated weighted MCLR~\eqref{simplified form 2}, with the only difference that positive samples are drawn from $p^+(\mb x)$ and negative samples from $p^-(\mb x)$. This formulation unifies a broad class of CFG-style methods, as illustrated below.

\textbf{Standard CFG.}  
    Let $p^+(\mb x)=p(\mb x|\mb c)$ and $p^-(\mb x)=p(\mb x)$. This recovers the standard CFG and corresponds to the weighted MCLR formulation discussed in previous sections.

\textbf{Autoguidance.}  
    Let $p^+(\mb x)$ be the distribution induced by a strong diffusion model and $p^-(\mb x)$ the distribution induced by a weaker model. We recover the Autoguidance~\cite{karras2024guiding}.

\textbf{Inference-Time Alignment Guidance.}  
    Let $p^+(\mb x)$ be the distribution induced by a diffusion model fine-tuned on high-reward data and $p^-(\mb x)$ the distribution induced by the base model (or a low-reward fine-tuned model). We recover a family of inference-time alignment methods~\cite{frans2025diffusion,jin2025inference,cheng2025diffusion,jiang2026rethinking}.

Lastly, note that~\eqref{Optimization problem that unifies CFG variants} without adaptive weighting is the ELBO-based approximation of the following likelihood-based objective, which is presented in~\Cref{fig:MCLR CFG Equivalence figure}:
\begin{align}
    \max_{\mb\theta}\;\mathbb{E}_{p^+(\mb x)}\big[\log p_{\mb\theta}(\mb x)\big]+\eta\,\underbrace{\mathbb{E}_{\mb x\sim p^+(\cdot),\mb y\sim p^-(\cdot)}\big[\log\frac{p_{\mb\theta}(\mb x)}{p_{\mb\theta}(\mb y)}\big]}_{\text{MCLR Regularization}}.
\end{align}
\subsubsection{Population-Level Simulation under the Mixture-Error Model}
\label{synthetic experiments}
We provide a population-level synthetic simulation to illustrate the
effects of MCLR and CFG in the diffusion-model setting. The ground-truth
conditional distributions consist of four equally likely
two-dimensional isotropic Gaussians with standard deviation $0.5$,
whose means are evenly placed on a circle of radius $2.6$. The
unconditional distribution is their uniform mixture. We construct a
weak base model satisfying the mixture-error assumption:
\begin{align}
    p_\text{ref}(\mb x|\mb c)
    =(1-\rho)p(\mb x|\mb c)+\rho p(\mb x),
\end{align}
where $\rho=0.75$.

We evaluate both methods using their exact population score fields. For
MCLR, we fix the score bound to $M=5$ and vary only the contrastive
strength $\eta$. We use the variance-exploding forward process
$\mb x_t=\mb x+12t\mb\epsilon$ and solve the corresponding
probability-flow ODE from $t=1$ to $t=0$ using $400$ Euler steps. We
generate $1{,}000$ samples per class and report the class-averaged
sliced $2$-Wasserstein distance (SWD), estimated using $256$ random
directions.

As shown in~\Cref{fig:mclr vs cfg simulation}, both MCLR and CFG reduce
the inter-class leakage of the weak base model as their respective
strengths increase. At the same time, overly strong contrastive or
guidance effects can alter the within-class geometry of the generated
distributions, demonstrating a shared trade-off between class
specificity and distributional fidelity. Under the assumed
mixture-error model, setting $\eta=\rho$ cancels the unconditional
contamination and yields the closest recovery of the target conditional
distributions.

\begin{figure}[h]
    \centering
    \includegraphics[width=0.5\linewidth]{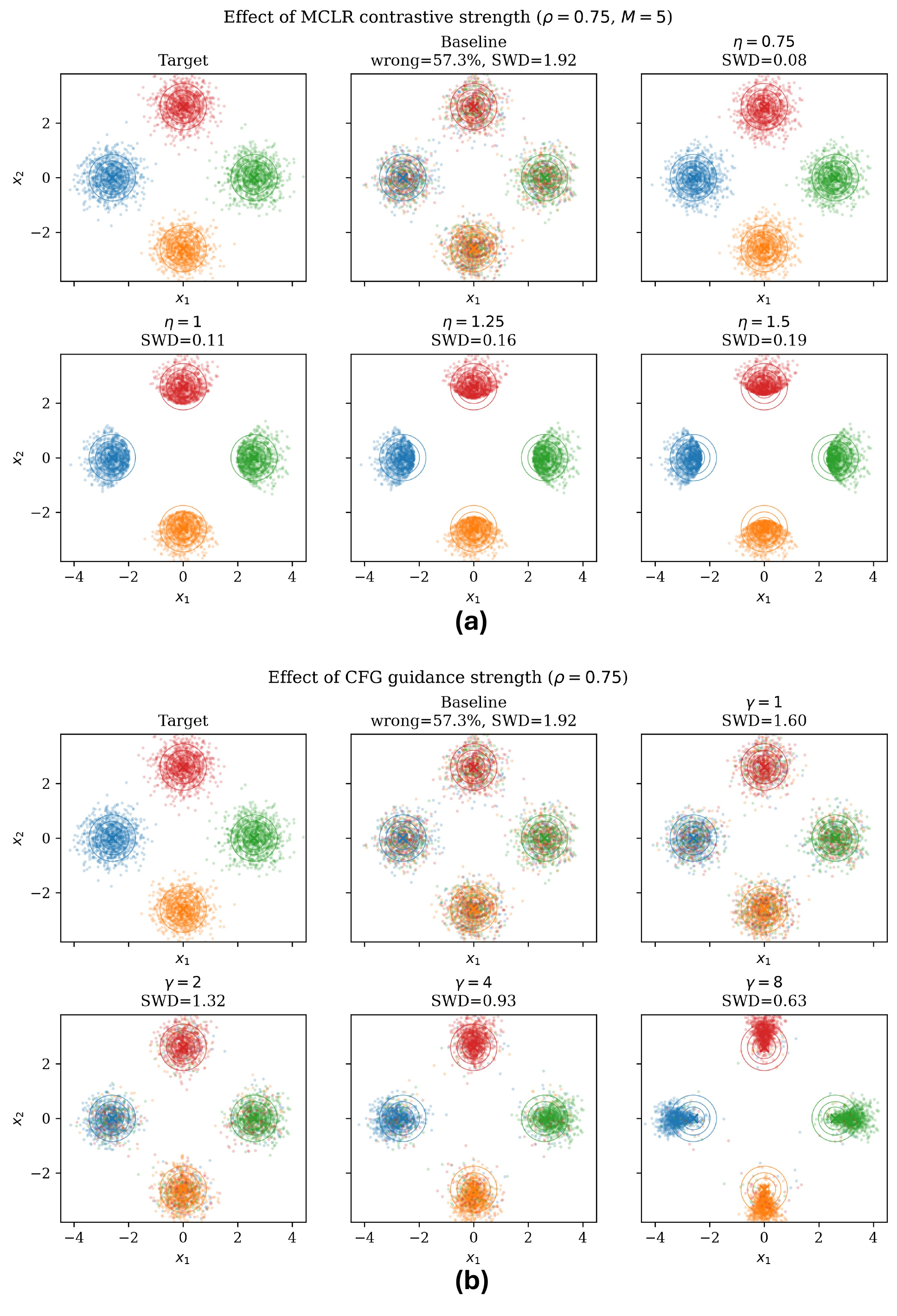}
\caption{\textbf{Population-level comparison of MCLR and CFG.}
``Target'' shows samples from the ground-truth conditional distributions,
whereas ``Baseline'' shows samples from the weak reference model under
the mixture-error setting with $\rho=0.75$. Colors denote conditioning
classes. ``wrong'' is the fraction of samples whose most likely ground-truth class differs from the conditioning class, and SWD is the class-averaged sliced
$2$-Wasserstein distance to the target conditionals (lower is better).
(a) MCLR with fixed $M=5$ and varying $\eta$. (b) CFG with varying $\gamma$. Both methods reduce
inter-class leakage, while large strengths can alter the within-class
geometry.}
\label{fig:mclr vs cfg simulation}
\end{figure}

\subsection{Comparison with Contrastive Flow Matching}
\label{sec:comparison-with-cfm}
\textbf{Basics of Flow Matching.}
In this section we study Contrastive Flow Matching~\cite{stoica2025contrastive} (CFM), a method that enhances class-specific features but with a different mechanism from our proposed MCLR. Unlike MCLR, CFM is derived in the flow matching framework. Specifically, flow matching defines the forward process as:
\begin{align}
    \mb x_t=\alpha_t\mb x+\beta_t\mb\epsilon,
    \label{stochastic interpolant}
\end{align}
where $\mb\epsilon\sim \mathcal{N}(\mb 0,\mb I)$, and aims to learn a velocity field over a probability-flow ODE, $\der\mb x_t=\mb v(\mb x_t,t)\der t$, whose distribution at time $t$ is the marginal distribution $p_t(\mb x_t)$ induced by~\eqref{stochastic interpolant}. Such a velocity field can be approximated by minimizing the following flow matching loss:
\begin{align}   \mathcal{L}_\mathrm{FM}:=\mathbb{E}_{\mb c,t,\mb x \sim p(\mb x|\mb c),\mb\epsilon\sim\mathcal{N}(\mb 0,\mb I),\mb x_t=\alpha_t\mb x+\beta_t\mb\epsilon}\bigl[||\mb v_{\mb\theta}(\mb x_t,t,\mb c)-(\dot{\alpha_t}\mb x+\dot{\beta_t}\mb\epsilon)||_2^2\bigr],
    \label{flow matching loss}
\end{align}
where $\mb v_{\mb\theta}(\mb x_t,t,\mb c)$ is known as the flow model. Define the transition probability corresponding to~\eqref{stochastic interpolant} as:
\begin{align}
    p_{0t}(\mb x_t|\mb x)=\mathcal{N}(\mb x_t;\alpha_t\mb x,\beta_t^2\mb I).
\end{align}
It can be shown that~\cite{flowsanddiffusions2026}:
\begin{align}
    \dot{\alpha_t}\mb x+\dot{\beta_t}\mb\epsilon=a_t\nabla_{\mb x_t}\log p_{0t}(\mb x_t|\mb x)+b_t\mb x_t,
        \label{conditional score conditional flow matching equivalence}
\end{align}
where:
\begin{align}
    a_t=(\beta_t^2\frac{\dot{\alpha_t}}{\alpha_t}-\dot{\beta_t}\beta_t),\qquad b_t=\frac{\dot{\alpha_t}}{\alpha_t}.
\end{align}
Therefore, by parameterizing the flow model as:
\begin{align}
    \mb v_{\mb\theta}(\mb x_t,t,\mb c) = a_t\mb s_{\mb\theta}(\mb x_t,t,\mb c)+b_t\mb x_t,
    \label{score parameterizaiton for flow models}
\end{align}
the flow matching loss~\eqref{flow matching loss} becomes equivalent to the denoising score matching loss in~\eqref{denoising score matching}. In fact, as shown in~\cite{flowsanddiffusions2026}, the velocity field is equivalent to the score function up to the same linear transformation:
\begin{align}
    \mb v(\mb x_t,t,\mb c)=a_t\nabla_{\mb x_t}\log p_t(\mb x_t|\mb c)+b_t\mb x_t,
    \label{score flow matching equiavalence}
\end{align}
hence flow matching is equivalent to denoising score matching.

\paragraph{Contrastive Flow Matching.}
Contrastive Flow Matching adds an additional contrastive term upon~\eqref{flow matching loss}, resulting in:
\begin{equation}
\begin{aligned}
\mathcal{L}_\mathrm{CFM}=\mathbb{E}_{\mb c,t,\mb x \sim p(\mb x|\mb c),\mb\epsilon_1\sim\mathcal{N}(\mb 0,\mb I),\mb x_t=\alpha_t\mb x+\beta_t\mb\epsilon_1}\bigl[||\mb v_{\mb\theta}(\mb x_t,t,\mb c)-(\dot{\alpha_t}\mb x+\dot{\beta_t}\mb\epsilon_1)||_2^2\bigr]\\
+\eta\, \mathbb{E}_{\substack{\mb c,\tilde{\mb c},t,\mb\epsilon_2\sim\mathcal{N}(\mb 0,\mb I),\mb\epsilon_3\sim\mathcal{N}(\mb 0,\mb I),\mb x\sim p(\cdot|\mb c),\mb y\sim p(\cdot|\tilde{\mb c})\\
\mb x_t=\alpha_t\mb x+\beta_t\mb\epsilon_2,\mb y_t=\alpha_t\mb y+\beta_t\mb\epsilon_3}}\bigl[||\mb v_{\mb\theta}(\mb x_t,t,\mb c)-(\dot{\alpha_t}\mb x+\dot{\beta_t}\mb\epsilon_2)||_2^2\\
-||\mb v_{\mb\theta}(\mb x_t,t,\mb c)-(\dot{\alpha_t}\mb y+\dot{\beta_t}\mb\epsilon_3)||_2^2\bigr].
\end{aligned}
\label{contrastive flow matching objective}
\end{equation}
Since we derive MCLR in terms of the score function, for better comparison, we first convert the contrastive flow matching objective into its equivalent score function form by substituting~\eqref{conditional score conditional flow matching equivalence} and ~\eqref{score parameterizaiton for flow models} into~\eqref{contrastive flow matching objective}, resulting in the following optimization problem:
\begin{equation}
\begin{aligned}
\min_{\mb\theta}\mathcal{L}_\mathrm{CFM}:=\mathbb{E}_{\mb c,t,\mb x\sim p(\mb x|\mb c),p_{0t}(\mb x_t|\mb x)}\bigl[||a_t\mb s_{\mb\theta}(\mb x_t,t,\mb c)-a_t\nabla_{\mb x_t}\log p_{0t}(\mb x_t|\mb x)||_2^2\bigr]\\
    +\eta\,\mathbb{E}_{\substack{\mb c,\tilde{\mb c},t,\mb x\sim p(\cdot|\mb c),\mb y\sim p(\cdot|\tilde{\mb c})\\p_{0t}(\mb x_t|\mb x),p_{0t}(\mb y_t|\mb y)}}\bigl[||a_t\mb s_{\mb\theta}(\mb x_t,t,\mb c)-a_t\nabla_{\mb x_t}\log p_{0t}(\mb x_t|\mb x)||_2^2\\-||a_t\mb s_{\mb\theta}(\mb x_t,t,\mb c)-a_t\nabla_{\mb y_t}\log p_{0t}(\mb y_t|\mb y)+b_t(\mb x_t-\mb y_t)||_2^2\bigr],\\
\Leftrightarrow
\min_{\mb\theta}\mathbb{E}_{\mb c,t,\mb x\sim p(\mb x|\mb c),p_{0t}(\mb x_t|\mb x)}\bigl[||\mb s_{\mb\theta}(\mb x_t,t,\mb c)-\nabla_{\mb x_t}\log p_{0t}(\mb x_t|\mb x)||_2^2\bigr]\\
    +\eta\,\mathbb{E}_{\substack{\mb c,\tilde{\mb c},t,\mb x\sim p(\cdot|\mb c),\mb y\sim p(\cdot|\tilde{\mb c})\\p_{0t}(\mb x_t|\mb x),p_{0t}(\mb y_t|\mb y)}}\bigl[||\mb s_{\mb\theta}(\mb x_t,t,\mb c)-\nabla_{\mb x_t}\log p_{0t}(\mb x_t|\mb x)||_2^2\\-||\mb s_{\mb\theta}(\mb x_t,t,\mb c)-\nabla_{\mb y_t}\log p_{0t}(\mb y_t|\mb y)+\frac{b_t}{a_t}(\mb x_t-\mb y_t)||_2^2\bigr].
    \label{CFM in score function form}
\end{aligned}
\end{equation}
From~\eqref{CFM in score function form}, we see that CFM is similar to MCLR in~\eqref{the full elbo-approximated MCLR objective in appendix score-based form}, as both add an additional negative term to the standard denoising score matching objective. However, unlike MCLR, the negative term in CFM does not directly correspond to a score matching target. To better understand CFM's effect, we derive the optimal solution to $\mathcal{L}_\mathrm{CFM}$ in the following theorem.
\begin{theorem}
\label{optimal solution to CFM}
The optimal solution to the CFM objective~\eqref{CFM in score function form} is:
\begin{align}
    \mb s_{\mb\theta^*}(\mb x_t,t,\mb c)=(1+\eta)\nabla_{\mb x_t}\log p_t(\mb x_t|\mb c)+\eta \frac{b_t}{a_t}(\mb x_t-\mathbb{E}_{p_t(\mb x_t)}[\mb x_t]).
    \label{optimal solution for CFM}
\end{align}
\end{theorem}
\begin{proof}
Note that:
\begin{align}
\mathbb{E}_{\mb x\sim p(\mb x|\mb c),p_{0t}(\mb x_t|\mb x)}\bigl[||\mb s_{\mb\theta}(\mb x_t,t,\mb c)-\nabla_{\mb x_t}\log p_{0t}(\mb x_t|\mb x)||_2^2\bigr]\\
=\mathbb{E}_{\mb x_t\sim p_t(\mb x_t|\mb c)}\bigl[||\mb s_{\mb\theta}(\mb x_t,t,\mb c)||_2^2-2\mb s_{\mb\theta}(\mb x_t,t,\mb c)^T\nabla_{\mb x_t}\log p_t(\mb x_t|\mb c)\bigr]+C_1,
\label{idk how to name this 3}
\end{align}
where $C_1$ is a constant independent of $\mb\theta$.

Moreover, we have:
\begin{align}
    \mathbb{E}_{\mb x\sim p(\cdot|\mb c),\mb y\sim p(\cdot),p_{0t}(\mb x_t|\mb x),p_{0t}(\mb y_t|\mb y)}\bigl[||\mb s_{\mb\theta}(\mb x_t,t,\mb c)-\nabla_{\mb y_t}\log p_{0t}(\mb y_t|\mb y)+\frac{b_t}{a_t}(\mb x_t-\mb y_t)||_2^2\bigr]\\
    =\mathbb{E}_{\mb x\sim p(\cdot|\mb c),\mb y\sim p(\cdot),p_{0t}(\mb x_t|\mb x),p_{0t}(\mb y_t|\mb y)}\bigl[||\mb s_{\mb\theta}(\mb x_t,t,\mb c)||_2^2-2\mb s_{\mb\theta}(\mb x_t,t,\mb c)^T\nabla_{\mb y_t}\log p_{0t}(\mb y_t|\mb y)\\
    +2\frac{b_t}{a_t}\mb s_{\mb\theta}(\mb x_t,t,\mb c)^T(\mb x_t-\mb y_t)\bigr]+C_2\\
    =\mathbb{E}_{\mb x_t\sim p_t(\mb x_t|\mb c),\mb y_t \sim p_t(\mb y_t)}\bigl[||\mb s_{\mb\theta}(\mb x_t,t,\mb c)||_2^2-2\mb s_{\mb\theta}(\mb x_t,t,\mb c)^T\nabla_{\mb y_t}\log p_t(\mb y_t)\label{idk how to name this}\\ 
    +2\frac{b_t}{a_t}\mb s_{\mb\theta}(\mb x_t,t,\mb c)^T(\mb x_t-\mb y_t)\bigr]+C_2\\
    =\mathbb{E}_{\mb x_t\sim p_t(\mb x_t|\mb c)}\bigl[||\mb s_{\mb\theta}(\mb x_t,t,\mb c)||_2^2+2\frac{b_t}{a_t}\mb s_{\mb\theta}(\mb x_t,t,\mb c)^T(\mb x_t-\mathbb{E}_{p_t(\mb y_t)}[\mb y_t])\bigr]+C_2\label{idk how to name this 2},
\end{align}
where $C_2$ is a constant independent of $\mb \theta$ and the second term in~\eqref{idk how to name this} vanishes since:
\begin{align}
    \mathbb{E}_{\mb y_t\sim p_t(\mb y_t)}\bigl[\nabla_{\mb y_t}\log p_t(\mb y_t)\bigr]&=\int_{\mb y_t}\nabla_{\mb y_t}p_t(\mb y_t)\der \mb y_t\\&=0,
\end{align}
where the last equality follows from integration by parts applied
coordinatewise, under the standard assumption that $p_t$ decays
sufficiently fast at infinity.
Substituting~\eqref{idk how to name this 3} and~\eqref{idk how to name this 2} into~\eqref{CFM in score function form}, CFM becomes equivalent to the following optimization problem:
\begin{align}
   \min_{\mb\theta} \mathbb{E}_{\mb c,t,\mb x_t\sim p_t(\mb x_t|\mb c)}\bigl[||\mb s_{\mb\theta}(\mb x_t,t,\mb c)||_2^2-2\bigl((1+\eta)\nabla_{\mb x_t}\log p_t(\mb x_t|\mb c)+\eta\frac{b_t}{a_t}(\mb x_t-\mathbb{E}_{p_t(\mb y_t)}[\mb y_t])\bigr)^T\mb s_{\mb\theta}(\mb x_t,t,\mb c)\bigr],
\end{align}
hence the optimal solution is:
\begin{align}
        \mb s_{\mb\theta^*}(\mb x_t,t,\mb c)=(1+\eta)\nabla_{\mb x_t}\log p_t(\mb x_t|\mb c)+\eta \frac{b_t}{a_t}(\mb x_t-\mathbb{E}_{p_t(\mb x_t)}[\mb x_t]).
\end{align}
This completes the proof.
\end{proof}
Theorem~\ref{optimal solution to CFM} shows that the population-optimal
CFM score consists of a rescaled conditional score and a global affine
correction determined by the displacement from the unconditional noisy
mean. Specifically, the correction has the form
$\eta(b_t/a_t)(\mb x_t-\mb\mu_t)$, where
$\mb\mu_t:=\mathbb{E}_{p_t(\mb x_t)}[\mb x_t]$.

This mechanism differs from both CFG and MCLR. The CFG correction is
determined by the local conditional--unconditional score difference,
whereas the MCLR correction additionally contains a sample-adaptive
coefficient determined by the local likelihood ratio. In contrast, the
CFM correction depends only on the displacement from the global
unconditional mean and a time-dependent coefficient. Hence, CFM does
not explicitly adapt its correction to the local conditional--
unconditional density geometry.



To illustrate this distinction, consider a well-separated Gaussian
mixture and a noise schedule satisfying $b_t/a_t>0$. In this example,
the CFM correction points away from the unconditional mean whenever
$\mb x_t\neq\mb\mu_t$. Consequently, the correction may remain
nonzero even when $\mb x_t$ lies in a high-density, class-unambiguous
region. In contrast, the CFG correction can become small when the
unconditional mixture score is dominated by the target component,
while MCLR additionally attenuates its adaptive coefficient as the
conditional--unconditional likelihood ratio increases. Consistent with these distinct correction mechanisms, our controlled experiments in~\Cref{empirical comparison between MCLR and CFM} show that MCLR yields larger and more sustained improvements than CFM across different regularization strengths.

\section{Theoretical Analysis of CC-DPO}
\subsection{Basics of DPO}
\label{basics of dpo appendix}
\paragraph{Reward Modeling.}
For a given prompt $\mb c$ and two associated outputs $\mb x_{w}$ and $\mb x_{l}$, where the subscripts '$w$' stands for 'winning' while '$l$' stands for 'losing', implying that $\mb x_{w}$ is preferred over $\mb x_{l}$, DPO models the human preference distribution with the Bradley-Terry (BT) model~\cite{bradley1952rank}:
\begin{align}
    p(\mb x_{w}\succ\mb x_{l}|\mb c)=\frac{\exp(r^*(\mb x_w|\mb c))}{\exp(r^*(\mb x_w|\mb c))+\exp(r^*(\mb x_l|\mb c))}=\mathrm{Sigmoid}(r^*(\mb x_w|\mb c)-r^*(\mb x_l|\mb c))
    \label{BT model}
\end{align}
where $r^*(\mb x|\mb c)$ is the underlying optimal reward function for prompt $\mb c$. Intuitively, the preferred samples $\mb x_w$ should have higher reward values compared to the non-preferred samples $\mb x_l$. Assuming access to a dataset of sampled comparisons $\mathcal{S}=\{(\mb c^{i}, \mb x_w^{i}, \mb x_l^{i})\}_{i=1}^N$, one can learn the optimal reward function via maximum likelihood estimation:
\begin{align}
    r^*&=\argmax_{r}\;\mathbb{E}_{(\mb c, \mb x_w, \mb x_l)\sim \mathcal{S}}\bigl[\log p(\mb x_w\succ\mb x_l|\mb c)\bigl]\\
    &=\argmin_{r}\; -\mathbb{E}_{(\mb c,\mb x_w,\mb x_l)\sim S}\bigl[\log\mathrm{Sigmoid}(r(\mb x_{w}|\mb c)-r(\mb x_{l}|\mb c))\bigl].
    \label{Reward modeling objective}
\end{align}

\paragraph{RL fine-tuning Phase.} Assuming access to the optimal preference reward function $r^*$, one can fine-tune a base model $p_\text{ref}(\mb x|\mb c)$ to align with the preference dataset by optimizing the following objective: 
\begin{align}
    \max_{\mb\theta}\;\mathbb{E}_{\mb c}\bigl[\mathbb{E}_{\mb x\sim p_{\mb\theta}(\mb x|\mb c)}\bigl[r^*(\mb x|\mb c)\bigl]-\beta D_{KL}(p_{\mb\theta}(\mb x|\mb c)||p_\text{ref}(\mb x|\mb c))\bigl],
    \label{RL fine-tuning objective}
\end{align}
where $\beta$ controls the KL-regularization strength.
Intuitively, objective~\eqref{RL fine-tuning objective} encourages the fine-tuned model to achieve high reward value in expectation, and at the same time not deviate too much from the base model. 

\paragraph{DPO Objective.} Under certain regularity conditions, the optimal solution to the fine-tuning objective~\eqref{RL fine-tuning objective} admits the following closed-form:
\begin{align}
    p_{\mb\theta^*}(\mb x|\mb c)=\frac{1}{Z(\mb c)}p_\text{ref}(\mb x|\mb c)\exp{(\frac{1}{\beta}r^*(\mb x|\mb c))},
    \label{optimal policy}
\end{align}
where $Z(\mb c)$ is a partition function for normalizing the density. To prove this, consider optimizing the fine-tuning objective~\eqref{RL fine-tuning objective} for a target condition $\mb c$:
\begin{align}
        \argmax_{\mb\theta}\;\mathbb{E}_{\mb x\sim p_{\mb\theta}(\mb x|\mb c)}\bigl[r^*(\mb x|\mb c)\bigl]&-\beta D_{KL}(p_{\mb\theta}(\mb x|\mb c)||p_\text{ref}(\mb x|\mb c))\label{fine-tuning objective with reward and kl}\\
        &=\argmax_{\mb\theta}\mathbb{E}_{\mb x\sim p_{\mb\theta}(\mb x|\mb c)}\bigl[r^*(\mb x|\mb c)-\beta\log\frac{p_{\mb\theta}(\mb x|\mb c)}{p_\text{ref}(\mb x|\mb c)}\bigl]\\
        &=\argmin_{\mb\theta}\mathbb{E}_{\mb x\sim p_{\mb\theta}(\mb x|\mb c)}\bigl[\log\frac{p_{\mb\theta}(\mb x|\mb c)}{p_\text{ref}(\mb x|\mb c)}-\frac{1}{\beta}r^*(\mb x|\mb c)\bigl]\\
        &=\argmin_{\mb\theta}\mathbb{E}_{\mb x\sim p_{\mb\theta}(\mb x|\mb c)}\bigl[\log\frac{p_{\mb\theta}(\mb x|\mb c)}{\frac{1}{Z(\mb c)}p_\text{ref}(\mb x|\mb c)\exp(\frac{1}{\beta}r^*(\mb x|\mb c))}-\log Z(\mb c)\bigl]\label{intermediate step, KL for DPO},
\end{align}
where $Z(\mb c)$ is the partition function that normalizes $p_\text{ref}(\mb x|\mb c)\exp(\frac{1}{\beta}r^*(\mb x|\mb c))$:
\begin{align}
    Z(\mb c)=\int_{\mb x}p_\text{ref}(\mb x|\mb c)\exp(\frac{1}{\beta}r^*(\mb x|\mb c))\der\mb x,
    \label{DPO partition function}
\end{align}
such that  $p^*(\mb x|\mb c):=\frac{1}{Z(\mb c)}p_\text{ref}(\mb x|\mb c)\exp(\frac{1}{\beta}r^*(\mb x|\mb c))$ is a valid probability distribution. Since $Z(\mb c)$ is independent of $\mb\theta$, the fine-tuning objective~\eqref{intermediate step, KL for DPO} is equivalent to:
\begin{align}
    \min_{\mb\theta} D_\text{KL}(p_{\mb\theta}(\mb x|\mb c)||p^*(\mb x|\mb c)),
\end{align}
which achieves its minimum value $0$ if and only if:
\begin{align}
p_{\mb\theta}(\mb x|\mb c)=p^*(\mb x|\mb c)=\frac{1}{Z(\mb c)}p_\text{ref}(\mb x|\mb c)\exp(\frac{1}{\beta}r^*(\mb x|\mb c)).
\label{optimal fine-tuned distribution under DPO}
\end{align}

With some algebra, the optimal reward function can be expressed with $p_{\mb\theta^*}(\mb x|\mb c)$:
\begin{align}
    r^*(\mb x|\mb c)=\beta\log\frac{p_{\mb\theta^*}(\mb x|\mb c)}{p_\text{ref}(\mb x|\mb c)}+\beta\log Z(\mb c).\label{relation between optimal reward and optimal distribution}
\end{align}
The relationship~\eqref{relation between optimal reward and optimal distribution} between the optimal reward and the optimal fine-tuned distribution suggests a convenient parameterization of the reward model:
\begin{align}
    r_{\mb\theta}(\mb x|\mb c)=\beta\log\frac{p_{\mb\theta}(\mb x|\mb c)}{p_\text{ref}(\mb x|\mb c)}.\label{model parameterization}
\end{align}
Substituting~\eqref{model parameterization} into~\eqref{Reward modeling objective} results in the DPO objective:
\begin{align}
    \argmin_{\mb\theta} \; -\mathbb{E}_{(\mb c,\mb x_w,\mb x_l)\sim S}\bigl[\log\mathrm{Sigmoid}(\beta \log\frac{p_{\mb\theta}(\mb x_w|\mb c)}{p_\text{ref}(\mb x_w|\mb c)}-\beta \log\frac{p_{\mb\theta}(\mb x_l|\mb c)}{p_\text{ref}(\mb x_l|\mb c)})\bigl].
    \label{DPO objective}
\end{align}
In this way, one can directly fine-tune the base model without explicitly modeling the reward.

\subsection{Improving Conditional Modeling with CC-DPO}
\label{sec: proof of CC-DPO optimal solution}
To adapt DPO for improving class specificity, we may treat samples from the target class $\mb c$ as preferred data ($\mb x_w$) and samples from other randomly selected classes as non-preferred data ($\mb x_l$). This leads to the following objective:
\begin{align}
        \min_{\mb\theta} \; -\mathbb{E}_{\mb c,\tilde{\mb c}, \mb x_w\sim p(\mb x|\mb c),\mb x_l\sim p(\mb x|\mb \tilde{\mb c})}\bigl[\log\mathrm{Sigmoid}(\beta \log\frac{p_{\mb\theta}(\mb x_w|\mb c)}{p_\text{ref}(\mb x_w|\mb c)}-\beta \log\frac{p_{\mb\theta}(\mb x_l|\mb c)}{p_\text{ref}(\mb x_l|\mb c)})\bigl]\\
        =\min_{\mb\theta} \; -\mathbb{E}_{\mb c, \mb x_w\sim p(\mb x|\mb c),\mb x_l\sim p(\mb x)}\bigl[\log\mathrm{Sigmoid}(\beta \log\frac{p_{\mb\theta}(\mb x_w|\mb c)}{p_\text{ref}(\mb x_w|\mb c)}-\beta \log\frac{p_{\mb\theta}(\mb x_l|\mb c)}{p_\text{ref}(\mb x_l|\mb c)})\bigl],
    \label{DPO objective-conditional appendix}
\end{align}
which admits a closed-form solution as stated in the following theorem.

\noindent\textbf{Theorem~\ref{CC-DPO theorem 2}.}\textit{
Under certain regularity conditions, the optimal solution to~\eqref{DPO objective-conditional appendix} is:
\begin{align}
    p_{\mb\theta^*}(\mb x|\mb c)=\frac{1}{\tilde{Z}(\mb c)}p_\text{ref}(\mb x|\mb c)\left(\frac{p(\mb x|\mb c)}{p(\mb x)}\right)^{\frac{1}{\beta}},
    \label{optimal CC-DPO solution appendix}
\end{align}
where $\tilde{Z}(\mb c)=\int_{\mb x}p_\text{ref}(\mb x|\mb c)\left(\frac{p(\mb x|\mb c)}{p(\mb x)}\right)^{\frac{1}{\beta}}d\mb x$ is the normalizing constant.} 
\begin{proof}
    From~\eqref{Reward modeling objective} to~\eqref{optimal policy}, it is clear that the optimal solution to CC-DPO~\eqref{DPO objective-conditional appendix} is fully determined by the base model and the optimal solution to the following reward modeling objective:
 \begin{align}
    r^*=\argmin_{r}\; -\mathbb{E}_{\mb c,\mb x_w\sim p(\mb x|\mb c),\mb x_l\sim p(\mb x)}\bigl[\log\mathrm{Sigmoid}(r(\mb x_{w}|\mb c)-r(\mb x_{l}|\mb c))\bigl].
    \label{Reward modeling objective-conditional-1}
\end{align}
Note that $r^*$ is the collection of optimal rewards $r^*(\cdot|\mb c_k)$ for each $\mb c_k\in\{\mb c_i\}_{i=1}^M$. Without loss of generality, we drop the subscript and solve for the optimal reward for a single target $\mb c$:
\begin{align}
    r^*(\cdot|\mb c)&=\argmin_{r(\cdot|\mb c)}\; -\mathbb{E}_{\mb x_w\sim p(\mb x|\mb c),\mb x_l\sim p(\mb x)}\bigl[\log\mathrm{Sigmoid}(r(\mb x_{w}|\mb c)-r(\mb x_{l}|\mb c))\bigl]
    \label{Reward modeling objective-conditional}\\
    &=\argmax_{r(\cdot|\mb c)}\int_{\mb x_w}\int_{\mb x_l}\log\mathrm{Sigmoid}(r(\mb x_w|\mb c)-r(\mb x_l|\mb c))p(\mb x_w|\mb c)p(\mb x_l)\der\mb x_w \der\mb x_l.
    \label{Reward modeling objective-condtional integral}
\end{align}
Note that for any arbitrary pair of points $(\mb x_1,\mb x_2)$, they contribute to the integral~\eqref{Reward modeling objective-condtional integral} for the following amount:
\begin{align}
    \log\mathrm{Sigmoid}\left(r(\mb x_1|\mb c)-r(\mb x_2|\mb c)\right)p(\mb x_1|\mb c)p(\mb x_2)+\log\mathrm{Sigmoid}\left(r(\mb x_2|\mb c)-r(\mb x_1|\mb c)\right)p(\mb x_2|\mb c)p(\mb x_1),
    \label{reward modeling objective-conditional integrand}
\end{align}
hence $r^*(\cdot|\mb c)$ is an optimal solution to~\eqref{Reward modeling objective-condtional integral} if it maximizes~\eqref{reward modeling objective-conditional integrand} for $\forall (\mb x_1,\mb x_2)$. To find such $r^*(\cdot|\mb c)$, let's define $w=r(\mb x_1|\mb c)-r(\mb x_2|\mb c)$ and solve the following optimization problem:
\begin{align}
    \max_{\mb w}\;\mathcal{L}(w):=\log\mathrm{Sigmoid}(w)p(\mb x_1|\mb c)p(\mb x_2)+\log\mathrm{Sigmoid}(-w)p(\mb x_2|\mb c)p(\mb x_1).
\end{align}
Note that:
\begin{align}
    \nabla_{w}\mathcal{L}(w)=\mathrm{Sigmoid}(-w)p(\mb x_1|\mb c)p(\mb x_2)-\mathrm{Sigmoid}(w)p(\mb x_2|\mb c)p(\mb x_1),
\end{align}
which implies the stationary point $w^*$ must satisfy:
\begin{align}
    \mathrm{Sigmoid}(w^*)p(\mb x_2|\mb c)p(\mb x_1)=\mathrm{Sigmoid}(-w^*)p(\mb x_1|\mb c)p(\mb x_2)\\
    \Rightarrow \mathrm{Sigmoid}(w^*)\bigl(p(\mb x_2|\mb c)p(\mb x_1)+p(\mb x_1|\mb c)p(\mb x_2)\bigl)=p(\mb x_1|\mb c)p(\mb x_2)\label{intermdiate step 1}\\
    \Rightarrow \mathrm{Sigmoid}(w^*)=\frac{p(\mb x_1|\mb c)p(\mb x_2)}{p(\mb x_2|\mb c)p(\mb x_1)+p(\mb x_1|\mb c)p(\mb x_2)}\label{intermdiate step 2}\\
    \Rightarrow \mathrm{Sigmoid}(w^*)=\frac{1}{1+\frac{p(\mb x_2|\mb c)p(\mb x_1)}{p(\mb x_1|\mb c)p(\mb x_2)}}\label{intermediate step 3}\\
    \Rightarrow \mathrm{Sigmoid}(w^*)=\frac{1}{1+\exp{\bigl(-\log\frac{p(\mb x_1|\mb c)p(\mb x_2)}{p(\mb x_1)p(\mb x_2|\mb c)}\bigl)}}\\
    \Rightarrow w^*=\log\frac{p(\mb x_1|\mb c)p(\mb x_2)}{p(\mb x_1)p(\mb x_2|\mb c)},
\end{align}
which implies:
\begin{align}
    r^*(\mb x|\mb c)=\log\frac{p(\mb x|\mb c)}{p(\mb x)}+K,
    \label{optimal reward for DPO}
\end{align}
where $K$ is any finite constant.
Moreover, since:
\begin{align}
    \nabla^2_{w}\mathcal{L}(w)&=-\mathrm{Sigmoid}(w)\mathrm{Sigmoid}(-w)p(\mb x_1|\mb c)p(\mb x_2)-\mathrm{Sigmoid}(-w)\mathrm{Sigmoid}(w)p(\mb x_2|\mb c)p(\mb x_1)\\
    &<0,
\end{align}
we know $\mathcal{L}(w)$ is concave and $w^*$ is the global maximizer and consequently $r^*(\mb x|\mb c)$ is the optimal reward function. Since this optimal reward function is consistent for any arbitrary pair of points $(\mb x_1,\mb x_2)$, it is the global reward function that maximizes the full objective~\eqref{Reward modeling objective-conditional}.

Substituting~\eqref{optimal reward for DPO} into~\eqref{optimal policy}, we obtain the optimal solution to CC-DPO:
\begin{align}
    p_{\mb\theta^*}(\mb x|\mb c)=\frac{1}{\tilde{Z}(\mb c)}p_\text{ref}(\mb x|\mb c)\left(\frac{p(\mb x|\mb c)}{p(\mb x)}\right)^{\frac{1}{\beta}},
    \label{CC-DPO fine-tuned result in appendix}
\end{align}
where $\tilde{Z}(\mb c)=\int_{\mb x}p_\text{ref}(\mb x|\mb c)\left(\frac{p(\mb x|\mb c)}{p(\mb x)}\right)^{\frac{1}{\beta}}d\mb x$ is the normalizing constant. This completes the proof.
\end{proof}

\noindent\textbf{Remark.} Note that from~\eqref{intermdiate step 1} to~\eqref{intermdiate step 2}, we require $p(\mb x_2|\mb c)p(\mb x_1)+p(\mb x_1|\mb c)p(\mb x_2)>0$ for $\forall \mb x_1,\mb x_2$, which holds true if both $p(\mb x|\mb c)$ and $p(\mb x)$ have full support on $\mathbb{R}^d$. Next, we discuss the corner cases where this assumption doesn't hold.
\begin{enumerate}[leftmargin=*]
    \item Suppose $p(\mb x_1|\mb c)=0$ but $p(\mb x_1)>0$, then we have:
    \begin{align}
        \mathcal{L}(w)=\log \mathrm{Sigmoid}\bigl(r(\mb x_2|\mb c)-r(\mb x_1|\mb c)\bigl)p(\mb x_2|\mb c)p(\mb x_1).
    \end{align}
In this case, the objective is maximized if $r^*(\mb x_1|\mb c)=-\infty$.
    \item Suppose $p(\mb x_1|\mb c)>0$ but $p(\mb x_1)=0$, then we have:
    \begin{align}
        \mathcal{L}(w)=\log\mathrm{Sigmoid}\bigl(r(\mb x_1|\mb c)-r(\mb x_2|\mb c)\bigl)p(\mb x_1|\mb c)p(\mb x_2).
    \end{align}
In this case, the objective is maximized if $r^*(\mb x_1|\mb c)=+\infty$.
    \item Suppose $p(\mb x_1|\mb c)=0$ and $p(\mb x_1)=0$, then $\mb x_1$ will never be sampled and it does not contribute to the overall objective. In such case, $r^*(\mb x_1|\mb c)$ is undefined and will implicitly depend on the practical parameterization of the reward model. 
\end{enumerate}
The first two cases are already covered by~\eqref{optimal CC-DPO solution appendix}. In particular, case 2 can lead to non-normalizable issue, as in this case $p_{\mb\theta^*}(\cdot|\mb c)$ will have infinite density at $\mb x_1$. Although since $p(\mb x_1)=\mathbb{E}_{\mb c}\left[p(\mb x_1|\mb c)\right]$, it will not be zero given $p(\mb x_1|\mb c)>0$. It is highly likely in practice, there exists $\mb x_1$ such that $p(\mb x_1|\mb c)>0$ but $p(\mb x_1)\approx 0$. In this regime, the ratio $\left(\frac{p(\mb x|\mb c)}{p(\mb x)}\right)^{\frac{1}{\beta}}$ blows up, forcing $p_{\mb\theta^*}(\cdot|\mb c)$ to essentially place all its mass on such $\mb x$, which again leads to non-normalizable issue.
\paragraph{Error Model for CC-DPO.} As with the MCLR case, CC-DPO recovers the ground-truth conditional distribution under an appropriate error model, as stated in the following corollary.
\begin{cor}
If the base model satisfies:
\begin{align}
p_\text{ref}(\mb x|\mb c)\propto p(\mb x|\mb c)^{1-\frac{1}{\beta}}p(\mb x)^{\frac{1}{\beta}} 
\label{gamma powered error model assumption}
\end{align}

then fine-tuning $p_\text{ref}(\mb x|\mb c)$ with the CC-DPO objective~\eqref{DPO objective-conditional} recovers the ground-truth conditional distribution $p(\mb x|\mb c)$. 
\label{corollary CC-DPO fine-tuning}  
\end{cor}
\begin{proof}
    Substituting~\eqref{gamma powered error model assumption} into~\eqref{CC-DPO fine-tuned result in appendix}, we get $p_{\mb\theta^*}(\mb x|\mb c)=p(\mb x|\mb c)$, which completes the proof.
\end{proof}
\section{Theoretical Analysis of CCA}
\label{sec: theoretical analysis on CCA}
In~\Cref{sec: proof of CC-DPO optimal solution}, we've shown that the underlying optimal reward function induced by the CC-DPO objective~\eqref{Reward modeling objective-conditional-1} has the form of the log likelihood-ratio~\eqref{optimal reward for DPO}. Interestingly, the same reward function can be obtained by minimizing the following optimization problem:
\begin{align}
    r^*(\cdot|\mb c)&=\argmax_{r(\cdot|\mb c)}\;\mathbb{E}_{p(\mb x|\mb c)}\log\mathrm{Sigmoid}(r(\mb x|\mb c)) + \mathbb{E}_{p(\mb x)}\log\mathrm{Sigmoid}(-r(\mb x|\mb c))\\
    &=\argmax_{r(\cdot|\mb c)}\;\int_{\mb x}\bigl(\log\mathrm{Sigmoid}(r(\mb x|\mb c))p(\mb x|\mb c)+\log\mathrm{Sigmoid}(-r(\mb x|\mb c))p(\mb x)\bigl)\der\mb x.
    \label{NCE}
\end{align}
Since the objective~\eqref{NCE} decomposes pointwise over $\mb x$, the optimal reward
function can be obtained by maximizing the integrand for each $\mb x$ independently:
\begin{align}
    \log\mathrm{Sigmoid}(r(\mb x|\mb c))p(\mb x|\mb c)+\log\mathrm{Sigmoid}(-r(\mb x|\mb c))p(\mb x).
\end{align}
To find $r^*(\cdot|\mb c)$, let's define $w=r(\mb x|\mb c)$ and solve the following optimization problem:
\begin{align}
    \max_{ w}\;\mathcal{L}(w):= \log\mathrm{Sigmoid}( w)p(\mb x|\mb c)+\log\mathrm{Sigmoid}(-w)p(\mb x).
\end{align}    
Note that:
\begin{align}
\nabla_{w}\mathcal{L}(w):=\mathrm{Sigmoid}(-w)p(\mb x|\mb c)-\mathrm{Sigmoid}(w)p(\mb x),
\end{align}
which implies the stationary point $w^*$ is:
\begin{align}
    w^*=r^*(\mb x|\mb c)=\log\frac{p(\mb x|\mb c)}{p(\mb x)}.
\end{align}
Furthermore, since:
\begin{align}
    \nabla_{w}^2\mathcal{L}(w)&=-\mathrm{Sigmoid}(w)\mathrm{Sigmoid}(-w)p(\mb x|\mb c)-\mathrm{Sigmoid}(w)\mathrm{Sigmoid}(-w)p(\mb x)\\
    &<0,
\end{align}
we know $\mathcal{L}(w)$ is concave and $r^*(\mb x|\mb c)=\log\frac{p(\mb x|\mb c)}{p(\mb x)}$ is the unique maximizer. Hence, we've proved that the optimization problems~\eqref{Reward modeling objective-conditional-1} and~\eqref{NCE} lead to the same optimal reward up to an additive constant.

Therefore, by parameterizing the reward function as:
\begin{align}
    r_{\mb\theta}(\mb x|\mb c) = \beta\log\frac{p_{\mb \theta}(\mb x|\mb c)}{p_\text{ref}(\mb x|\mb c)}
    \label{reward parameterization CCA}
\end{align},
and substitute this in~\eqref{NCE}, we get the following optimization problem:
\begin{align}
    p_{\mb\theta^*}(\cdot|\mb c)&=\argmax_{p_{\mb\theta}(\cdot|\mb c)}\;\mathbb{E}_{p(\mb x|\mb c)}\log\mathrm{Sigmoid}\bigl(\beta\log\frac{p_{\mb \theta}(\mb x|\mb c)}{p_\text{ref}(\mb x|\mb c)}\bigl) + \mathbb{E}_{p(\mb x)}\log\mathrm{Sigmoid}\bigl(-\beta\log\frac{p_{\mb \theta}(\mb x|\mb c)}{p_\text{ref}(\mb x|\mb c)}\bigl),
    \label{CCA}
\end{align}
which has the global maximizer:
\begin{align}
    \beta\log\frac{p_{\mb\theta^*}(\mb x|\mb c)}{p_\text{ref}(\mb x|\mb c)}&=r^*(\mb x|\mb c)=\log\frac{p(\mb x|\mb c)}{p(\mb x)}\\
    \Rightarrow p_{\mb\theta^*}(\mb x|\mb c)&=p_\text{ref}(\mb x|\mb c)\bigl(\frac{p(\mb x|\mb c)}{p(\mb x)}\bigr)^{\frac{1}{\beta}}.
    \label{optimal solution for CCA, non-normalized}
\end{align}
One issue with the optimization objective~\eqref{CCA} is that its optimal solution~\eqref{optimal solution for CCA, non-normalized} is not a valid probability distribution because it is not normalized. To alleviate this issue, we consider a slightly modified version of~\eqref{NCE} by introducing an additional constant $\lambda$:
\begin{align}
    r^*(\cdot|\mb c)&=\argmax_{r(\cdot|\mb c)}\;\mathbb{E}_{p(\mb x|\mb c)}\log\mathrm{Sigmoid}(r(\mb x|\mb c)) + \lambda\mathbb{E}_{p(\mb x)}\log\mathrm{Sigmoid}(-r(\mb x|\mb c)).
    \label{NCE with Lambda}
\end{align}

With similar proof  technique, it can be shown that:
\begin{align}
    r^*(\mb x|\mb c)=\log\frac{p(\mb x|\mb c)}{p(\mb x)}+\log\frac{1}{\lambda}.
    \label{CCA special parameterization with lambda}
\end{align}
Substituting~\eqref{reward parameterization CCA} in~\eqref{NCE with Lambda}, we get the following optimization problem:
\begin{align}
    p_{\mb\theta^*}(\cdot|\mb c)&=\argmax_{p_{\mb\theta}(\cdot|\mb c)}\;\mathbb{E}_{p(\mb x|\mb c)}\log\mathrm{Sigmoid}\bigl(\beta\log\frac{p_{\mb \theta}(\mb x|\mb c)}{p_\text{ref}(\mb x|\mb c)}\bigl) + \lambda\mathbb{E}_{p(\mb x)}\log\mathrm{Sigmoid}\bigl(-\beta\log\frac{p_{\mb \theta}(\mb x|\mb c)}{p_\text{ref}(\mb x|\mb c)}\bigl),
    \label{CCA with Lambda}
\end{align}
which has global maximizer:
\begin{align}
    \beta\log\frac{p_{\mb\theta^*}(\mb x|\mb c)}{p_\text{ref}(\mb x|\mb c)}=\log\frac{p(\mb x|\mb c)}{p(\mb x)}+\log\frac{1}{\lambda}\\
    \Rightarrow p_{\mb\theta^*}(\mb x|\mb c)=p_\text{ref}(\mb x|\mb c)\bigl(\frac{p(\mb x|\mb c)}{p(\mb x)}\bigr)^{\frac{1}{\beta}}(\frac{1}{\lambda})^{\frac{1}{\beta}}.
\end{align}
In this case, $p_{\mb\theta^*}(\mb x|\mb c)$ is a valid probability distribution as long as: 
\begin{align}
    \lambda^\frac{1}{\beta}=\int_{\mb x}p_\text{ref}(\mb x|\mb c)(\frac{p(\mb x|\mb c)}{p(\mb x)})^\frac{1}{\beta}\der\mb x.
\end{align}

Note that optimal solution to~\eqref{CCA with Lambda} exactly coincides with that of CC-DPO~\eqref{DPO objective-conditional}. The formulation in~\eqref{CCA with Lambda}, known as the Conditional Contrastive Alignment (CCA), which is essentially a combination of the Noise Contrastive Estimation (NCE)
~\cite{gutmann2010noise} and a special model parameterization~\eqref{CCA special parameterization with lambda}. This objective is first proposed by~\cite{chentoward} for improving the generation quality of visual autoregressive models without relying on CFG. It should be noted that CCA is theoretically correct only when $\lambda$ is fixed to a particular constant, and it must be tuned as a hyperparameter in practice. In contrast, we demonstrate that CC-DPO fine-tuning recovers the same optimal solution without introducing this additional hyperparameter $\lambda$.

\section{Practical Implementation Details}
\label{sec: practical implementation appendix}
\subsection{Approximating Log-Likelihood with ELBO}
\label{sec: ELBO denoiser form}
Implementing MCLR requires access to the log-likelihood, which is not directly available for diffusion models. We therefore approximate the log-likelihood using the evidence lower bound (ELBO) in~\eqref{MLE equals DSM}. In the following, we describe how this approximation is used in MCLR, CC-DPO, and CCA. Before doing so, we state the following fact.

\paragraph{Equivalence between Score Function and MMSE Denoisers.} For practical diffusion models, the drift coefficient $f(\cdot)$ in~\eqref{forward process} takes the form $f(\mb x,t)=f(t)\mb x$ where $f(\cdot): \mathbb{R}\rightarrow \mathbb{R}$. As a result, the corresponding transition distribution is Gaussian and can be written as~\cite{karras2022elucidating}:
\begin{align}
    p_{0t}(\mb x_t|\mb x)=\mathcal{N}(\mb x_t; s(t)\mb x, s^2(t)\sigma^2(t)\mb I)
\end{align},
where $\mathcal{N}(\mb x;\mb\mu,\mb\Sigma)$ denotes the Gaussian density with mean $\mb\mu$ and covariance $\mb\Sigma$ evaluated at $\mb x$, $s(t)=\exp(\int_0^tf(\xi)d\xi)$ and $\sigma(t)=\sqrt{\int_0^t\frac{g^2(\xi)}{s^2(\xi)}d\xi}$. The score of this transition distribution is therefore given by:
\begin{align}
    \nabla_{\mb x_t}\log p_{0t}(\mb x_t|\mb x)=\frac{s(t)\mb x-\mb x_t}{s^2(t)\sigma^2(t)}.
\end{align}
Accordingly, the score network can be parameterized in terms of a denoiser $\mathcal{D}_{\mb\theta}(\cdot)$ as:
\begin{align}
\mb{s}_{\mb\theta}(\mb x,t,\mb c)=\frac{\mathcal{D}_{\mb\theta}(\mb x_t;s^2(t)\sigma(t),\mb c)-\mb x_t}{s^2(t)\sigma^2(t)}.    
\end{align}
Without loss of generality, we set $s(t)=1$, under which the DSM objective in~\eqref{denoising score matching} becomes:
\begin{align}
    & \frac{1}{2}\int_{0}^T\mathbb{E}_{\mb c, p(\mb x|\mb c),p_{0t}(\mb x_t|\mb x)}[w(t) \|\frac{\mathcal{D}_{\mb\theta}(\mb x_t;\sigma(t),\mb c)-\mb x}{\sigma^2(t)}\|_2^2] d t \nonumber \\
    =\;& \mathbb{E}_{\mb c,t\sim\mathcal{U}[0,T],p(\mb x|\mb c), p_{0t}(\mb x_t|\mb x)}\left[\tilde{w}(t)||\mathcal{D}_{\mb\theta}(\mb x_t;\sigma(t),\mb c)-\mb x
    ||_2^2\right],
\end{align}
where $\tilde{w}(t)=\frac{Tw(t)}{2\sigma^4(t)}$
This shows that DSM is equivalent to training the MMSE denoiser
$\mathcal{D}_{\mb\theta}(\cdot;\sigma(t),\mb c)$ for data from class $\mb c$
corrupted by additive Gaussian noise with standard deviation $\sigma(t)$.

In this setting, the score function is related to the MMSE denoiser via
Tweedie’s formula~\cite{miyasawa1961empirical}:
\begin{align}
    \nabla_{\mb x}\log p_t(\mb x\mid\mb c)
    = \frac{\mathcal{D}(\mb x;\sigma(t),\mb c)-\mb x}{\sigma^2(t)},
\end{align}
where $\mathcal{D}(\mb x;\sigma(t),\mb c)$ denotes the MMSE denoiser and
$p_t(\mb x|\mb c)=\int p_{0t}(\mb x|\mb x_0)p_\text{data}(\mb x_0)d\mb x_0$ is the marginal distribution at time $t$. We are now ready to present the practical ELBO-approximated objectives of MCLR, CC-DPO, and CCA for
diffusion models.

\paragraph{ELBO-Approximated MCLR for Diffusion Models.} 

Let $\bm v_2 = \nabla_{\boldsymbol{x}_t}\log p_{0t}(\boldsymbol{x}_t|\boldsymbol{x})
- \boldsymbol{s}_{\boldsymbol{\theta}}(\boldsymbol{x}_t,t,\boldsymbol{c})$ and $\bm v_3 = \nabla_{\boldsymbol{x}_t}\log p_{0t}(\boldsymbol{x}_t|\boldsymbol{x})
- \boldsymbol{s}_{\boldsymbol{\theta}}(\boldsymbol{x}_t,t,\tilde{\boldsymbol{c}}) $.
Substituting ELBO~\eqref{MLE equals DSM} into~\eqref{MLE with Likelihood ratio constraint in the appendix !!!}, the MCLR regularized DSM becomes:
\begin{equation}
\begin{aligned}
 \min_{\boldsymbol{\theta}}\;
\mathcal{J}_{\text{DSM}}(\boldsymbol{\theta}; g^2(\cdot)) + \eta\,
\mathbb{E}_{\substack{\boldsymbol{c}, \tilde{\boldsymbol{c}},\,t\sim \mathcal{U}[0,T]\\
                               p(\boldsymbol{x}|\boldsymbol{c}),\,p_{0t}(\boldsymbol{x}_t|\boldsymbol{x})}}
\Bigl[
g^2(t)
\Bigl(
\|\bm v_2\|_2^2
-
\| \bm v_3\|_2^2
\Bigr)
\Bigr].
\end{aligned}
\label{Using ELBO for likelihood in MCLR}
\end{equation}

Using the denoiser parameterization of the score network, together with a customized training-time noise sampling distribution $p(t)$ and adaptive weighting $w(t)$, the MCLR objective becomes:
\begin{equation}
\label{MCLR_appendix_denoising_parameterizaiton with t}
\begin{aligned}
 \mathbb{E}_{\substack{\boldsymbol{c}, \tilde{\boldsymbol{c}},\,t\sim p(t)\\
                               p(\boldsymbol{x}|\boldsymbol{c}),\,p_{0t}(\boldsymbol{x}_t|\boldsymbol{x})}}
   \Bigl[
     w(t)\bigl(
       \|\mb x-\mathcal{D}_{\mb\theta}(\mb x_t;\sigma(t),\mb c)\|_2^2
       \;-  \;
       \|\mb x-\mathcal{D}_{\mb\theta}(\mb x_t;\sigma(t),\tilde{\mb c})\|_2^2
     \bigr)
   \Bigr],
\end{aligned}
\end{equation}
which can be approximated with Monte Carlo sampling during training. Rather than parameterizing the noise level through the time variable $t\sim p(t)$, one may equivalently define a training noise distribution directly over $\sigma$, denoted by $p(\sigma)$ and a corresponding noise adaptive weighting $w(\sigma)$. Under this formulation,~\eqref{MCLR_appendix_denoising_parameterizaiton with t} can be rewritten as:
\begin{equation}
\begin{aligned}
 \mathbb{E}_{\substack{\boldsymbol{c}, \tilde{\boldsymbol{c}},\mb\epsilon\sim\mathcal{N}(\mb 0,\mb I)\,\\
                                \sigma\sim p(\sigma),
                               p(\boldsymbol{x}|\boldsymbol{c})}}
   \Bigl[
     w(\sigma)\bigl(
       \|\mb x-\mathcal{D}_{\mb\theta}(\mb x+\sigma\mb\epsilon;\sigma,\mb c)\|_2^2
       \;-  \;
       \|\mb x-\mathcal{D}_{\mb\theta}(\mb x+\sigma\mb\epsilon;\sigma,\tilde{\mb c})\|_2^2
     \bigr)
   \Bigr].
\end{aligned}
\label{MCLR denoiser form in appendix}
\end{equation}

Let $\bm v_4 = \nabla_{\boldsymbol{x}_t}\log p_{0t}(\boldsymbol{x}_t|\boldsymbol{x})
- \boldsymbol{s}_{\boldsymbol{\theta}}(\boldsymbol{x}_t,t,\boldsymbol{c})$ and $\bm v_5 = \nabla_{\boldsymbol{y}_t}\log p_{0t}(\boldsymbol{y}_t|\boldsymbol{y})
- \boldsymbol{s}_{\boldsymbol{\theta}}(\boldsymbol{y}_t,t,\boldsymbol{c}) $.
Similarly, substituting ELBO into~\eqref{MLE with Likelihood ratio constraint}, we get the following equivalent optimization problem:
\begin{equation}
\begin{aligned}
 \min_{\boldsymbol{\theta}}\;
\mathcal{J}_{\text{DSM}}(\boldsymbol{\theta}; g^2(\cdot))  + \eta\,
\mathbb{E}_{\substack{\boldsymbol{c}, \tilde{\boldsymbol{c}},\,t\sim \mathcal{U}[0,T]\\
                               p(\boldsymbol{x}|\boldsymbol{c}),\,p(\boldsymbol{y}|\tilde{\boldsymbol{c}})\\\,p_{0t}(\boldsymbol{x}_t|\boldsymbol{x}),\,p_{0t}(\boldsymbol{y}_t|\boldsymbol{y})}}
\Bigl[
g^2(t)
\Bigl(
\| \bm v_4 \|_2^2
-
\| \bm v_5 \|_2^2
\Bigr)
\Bigr],
\end{aligned}
\label{Using ELBO for likelihood in MCLR alternative form}
\end{equation}
the denoiser form of which is presented in~\eqref{denoiser perspective}.
\paragraph{ELBO-Approximated CC-DPO for Diffusion Models.}
The CC-DPO algorithm requires access to the log-likelihood ratio for individual data point. Similar to MCLR, we approximate it with the ELBO:
\begin{align}
    \log \frac{p_{\mb\theta}(\mb x)}{p_\text{ref}(\mb x)} \approx \frac{T}{2}\mathbb{E}_{t\sim\mathcal{U}[0,T], p_{0t}(\mb x_t|\mb x)}\Bigl[
     g^2(t)\bigl(-
       \|\nabla_{\boldsymbol{x}_t}\log p_{0t}(\boldsymbol{x}_t|\boldsymbol{x})
         - \boldsymbol{s}_{\boldsymbol{\theta}}(\boldsymbol{x}_t,t,\boldsymbol{c})\|_2^2
       \;+\; \\
       \|\nabla_{\boldsymbol{x}_t}\log p_{0t}(\boldsymbol{x}_t|\boldsymbol{x})
         - \boldsymbol{s}_\text{ref}(\boldsymbol{x}_t,t,\mb c)\|_2^2
     \bigr)
   \Bigr].
\end{align}
Using the denoiser parameterization of the score network, together with a customized training noise distribution and adaptive weighting, the log-likelihood ratio takes the following form:
\begin{align}
    \log\frac{p_{\mb\theta}(\mb x)}{p_\text{ref}(\mb x)}\approx\mathbb{E}_{\sigma\sim p_\text{train}(\sigma), \mb\epsilon\sim\mathcal{N}(\mb 0,\mb I)}\bigl[w(\sigma)(-\|\mb x-\mathcal{D}_{\mb\theta}(\mb x+\sigma\mb\epsilon;\sigma,\mb c)\|_2^2+\|\mb x-\mathcal{D}_\text{ref}(\mb x+\sigma\mb\epsilon;\sigma,\mb c)\|_2^2)\bigl].
    \label{ELBO for DPO}
\end{align}
Define:
\begin{align}
    \Delta(\mb x, \sigma, \mb\epsilon,\mb c):= \|\mb x-\mathcal{D}_{\mb\theta}(\mb x+\sigma\mb\epsilon;\sigma,\mb c)\|_2^2-\|\mb x-\mathcal{D}_\text{ref}(\mb x+\sigma\mb\epsilon;\sigma,\mb c)\|_2^2.
\end{align}
By substituting~\eqref{ELBO for DPO} into DPO objective~\eqref{DPO objective-conditional}, we get the following optimization problem:

\begin{align}
    \min_{\mb\theta}\;-\mathbb{E}_{\mb c,\tilde{\mb c}, \mb x_w\sim p(\mb x|\mb c),\mb x_l\sim p(\mb x|\tilde{\mb c})}\bigl[\log\mathrm{Sigmoid}(\beta\mathbb{E}_{\sigma,\mb \epsilon}[w(\sigma)(-\Delta(\mb x_w,\sigma,\mb\epsilon,\mb c)+\Delta(\mb x_l,\sigma,\mb\epsilon,\mb c))])\bigr].
    \label{cc-dpo objective, original form}
\end{align}

Since the log-Sigmoid function is concave, applying Jensen’s inequality by moving the expectation outside yields the following upper bound of objective~\eqref{cc-dpo objective, original form}:
\begin{align}
    \min_{\mb\theta}\;-\mathbb{E}_{\mb c,\tilde{\mb c}, \mb x_w,\mb x_l,\sigma,\mb\epsilon}\bigl[\log\mathrm{Sigmoid}(\beta w(\sigma)(-\Delta(\mb x_w,\sigma,\mb\epsilon,\mb c)+\Delta(\mb x_l,\sigma,\mb\epsilon,\mb c)))\bigr],
    \label{cc-dpo objective, practical upper bound}
\end{align}
which serves as our final training objective. 

\paragraph{CCA for Diffusion Models.}
The Conditional Contrastive Alignment (CCA)~\cite{chentoward} objective takes the following form:
\begin{align}
    \max_{\mb\theta}\;\mathbb{E}_{\mb c,p(\mb x|\mb c)}\log\mathrm{Sigmoid}\bigl(\beta\log\frac{p_{\mb \theta}(\mb x|\mb c)}{p_\text{ref}(\mb x|\mb c)}\bigl) + \lambda\mathbb{E}_{\mb c,\tilde{\mb c}, p(\mb x|\tilde{\mb c})}\log\mathrm{Sigmoid}\bigl(-\beta\log\frac{p_{\mb \theta}(\mb x|\mb c)}{p_\text{ref}(\mb x|\mb c)}\bigl).
\end{align}

To adapt this optimization problem to diffusion models, we again approximate the log-likelihood with ELBO, resulting in the following optimization problem:
\begin{align}
    \max_{\mb\theta}\;&\mathbb{E}_{\mb c,\tilde{\mb c}, \mb x_w\sim p(\mb x|\mb c), \mb x_l\sim p(\mb x|\tilde{\mb c})}\bigl[\log\mathrm{Sigmoid}\bigl(\beta\mathbb{E}_{\sigma,\mb\epsilon}\bigl[-w(\sigma)\Delta(\mb x_w,\sigma,\mb\epsilon,\mb c)\bigr]\bigl) \nonumber \\ &+\lambda\log\mathrm{Sigmoid}\bigl(\beta\mathbb{E}_{\sigma,\mb\epsilon}\bigl[w(\sigma)\Delta(\mb x_l,\sigma,\mb\epsilon,\mb c)\bigr]    
    \bigr)\bigr].
    \label{CCA objective, ELBO}
\end{align}
Since the log-sigmoid function is concave, applying Jensen's inequality by moving the expectation outside yields the following lower bound of objective~\eqref{CCA objective, ELBO}:
\begin{align}
        \max_{\mb\theta}\;\mathbb{E}_{\mb c,\mb x_w, \mb x_l,\sigma,\mb\epsilon}\bigl[\log\mathrm{Sigmoid}\bigl(-\beta w(\sigma)\Delta(\mb x_w,\sigma,\mb\epsilon,\mb c)\bigl)+\lambda\log\mathrm{Sigmoid}\bigl(\beta w(\sigma)\Delta(\mb x_l,\sigma,\mb\epsilon,\mb c)    
    \bigr)\bigr],
    \label{CCA objective, lower}
\end{align}
which serves as our final training objective.

\paragraph{Remarks on the ELBO Approximation.} As discussed in~\Cref{subsection: analysis on MCLR for diffusion models}, despite standard practice in the literature, when replacing log-likelihood with ELBO, the resulting objectives do not necessarily correspond to the original likelihood formulations, since the ELBO does not enforce the regularity conditions required for the parameterized score function to define a valid score field. Consequently, ELBO-approximated MCLR, DPO and CCA should be viewed as \textbf{approximate} likelihood-ratio maximization procedures for diffusion models. This issue does not arise for autoregressive models, where likelihoods are available exactly.

\subsection{Building Training Data from a Minibatch}
The equivalence of~\eqref{MLE with Likelihood ratio constraint in appendix} and~\eqref{MLE with Likelihood ratio constraint in the appendix !!!} offer us two choices for constructing the contrastive pairs. The first is constructing tuples of $(\mb x,\mb y,\mb c)$, where $\mb x\sim p(\mb x|\mb c)$ and $\mb y$ sampled from other random classes. The other is constructing contrastive tuples $(\mb x, \mb c, \tilde{\mb c})$, where $\mb x \sim p(\mb x | \mb c)$ and $\tilde{\mb c}$ denotes a randomly chosen alternative class. In our implementation we choose the second method, but our preliminary results demonstrate both achieve similar performance.

Similarly, CC-DPO and CCA require constructing preference-style tuples $(\mb x_w, \mb x_l, \mb c)$. In practice, we build these tuples directly from each training minibatch
$\bigl\{(\mb x_i, \mb c_i)\bigr\}_{i=1}^N.$
We consider two strategies, described below.

\paragraph{Approach 1: Building $N$ pairs.}
The simplest strategy constructs one contrastive (or preference) tuple per sample in the minibatch.  
For MCLR, given a sample $(\mb x, \mb c)$, we randomly select another label $\tilde{\mb c}$ from the same minibatch, forming a tuple $(\mb x, \mb c, \tilde{\mb c})$. Repeating this process independently for each $\mb x$ (with replacement) yields $N$ tuples in total.  
The same strategy applies to CC-DPO and CCA, where we construct $N$ tuples $(\mb x_w, \mb x_l, \mb c)$ from the minibatch.

\paragraph{Approach 2: Building $NK$ pairs.}
Alternatively, we can construct multiple contrastive tuples per sample. For MCLR, given each $(\mb x, \mb c)$, we randomly select $K$ alternative labels $\{\tilde{\mb c}_k\}_{k=1}^K$ from the same minibatch, yielding $K$ tuples $(\mb x, \mb c, \tilde{\mb c}_k)$. Repeating this procedure for all $N$ samples results in $NK$ tuples in total.  
An analogous strategy is applied to CC-DPO and CCA, producing $NK$ tuples $(\mb x_w, \mb x_l, \mb c)$ per minibatch.

Compared to Approach~1, Approach~2 increases the number of training tuples constructed from each minibatch, allowing better exploitation of inter-class contrastive information. In our experiments, we adopt Approach~2 for EDM-based diffusion models, where it consistently yields improved quantitative performance. For VAR and SiT models, we use Approach~1 due to its lower computational overhead and comparable empirical performance.


\subsection{Overall Algorithm}
We are now ready to present the overall algorithm. The algorithm for MCLR is given in~\Cref{alg:MCLR}, while the algorithms for CC-DPO and CCA are presented in~\Cref{alg:CCA and CC-DPO}. In~\Cref{alg:MCLR}, we compute the MCLR regularization loss without including the standard DSM term. Empirically, we find that optimizing the MCLR term alone yields a better best-case FD\textsubscript{DINOv2} score. One possible explanation is that removing the DSM term makes the contrastive objective the dominant training signal, allowing the model to explore a larger solution space beyond the neighborhood constrained by DSM regularization. Accordingly, the main results report the best FD\textsubscript{DINOv2} scores obtained without the DSM term. Nevertheless, including the DSM term can improve training stability and lead to a more favorable Precision--Recall trade-off as shown in~\Cref{effect of DSM Regularization}.

\begin{algorithm}[t]
  \caption{MCLR}\label{alg:MCLR}
  \begin{algorithmic}[1]            
    \Require Pre-trained base model $\mathcal{D}_{\mb \theta_0}$;
            noise-adaptive weight function $w(\cdot)$;
            training noise schedule $p(\sigma)$;
            learning rate $\alpha$;
            training dataset representing $p(\mb x,\mb c)$;
            $K\geq 1$ (number of mismatched class labels per sample in Approach 2)            
    \State Initialize $\mb{\theta}\leftarrow\mb\theta_0$
    \For{each training iteration}
      \State Sample a minibatch $\bigl\{ (\mathbf{x}_i, \mathbf{c}_i)\bigr\}_{i=1}^{N}$ with $(\mb x_i,\mb c_i)\sim p(\mb x,\mb c)$.
      \State Sample $\{\sigma_i\}_{i=1}^{N}$ with $\sigma_i \sim p(\sigma)$
      \State Sample $\{\mb\epsilon_i\}_{i=1}^{N}$ with $\mb\epsilon_i \sim \mathcal{N}(\mb 0,\mb I)$
      
      \State Construct contrastive tuples from the minibatch:
      \If{Approach 1}
        \State Set $M \gets N$
        \State For each $i$, sample $\tilde{\mb c}_i$ from minibatch labels with $\tilde{\mb c}_i \neq \mb c_i$
        \State Define tuples $\{(\mb x_i,\mb c_i,\tilde{\mb c}_i,\sigma_i,\mb\epsilon_i)\}_{i=1}^{M}$
      \EndIf
      \If{Approach 2}
        \State Set $M \gets NK$
        \State For each $i$ and each $k=1,\dots,K$, sample $\tilde{\mb c}_{i,k}$ from minibatch with $\tilde{\mb c}_{i,k}\neq \mb c_i$
        \State Define tuples $\{(\mb x_{i,k},\mb c_{i,k},\tilde{\mb c}_{i,k},\sigma_{i,k},\mb\epsilon_{i,k})\}$ where
        \Statex \hspace{\algorithmicindent} $(\mb x_{i,k},\mb c_{i,k})=(\mb x_i,\mb c_i)$ and $(\sigma_{i,k},\mb\epsilon_{i,k})=(\sigma_i,\mb\epsilon_i)$
      \EndIf

      \State $\mathcal{L}\gets 0$
      \For{each tuple $(\mb x,\mb c,\tilde{\mb c},\sigma,\mb\epsilon)$}
        \State $\mathcal{L}\mathrel{+}= w(\sigma)\Bigl(
          \|\mathcal{D}_{\mb\theta}(\mb x+\sigma\mb\epsilon;\sigma,\mb c)-\mb x\|_2^2
          -\|\mathcal{D}_{\mb\theta}(\mb x+\sigma\mb\epsilon;\sigma,\tilde{\mb c})-\mb x\|_2^2
        \Bigr)$
      \EndFor

      \State $\mb{\theta}\gets \mb{\theta}-\alpha\nabla_{\mb{\theta}}\mathcal{L}$
    \EndFor
    \State \Return $\mathcal{D}_{\mb{\theta}}$
  \end{algorithmic}
\end{algorithm}

\begin{algorithm}[t]
  \caption{CC-DPO and CCA}\label{alg:CCA and CC-DPO}
  \begin{algorithmic}[1]            
    \Require Pre-trained base model $\mathcal{D}_{\mb \theta_0}$;
            noise-adaptive weight function $w(\cdot)$;
            training noise schedule $p(\sigma)$;
            learning rate $\alpha$;
            KL regularization strength $\beta$; CCA hyperparameter $\lambda$;
            training dataset representing $p(\mb x,\mb c)$;
            $K\geq 1$ (number of mismatched class labels per sample in Approach 2)            
    \State Initialize $\mb{\theta}\leftarrow\mb\theta_0$, $\mathcal{D}_\text{ref}\leftarrow \mathcal{D}_{\mb\theta_0}$
    \For{each training iteration}
      \State Sample a minibatch $\bigl\{ (\mathbf{x}_i, \mathbf{c}_i)\bigr\}_{i=1}^{N}$ with $(\mb x_i,\mb c_i)\sim p(\mb x,\mb c)$.
      \State Sample $\{\sigma_i\}_{i=1}^{N}$ with $\sigma_i \sim p(\sigma)$
      \State Sample $\{\mb\epsilon_i\}_{i=1}^{N}$ with $\mb\epsilon_i \sim \mathcal{N}(\mb 0,\mb I)$
      
      \State Construct contrastive tuples from the minibatch:
      \If{Approach 1}
        \State Set $M \gets N$
        \State For each $i$, sample $\tilde{\mb x}_i$ from minibatch samples with label $\tilde{\mb c}_i \neq \mb c_i$
        \State Define tuples $\{(\mb x_i,\tilde{\mb x}_i, \mb c_i,\sigma_i,\mb\epsilon_i)\}_{i=1}^{M}$
      \EndIf
      \If{Approach 2}
        \State Set $M \gets NK$
        \State For each $i$ and each $k=1,\dots,K$, sample $\tilde{\mb x}_{i,k}$ from minibatch samples with label $\tilde{\mb c}_{i,k}\neq \mb c_i$
        \State Define tuples $\{(\mb x_{i,k}, \tilde{\mb x}_{i,k},\mb c_{i,k}, \sigma_{i,k},\mb\epsilon_{i,k})\}$ where
        \Statex \hspace{\algorithmicindent} $(\mb x_{i,k},\mb c_{i,k})=(\mb x_i,\mb c_i)$ and $(\sigma_{i,k},\mb\epsilon_{i,k})=(\sigma_i,\mb\epsilon_i)$
      \EndIf

      \State $\mathcal{L}\gets 0$
      \For{each tuple $(\mb x,\tilde{\mb x}, \mb c,\sigma,\mb\epsilon)$}
      \State     $\Delta(\mb x_w, \sigma, \mb\epsilon,\mb c):= \|\mb x-\mathcal{D}_{\mb\theta}(\mb x+\sigma\mb\epsilon;\sigma,\mb c)\|_2^2-\|\mb x-\mathcal{D}_\text{ref}(\mb x+\sigma\mb\epsilon;\sigma,\mb c)\|_2^2$
      \State     $\Delta(\mb x_l, \sigma, \mb\epsilon,\mb c):= \|\tilde{\mb x}-\mathcal{D}_{\mb\theta}(\tilde{\mb x}+\sigma\mb\epsilon;\sigma,\mb c)\|_2^2-\|\tilde{\mb x}-\mathcal{D}_\text{ref}(\tilde{\mb x}+\sigma\mb\epsilon;\sigma,\mb c)\|_2^2$
      \If{CC-DPO}
              \State $\mathcal{L}\mathrel{+}= -\log\mathrm{Sigmoid}(\beta w(\sigma)(-\Delta(\mb x_w,\sigma,\mb\epsilon,\mb c)+\Delta(\mb x_l,\sigma,\mb\epsilon,\mb c)))$
      \EndIf
      \If{CCA}
              \State $\mathcal{L}\mathrel{-}= \log\mathrm{Sigmoid}\bigl(-\beta w(\sigma)\Delta(\mb x_w,\sigma,\mb\epsilon,\mb c)\bigl)+\lambda\log\mathrm{Sigmoid}\bigl(\beta w(\sigma)\Delta(\mb x_l,\sigma,\mb\epsilon,\mb c)\bigr)$
      \EndIf
      
      \EndFor

      \State $\mb{\theta}\gets \mb{\theta}-\alpha\nabla_{\mb{\theta}}\mathcal{L}$
    \EndFor
    \State \Return $\mathcal{D}_{\mb{\theta}}$
  \end{algorithmic}
\end{algorithm}

\subsection{Hyperparameters}
For each method, we carefully tune hyperparameters using grid search to ensure the best performance. The full set of hyperparameter configurations will be made publicly available upon publication.

\subsection{Computational Efficiency}
\label{sec:computatinal cost}
MCLR fine-tuning incurs additional cost, but it is modest relative to pretraining. In our runs on 4$\times$A40 40GB GPUs, for EDM2-S (64×64), MCLR fine-tuning took 5h23m to reach best FD\textsubscript{DINOv2} score 56 with Approach 1 (construct N contrastive pairs per batch) and 18h40m to reach best FD\textsubscript{DINOv2} 52 with Approach 2 (construct NK contrastive pairs per batch), starting from a base model with FD\textsubscript{DINOv2} score 95. In contrast, the pretraining stage requires training with 32 A100 GPUs for weeks~\cite{karras2024analyzing}. A comprehensive summary of computational cost can be found within our codebase upon publication.

\section{Additional Experimental Results}
\label{additional results appendix}
In this section, we present additional experimental results that complement and extend those in the main text. Specifically, we focus on validating the following three claims:
\begin{enumerate}
    \item MCLR induces progressive class separation, leading to a fidelity--diversity trade-off analogous to that produced by increasing the guidance strength in classifier-free guidance (CFG).
    \item MCLR matches or outperforms training-time baselines, including CC-DPO and CCA.
    \item MCLR achieves performance comparable to CFG.
\end{enumerate}


\subsection{Progressive Class Separation and the Fidelity-Diversity Trade-off.}
Training with MCLR induces notable \emph{progressive class separation}, which gives rise to a fidelity--diversity trade-off similar that observed when increasing the guidance scale in CFG. This behavior is quantitatively supported by the following observations.

\noindent\textbf{(i) Increased Precision and Inception Score.}
As shown in~\Cref{fig:ImageNet64 EDM2 quant}(b,d,g) for ImageNet-64,~\Cref{fig:ImageNet256 VAR quant}(c,f,i) for ImageNet-256 and~\Cref{fig:ImageNet512 EDM2 quant}(b,d,g) for ImageNet-512, continued training with MCLR leads to a progressive increase in Inception Score, indicating increasingly class-discriminative generations. This is accompanied by a corresponding increase in Precision, reflecting improved image fidelity.

\noindent\textbf{(ii) Decreased Recall.} Conversely, as shown in~\Cref{fig:ImageNet64 EDM2 quant}(e,h),~\Cref{fig:ImageNet512 EDM2 quant}(e,h), and~\Cref{fig:ImageNet256 VAR quant}(g,j), excessive training reduces within-class diversity, which manifests as a decrease in Recall (see also~\Cref{fig:imagenet512 quant}(e)).

Taken together, these effects result in FD--IS and Precision--Recall trade-offs that closely mirror those induced by CFG, as illustrated in~\Cref{fig:ImageNet64 EDM2 quant}(c,f,i),~\Cref{fig:ImageNet512 EDM2 quant}(c,f,i), and~\Cref{fig:ImageNet256 VAR quant}(d,e,h,k). Qualitatively, as shown in~\Cref{fig:imagenet512 qualitative teaser,fig:class_207_qua,fig:class_343_qua,fig:class_513_qua,fig:class_515_qua,fig:class_617_qua,fig:additional_64_qua_1,fig:additional_64_qua_2,fig:additional_64_qua_3,fig:additional_512_qua_1,fig:additional_512_qua_2,fig:additional_512_qua_3,fig:additional_512_qua_4}, continued MCLR training leads to gradually enhanced distinct, class-specific visual characteristics in generated samples.

\subsection{MCLR Outperforms Training-time Baselines.}
\textbf{Diffusion Models.} As shown in~\Cref{fig:ImageNet64 EDM2 quant,fig:ImageNet512 EDM2 quant}, for EDM models trained on ImageNet, MCLR achieves substantially better best-case FD\textsubscript{DINOv2} scores than CCA, CC-DPO. Moreover, MCLR traverses significantly wider FD--IS and Precision--Inception trade-offs region with a faster training speed. In contrast, CCA and CC-DPO often converge early, exhibit zigzag training trajectories, and become trapped in suboptimal local minima, as evidenced by slowly improving or stalled learning curves.
\textbf{Autoregressive Models.} As shown in~\Cref{fig:ImageNet256 VAR quant}, MCLR, CC-DPO, and CCA yield comparable performance improvement on VAR-d24. Specifically, we see MCLR consistently achieves higher precision (see~\Cref{fig:ImageNet256 VAR quant}(f,h,i,k)) and inception score (see~\Cref{fig:ImageNet256 VAR quant}(c)) in later stages of training.

\subsection{MCLR Achieves Comparable Performance as CFG.} 
\paragraph{Diffusion Models.} For EDM models trained on ImageNet, CFG generally exhibits a better FD--IS trade-off and achieves a lower best-case FD\textsubscript{DINOv2} than MCLR. Nevertheless, this gap is moderate for EDM2-L model,
where CFG attains a best-case FD\textsubscript{DINOv2} of 39.86 compared
to 42.50 for MCLR as shown in~\Cref{fig:ImageNet512 EDM2 quant}(c). Moreover, for EDM2-L model, when evaluated using Precision--Recall, MCLR demonstrates
competitive and in some regimes superior behavior.
As shown in~\Cref{fig:ImageNet512 EDM2 quant}(f,g), MCLR matches CFG in the
high-recall regime (early training stages) and achieves
a substantially higher best-case precision in the high-precision regime
(later training stages), where CFG begins to generate images with
oversaturated colors. 

Applying CFG on top of an MCLR-fine-tuned model further alleviates the performance gap between MCLR and CFG, yielding higher best-case Inception Score and Precision on both EDM2-S and EDM2-L models, as shown in~\Cref{fig:ImageNet64 EDM2 quant}(c,f,i) and~\Cref{fig:ImageNet512 EDM2 quant}(c,f,i).

Qualitatively, MCLR and CFG often produce similar effects, both significantly enhancing class-specific structures in the generated
images, as shown in~\Cref{fig:imagenet512 qualitative teaser,fig:class_207_qua,fig:class_343_qua,fig:class_513_qua,fig:class_515_qua,fig:class_617_qua,fig:additional_64_qua_1,fig:additional_64_qua_2,fig:additional_64_qua_3,fig:additional_512_qua_1,fig:additional_512_qua_2,fig:additional_512_qua_3,fig:additional_512_qua_4}.
This similarity is expected, as both methods improve conditional
modeling through inter-class contrastive signals.
The key distinction is that MCLR internalizes this mechanism during
training, whereas CFG applies it at inference time.
\paragraph{Autoregressive Models.}
For VAR-d24 model, MCLR achieves a similar FD--IS trade-off to
CFG in terms of both FD\textsubscript{DINOv2} and FID, with CFG exhibiting
a slightly better best-case FID.
Consistent with diffusion models, MCLR outperforms CFG in the
Precision--Recall trade-off and achieves a higher best-case precision.

\begin{figure*}[t]
    \centering
    \includegraphics[width=0.8\linewidth]{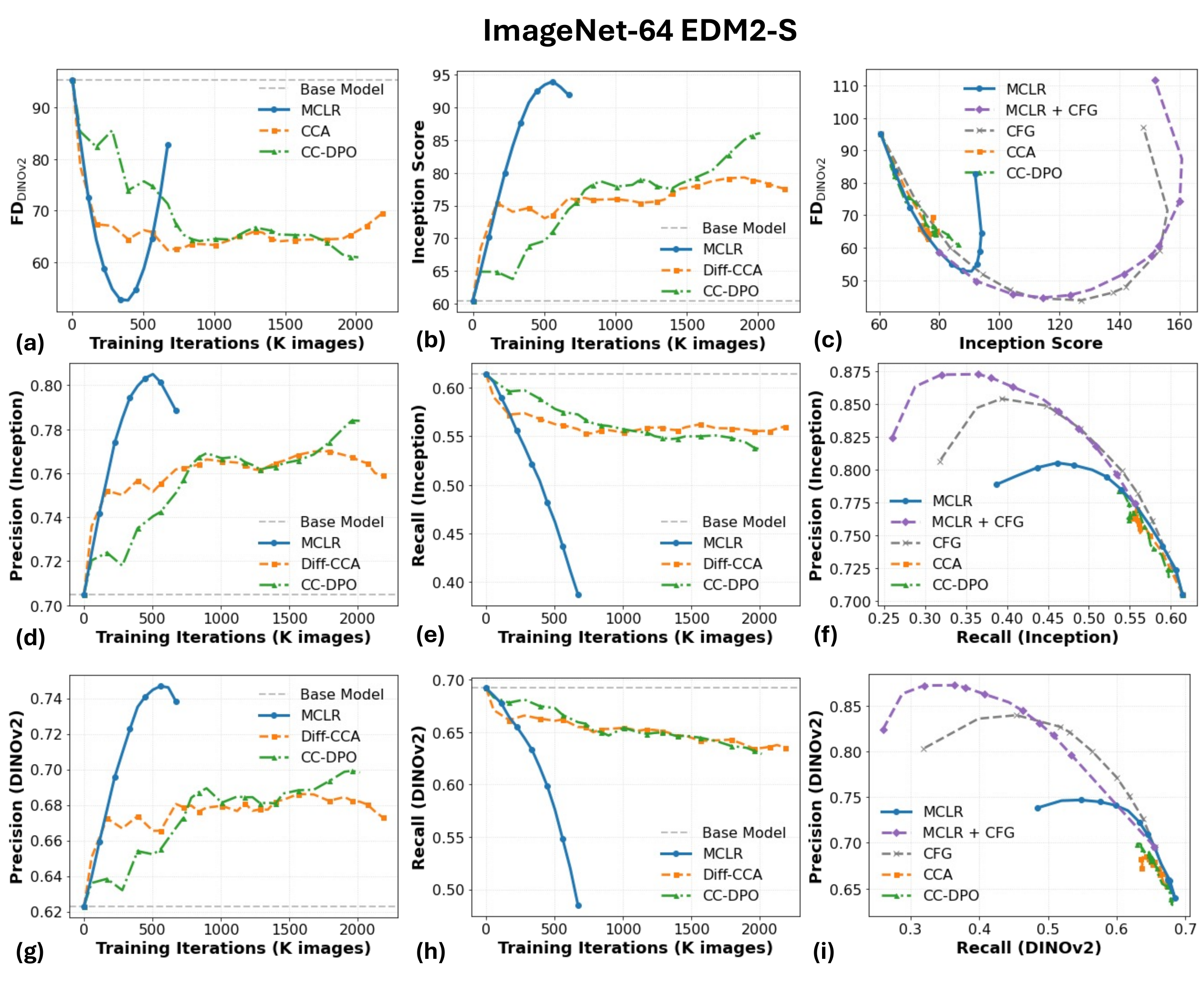}%
    \caption{\textbf{Quantitative Results for EDM2-S trained on ImageNet-64$\times$64.} (a), (b), (d), (e), (g), (h) show the evolution of FD, Inception Score, Precision (calculated with Inception features), and Recall (calculated with Inception features), Precision (calculated with \text{DINOv2} features), and Recall (calculated with \text{DINOv2} features), respectively, as functions of training iterations.
(c) shows the FD--IS trade-offs, while (f), (i) depict the Precision--Recall trade-offs calculated with Inception and \text{DINOv2} features respectively.
We evaluate classifier-free guidance (CFG) scales $\gamma \in \{0.1, 0.2, 0.3, 0.4, 0.5, 0.7, 0.9, 1, 1.5, 2.0, 3.0\}$.}
    \label{fig:ImageNet64 EDM2 quant}
\end{figure*}

\begin{figure*}[t]
    \centering
    \includegraphics[width=0.8\linewidth]{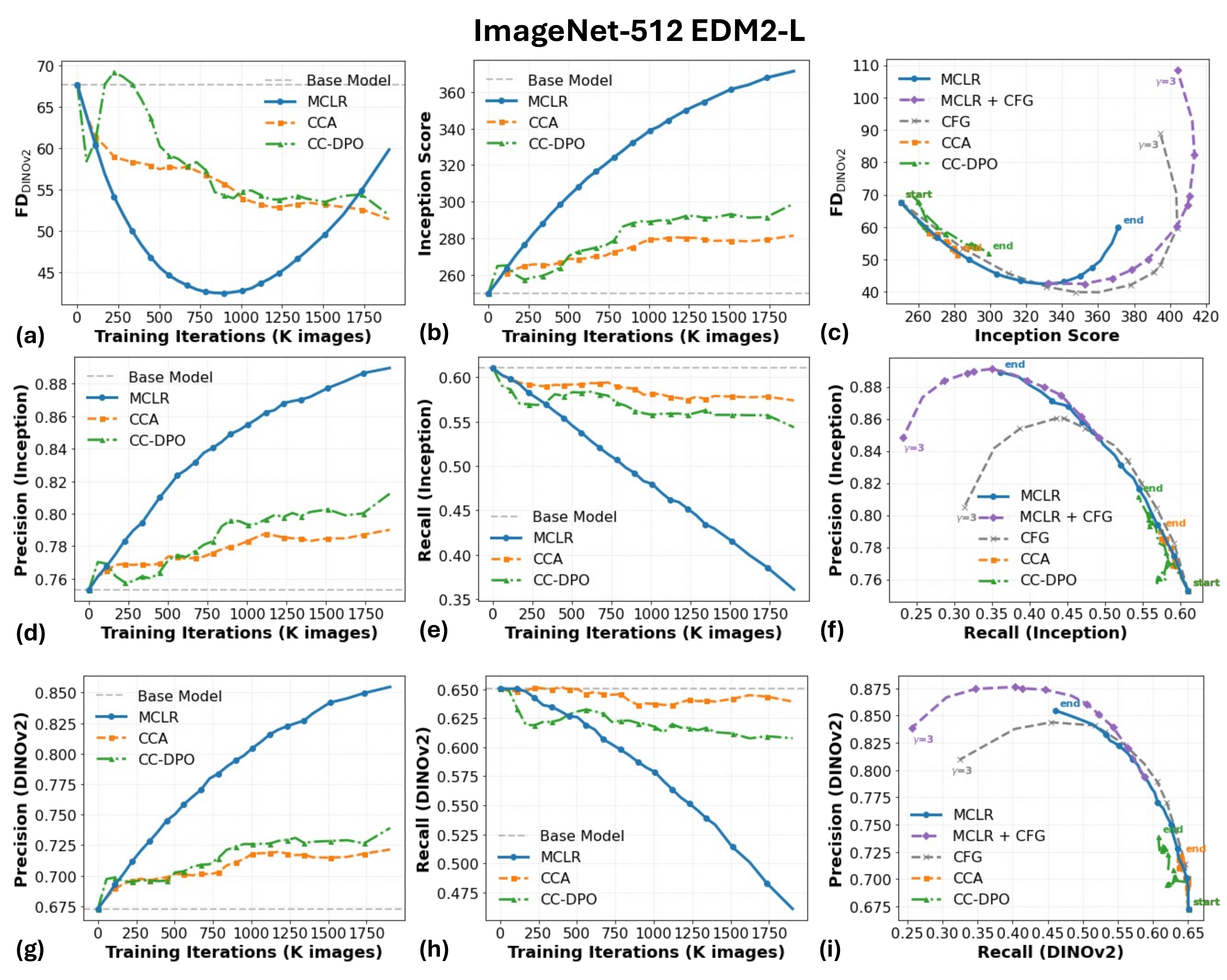}\caption{\textbf{Quantitative Results for EDM2-L trained on ImageNet-512$\times$512.} (a), (b), (d), (e), (g), (h) show the evolution of FD, Inception Score, Precision (calculated with Inception features), and Recall (calculated with Inception features), Precision (calculated with \text{DINOv2} features), and Recall (calculated with \text{DINOv2} features), respectively, as functions of training iterations.
(c) shows the FD--IS trade-offs, while (f), (i) depict the Precision--Recall trade-offs calculated with Inception and \text{DINOv2} features respectively.
We evaluate classifier-free guidance (CFG) scales $\gamma \in \{0.1, 0.2, 0.3, 0.4, 0.5, 0.7, 0.9, 1, 1.5, 2.0, 3.0\}$.}
    \label{fig:ImageNet512 EDM2 quant}
\end{figure*}

\begin{figure*}[t]
    \centering
    \includegraphics[width=0.8\linewidth]{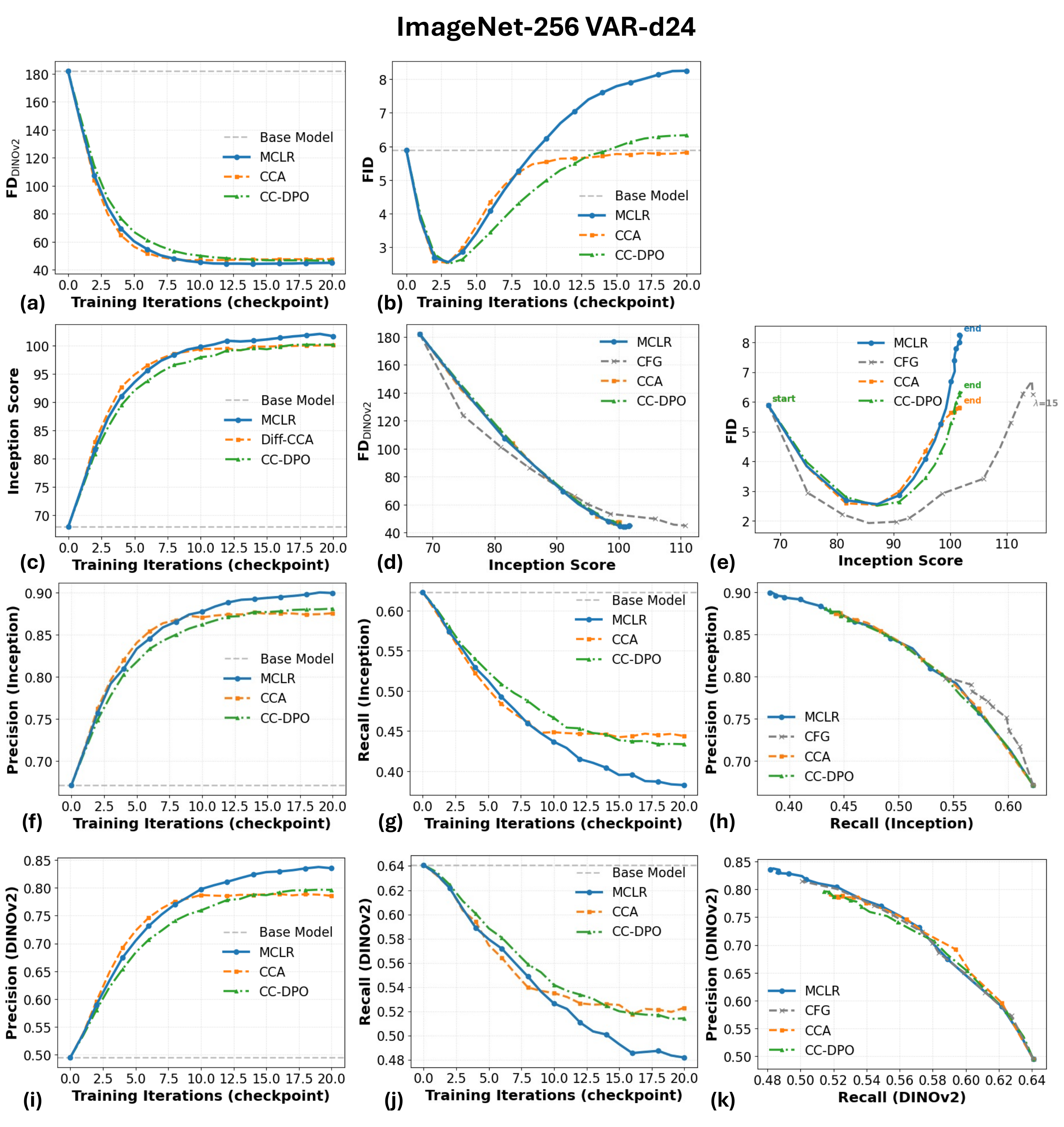}%
    \caption{\textbf{Quantitative Results for VAR-d24 trained on ImageNet-256$\times$256.} (a), (b), (c), (f), (g), (i), (j) show the evolution of FD\textsubscript{DINOv2}, FID, Inception Score, Precision (calculated with Inception features), and Recall (calculated with Inception features), Precision (calculated with \text{DINOv2} features), and Recall (calculated with \text{DINOv2} features), respectively, as functions of training iterations.
(d), (e) show the FD--IS trade-offs, while (h), (k) depict the Precision--Recall trade-offs calculated with Inception and \text{DINOv2} features respectively.
We evaluate classifier-free guidance (CFG) scales $\gamma \in \{0.5, 0.8, 1.1, 1.5, 1.7, 2.0, 2.5, 3.0, 4.0, 5.0, 7.0, 10.0, 15.0\}$.}
    \label{fig:ImageNet256 VAR quant}
\end{figure*}


\begin{figure*}[t]
    \centering
    \includegraphics[width=0.8\linewidth]{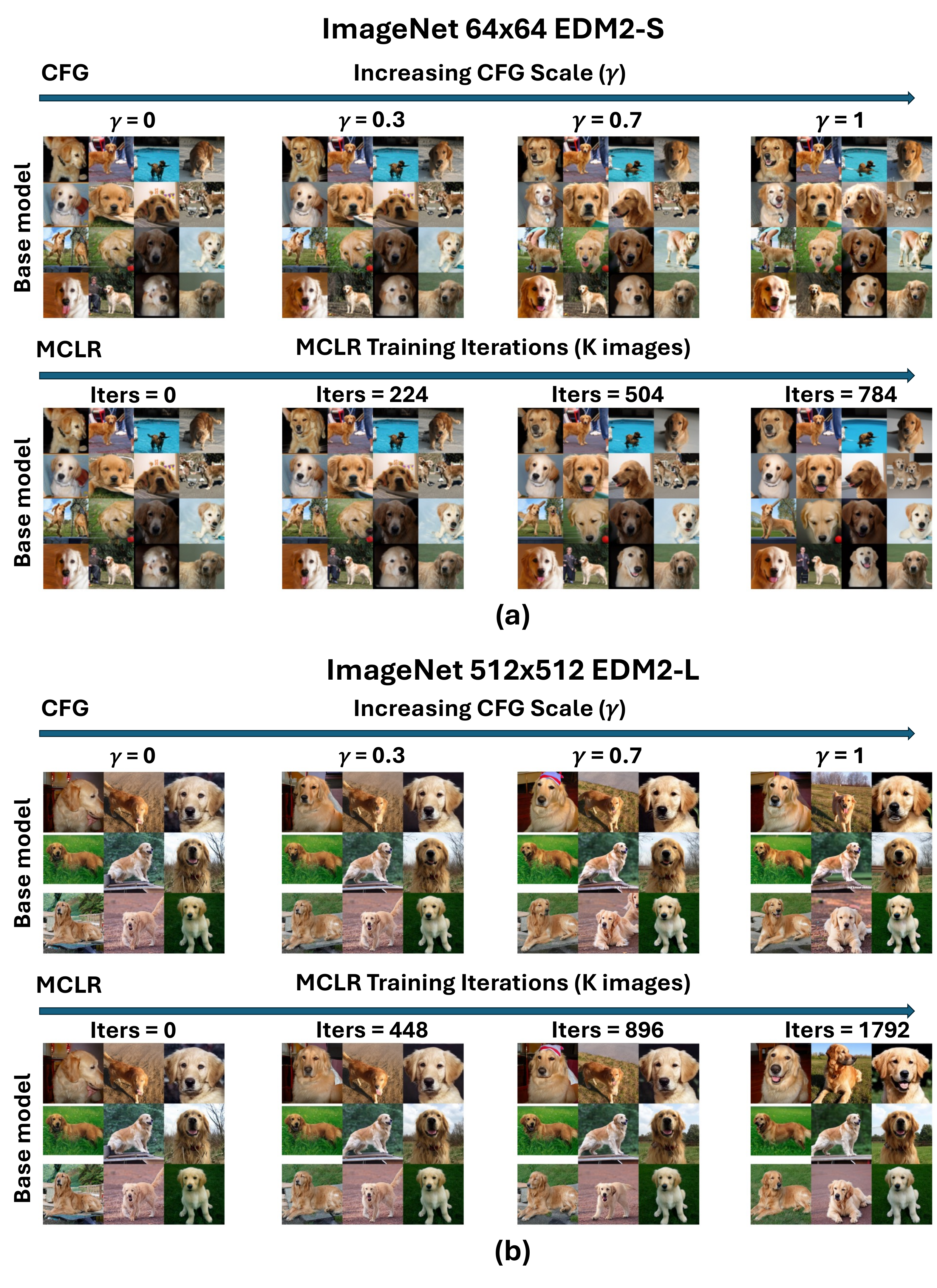}
    \caption{\textbf{Comparison between CFG and MCLR for Golden Retriever (Class 207).} (a,b) demonstrate the progressive evolution of generated samples on ImageNet-64×64 and ImageNet-512×512, respectively. Increasing the CFG scale $\gamma$ (top rows) and progressive MCLR training (bottom rows) produce similar effects, both enhancing class-specific structures in the generated images.}
    \label{fig:class_207_qua}
\end{figure*}

\begin{figure*}[t]
    \centering
    \includegraphics[width=0.8\linewidth]{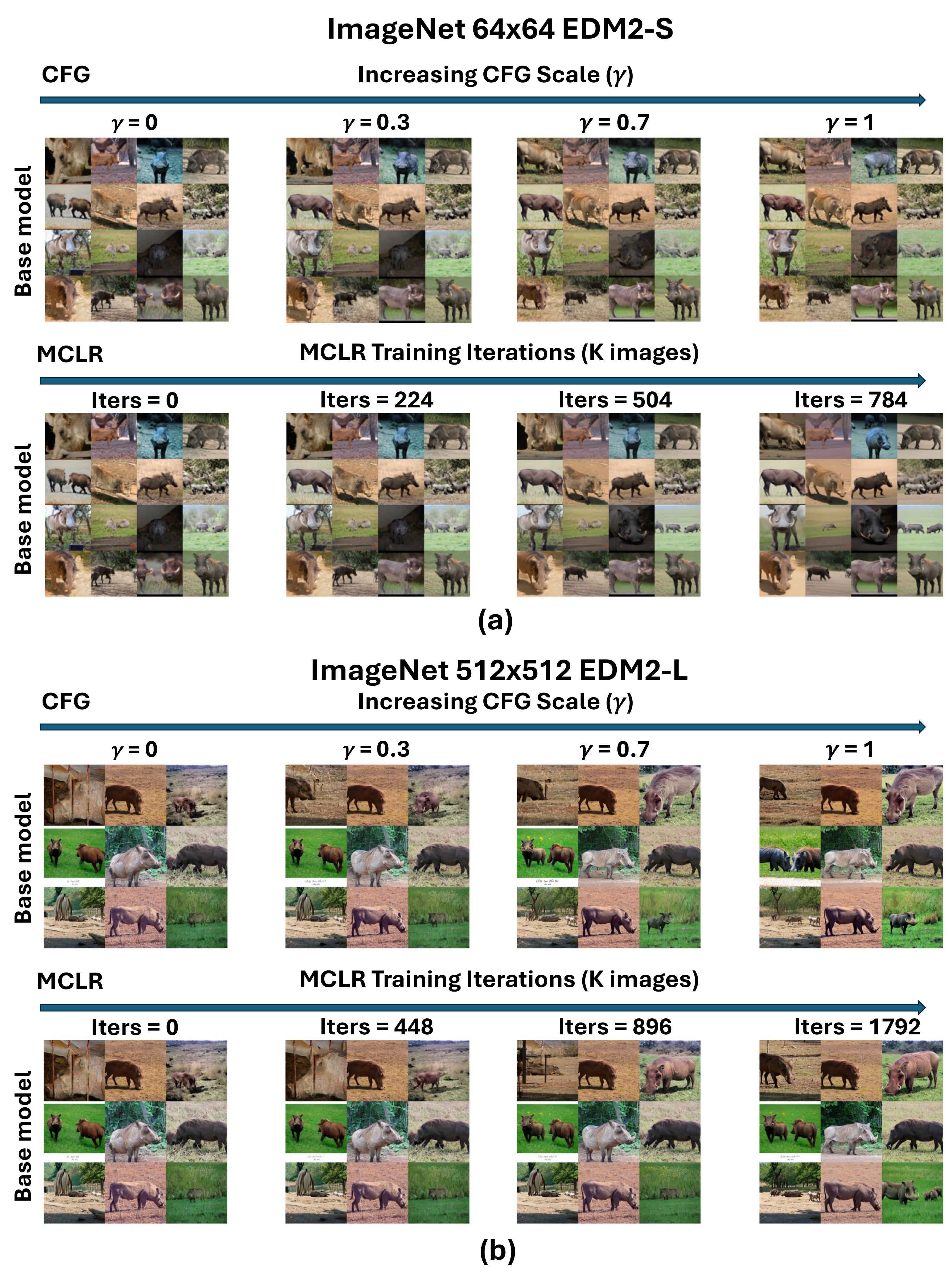}
    \caption{\textbf{Comparison between CFG and MCLR for Warthog (Class 343).} (a,b) demonstrate the progressive evolution of generated samples on ImageNet-64×64 and ImageNet-512×512, respectively. Increasing the CFG scale $\gamma$ (top rows) and progressive MCLR training (bottom rows) produce similar effects, both enhancing class-specific structures in the generated images.}    \label{fig:class_343_qua}
\end{figure*}

\begin{figure*}[t]
    \centering
    \includegraphics[width=0.8\linewidth]{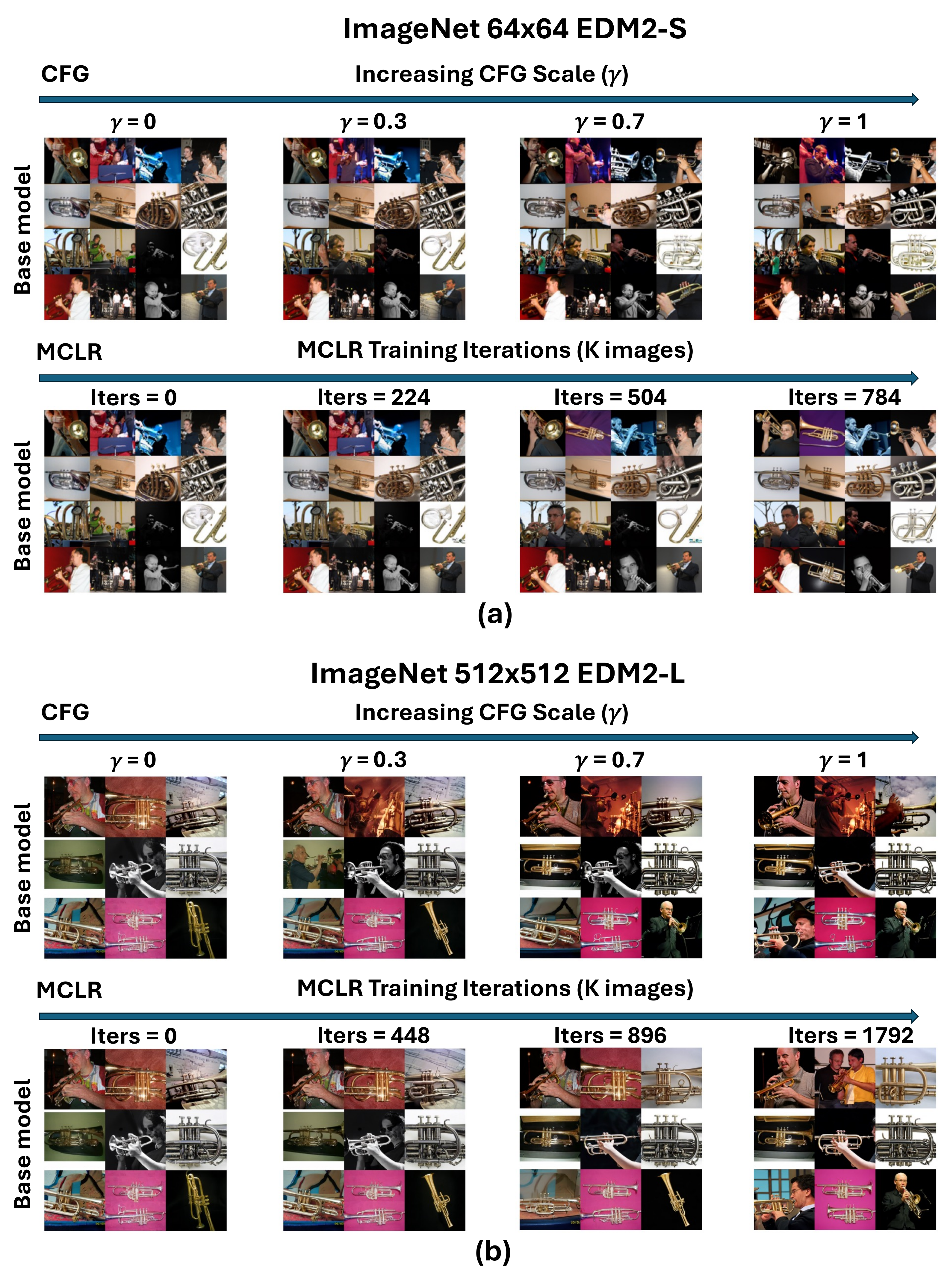}
    \caption{\textbf{Comparison between CFG and MCLR for Trumpet (Class 513).} (a,b) demonstrate the progressive evolution of generated samples on ImageNet-64×64 and ImageNet-512×512, respectively. Increasing the CFG scale $\gamma$ (top rows) and progressive MCLR training (bottom rows) produce similar effects, both enhancing class-specific structures in the generated images.}
    \label{fig:class_513_qua}
\end{figure*}

\begin{figure*}[t]
    \centering
    \includegraphics[width=0.8\linewidth]{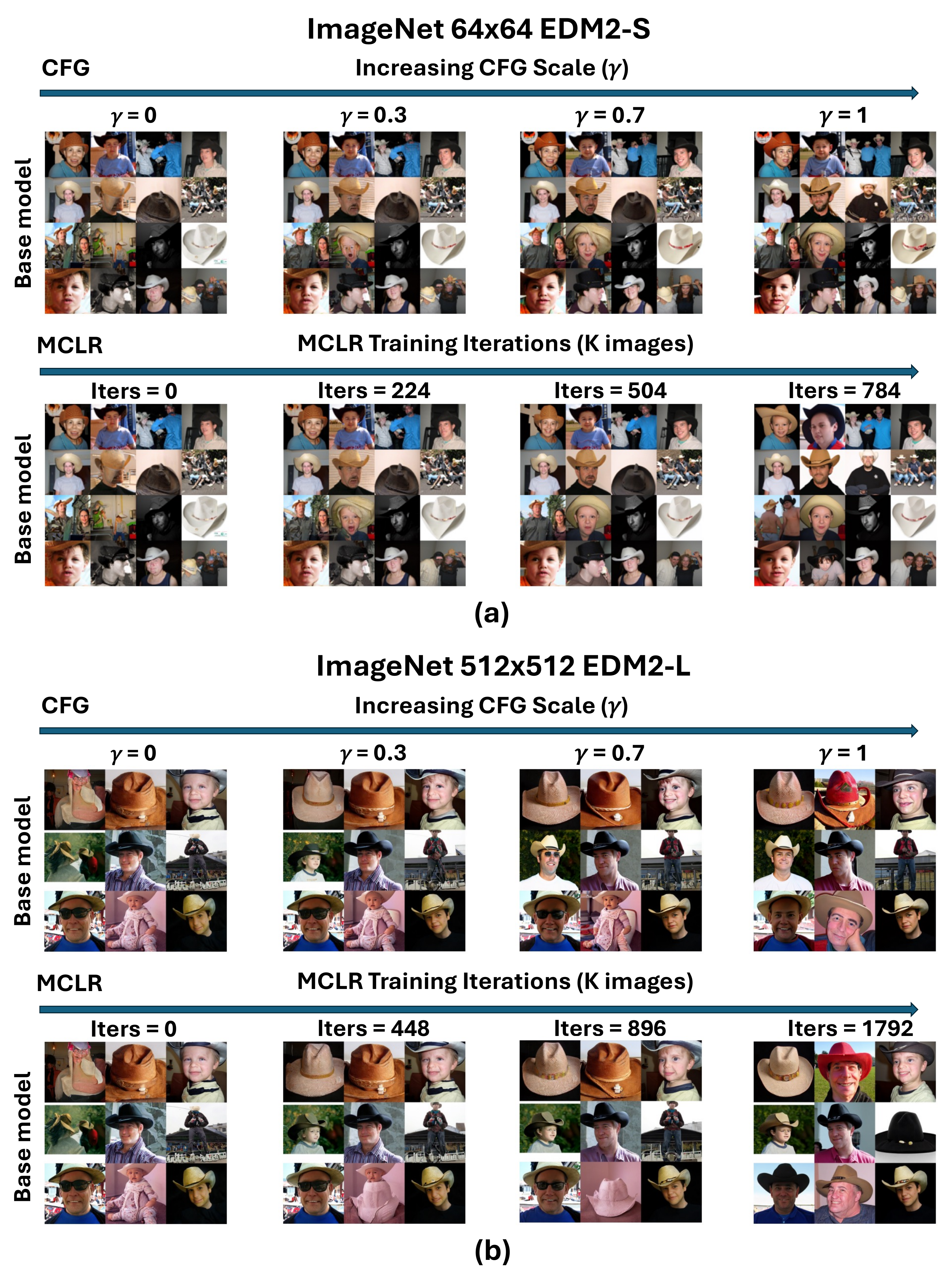}
\caption{\textbf{Comparison between CFG and MCLR for Cowboy Hat (Class 515).} (a,b) demonstrate the progressive evolution of generated samples on ImageNet-64×64 and ImageNet-512×512, respectively. Increasing the CFG scale $\gamma$ (top rows) and progressive MCLR training (bottom rows) produce similar effects, both enhancing class-specific structures in the generated images.}    \label{fig:class_515_qua}
\end{figure*}

\begin{figure*}[t]
    \centering
    \includegraphics[width=0.8\linewidth]{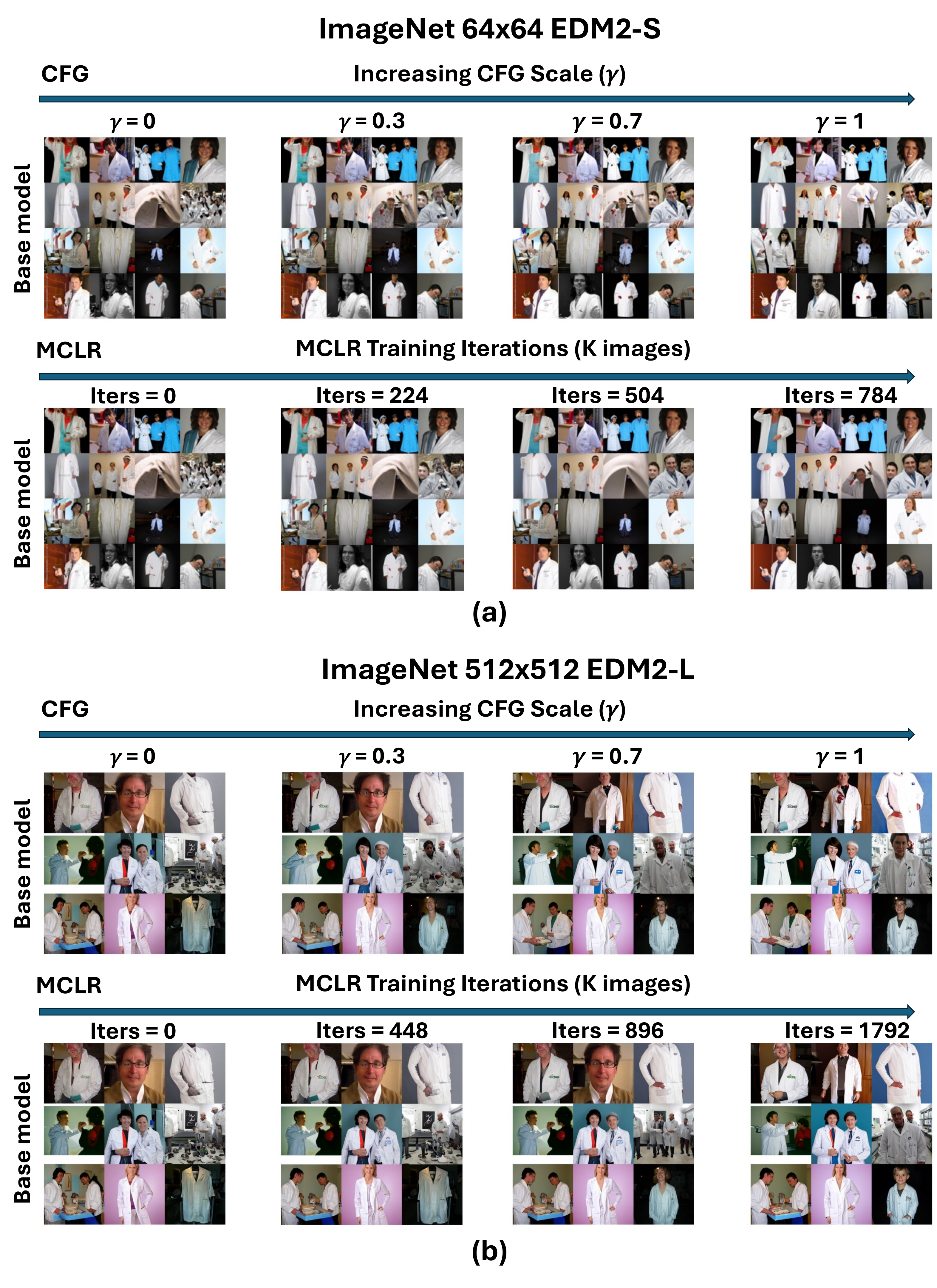}
\caption{\textbf{Comparison between CFG and MCLR for Lab Coat (Class 617).} (a,b) demonstrate the progressive evolution of generated samples on ImageNet-64×64 and ImageNet-512×512, respectively. Increasing the CFG scale $\gamma$ (top rows) and progressive MCLR training (bottom rows) produce similar effects, both enhancing class-specific structures in the generated images.}     
\label{fig:class_617_qua}
\end{figure*}

\begin{figure*}[t]
    \centering
    \includegraphics[width=1\linewidth]{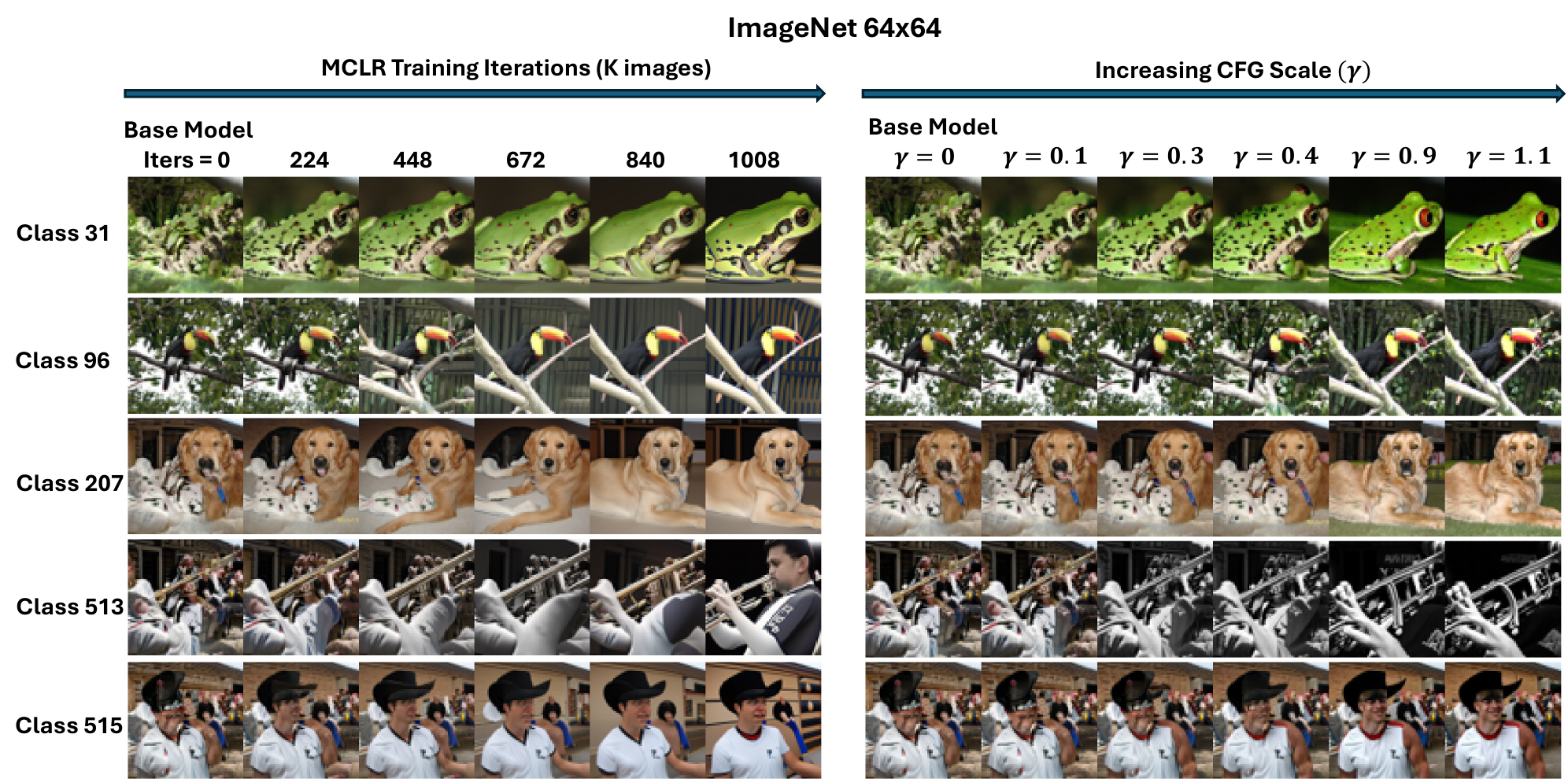}
\caption{\textbf{Comparison between CFG and MCLR.} Left and right figures demonstrate the progressive evolution of generated samples for MCLR and CFG on ImageNet-64×64, respectively, with all images initialized from the same random noise. Increasing the CFG scale $\gamma$ and progressive MCLR training produce similar effects, both enhancing class-specific structures in the generated images.}     
\label{fig:additional_64_qua_1}
\end{figure*}

\begin{figure*}[t]
    \centering
    \includegraphics[width=1\linewidth]{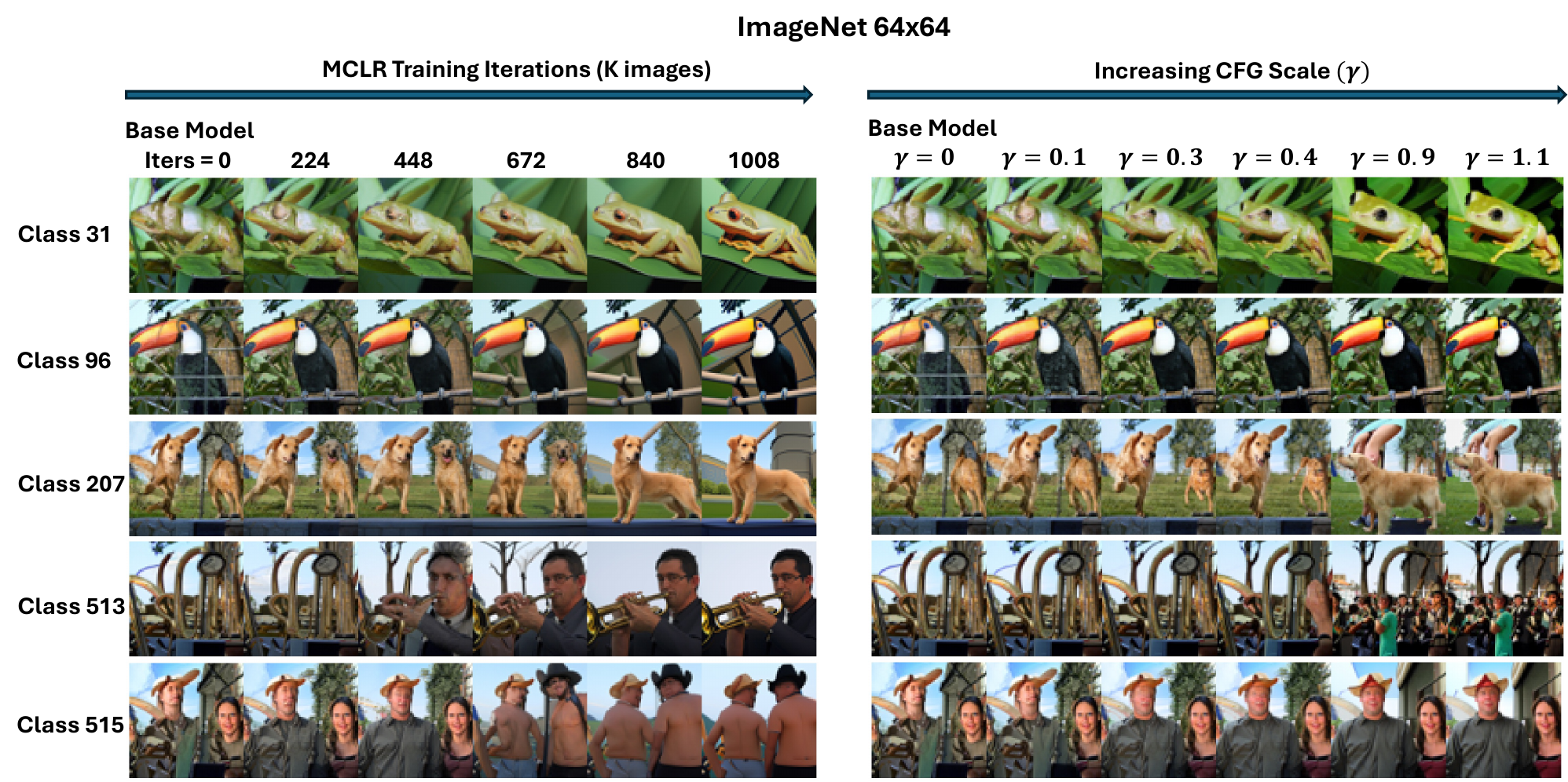}
\caption{\textbf{Comparison between CFG and MCLR.} Same as~\Cref{fig:additional_64_qua_1}, but with a different initial random noise.}     
\label{fig:additional_64_qua_2}
\end{figure*}

\begin{figure*}[t]
    \centering
    \includegraphics[width=1\linewidth]{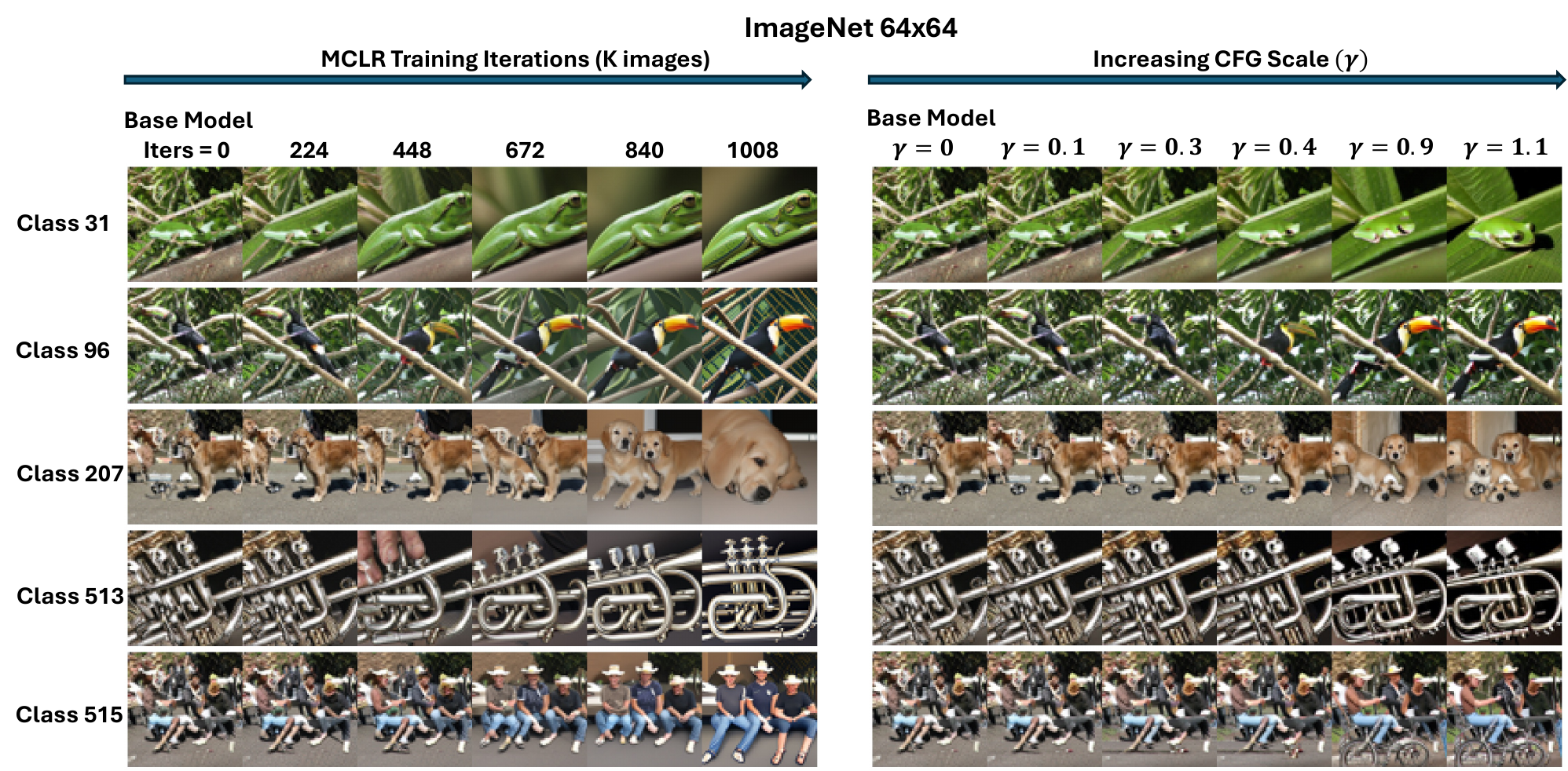}
\caption{\textbf{Comparison between CFG and MCLR.} Same as~\Cref{fig:additional_64_qua_1}, but with a different initial random noise.}    
\label{fig:additional_64_qua_3}
\end{figure*}

\begin{figure*}[t]
    \centering
    \includegraphics[width=1\linewidth]{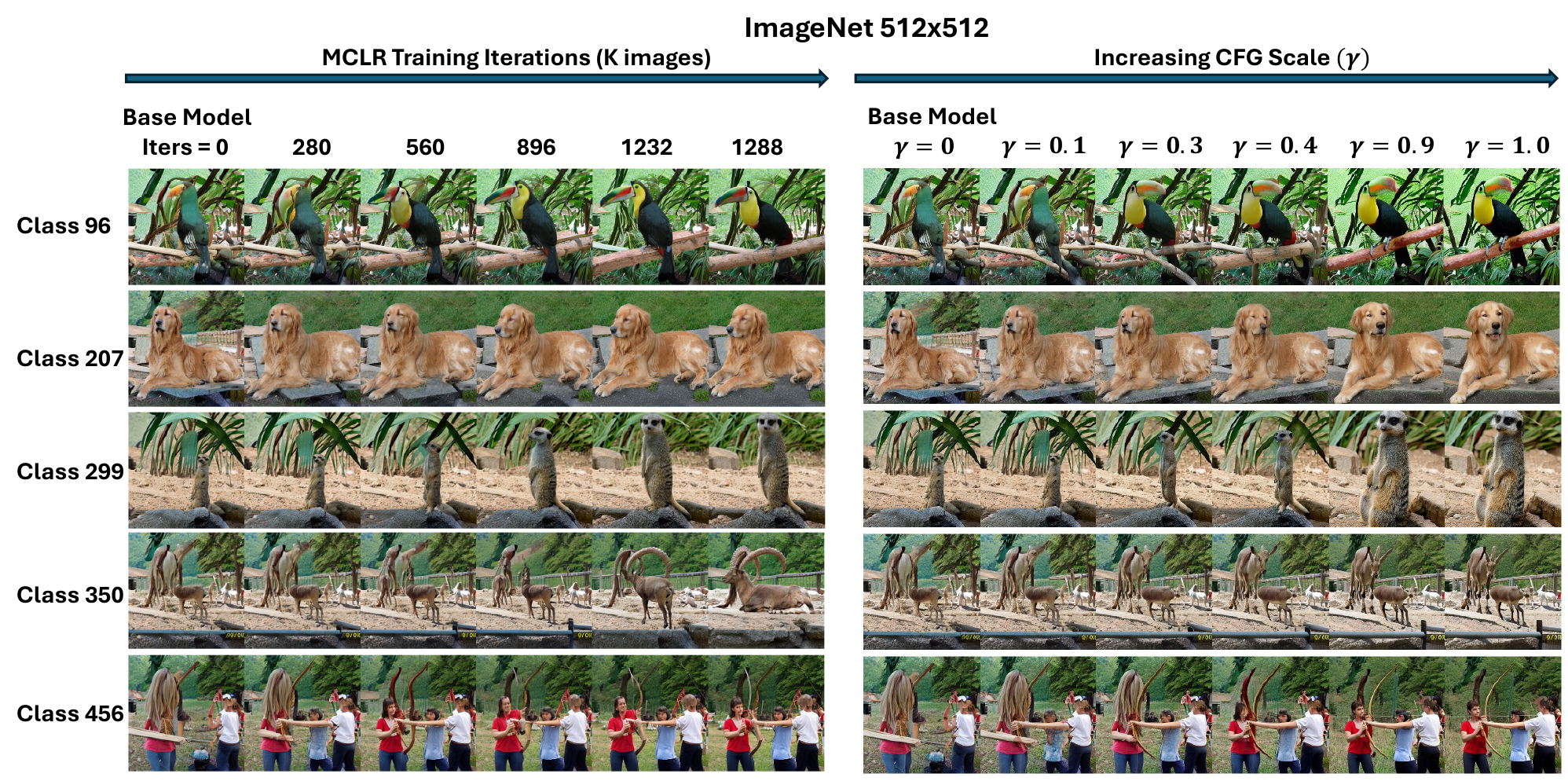}
\caption{\textbf{Comparison between CFG and MCLR.} Same as~\Cref{fig:additional_64_qua_1}, but for ImageNet-512x512 with a different initial random noise.}    
\label{fig:additional_512_qua_1}
\end{figure*}

\begin{figure*}[t]
    \centering
    \includegraphics[width=1\linewidth]{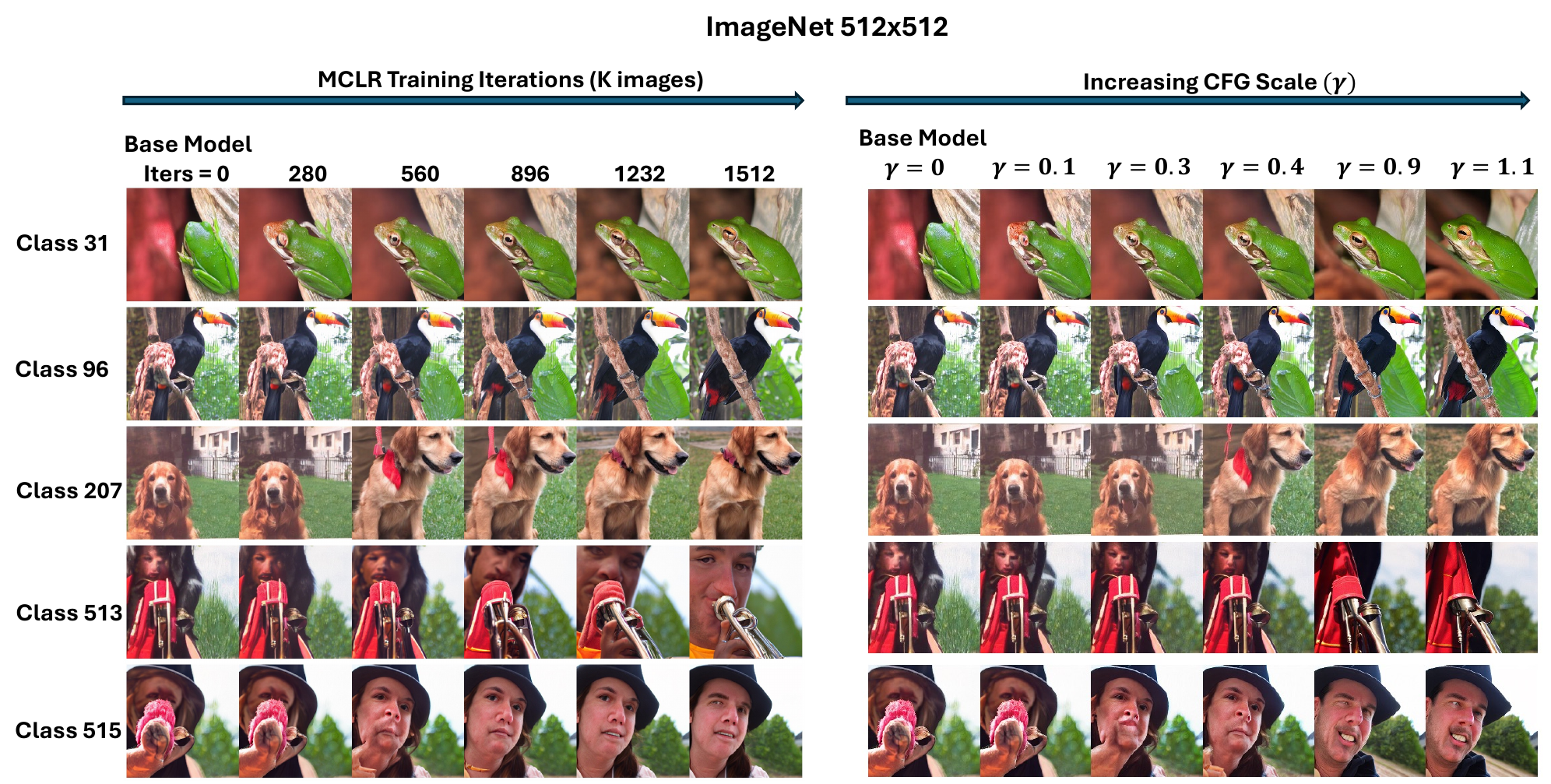}
\caption{\textbf{Comparison between CFG and MCLR.} Same as~\Cref{fig:additional_64_qua_1}, but for ImageNet-512x512 with a different initial random noise.}    
\label{fig:additional_512_qua_2}
\end{figure*}

\begin{figure*}[t]
    \centering
    \includegraphics[width=1\linewidth]{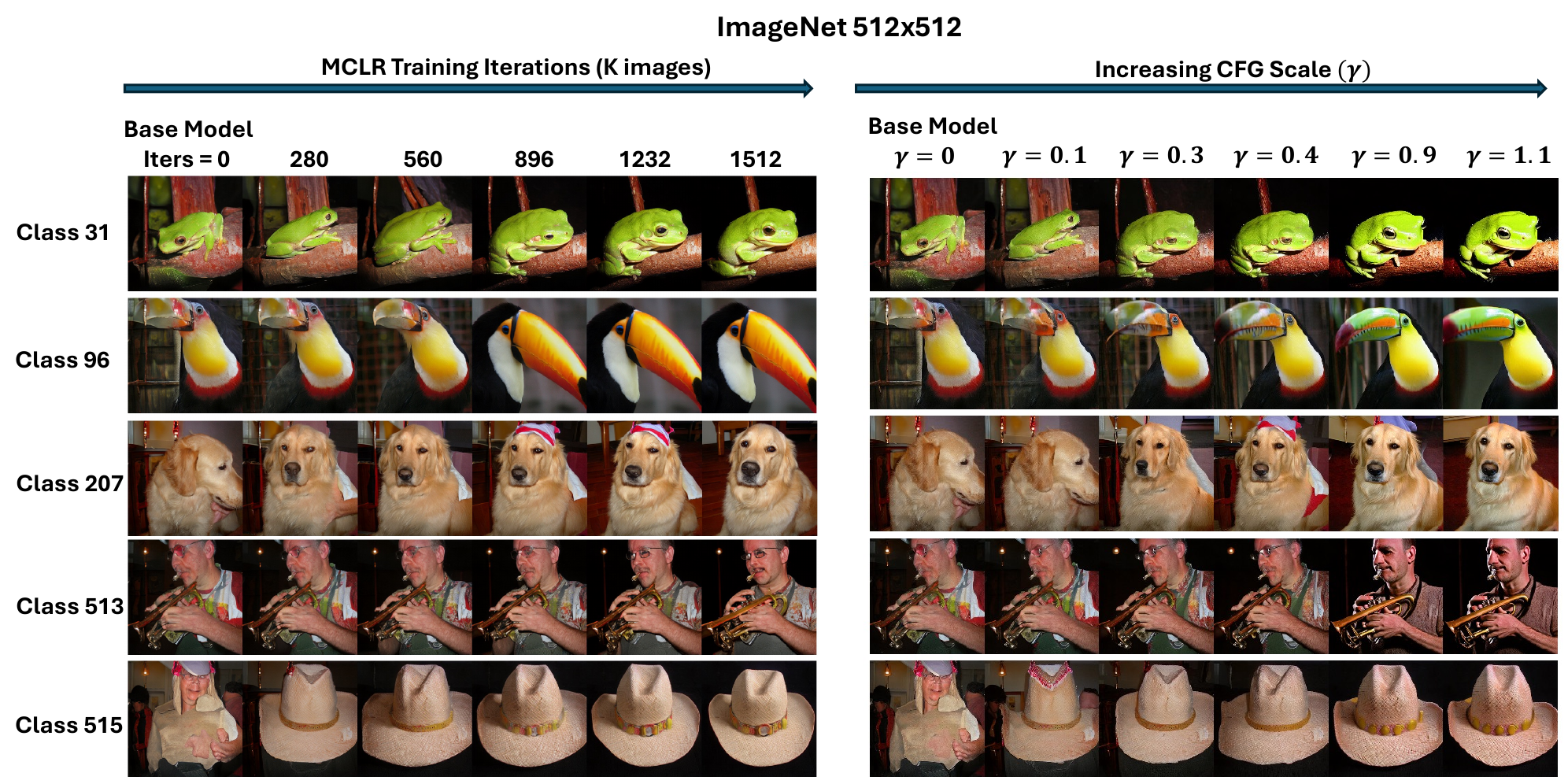}
\caption{\textbf{Comparison between CFG and MCLR.} Same as~\Cref{fig:additional_64_qua_1}, but for ImageNet-512x512 with a different initial random noise.}    
\label{fig:additional_512_qua_3}
\end{figure*}

\begin{figure*}[t]
    \centering
    \includegraphics[width=1\linewidth]{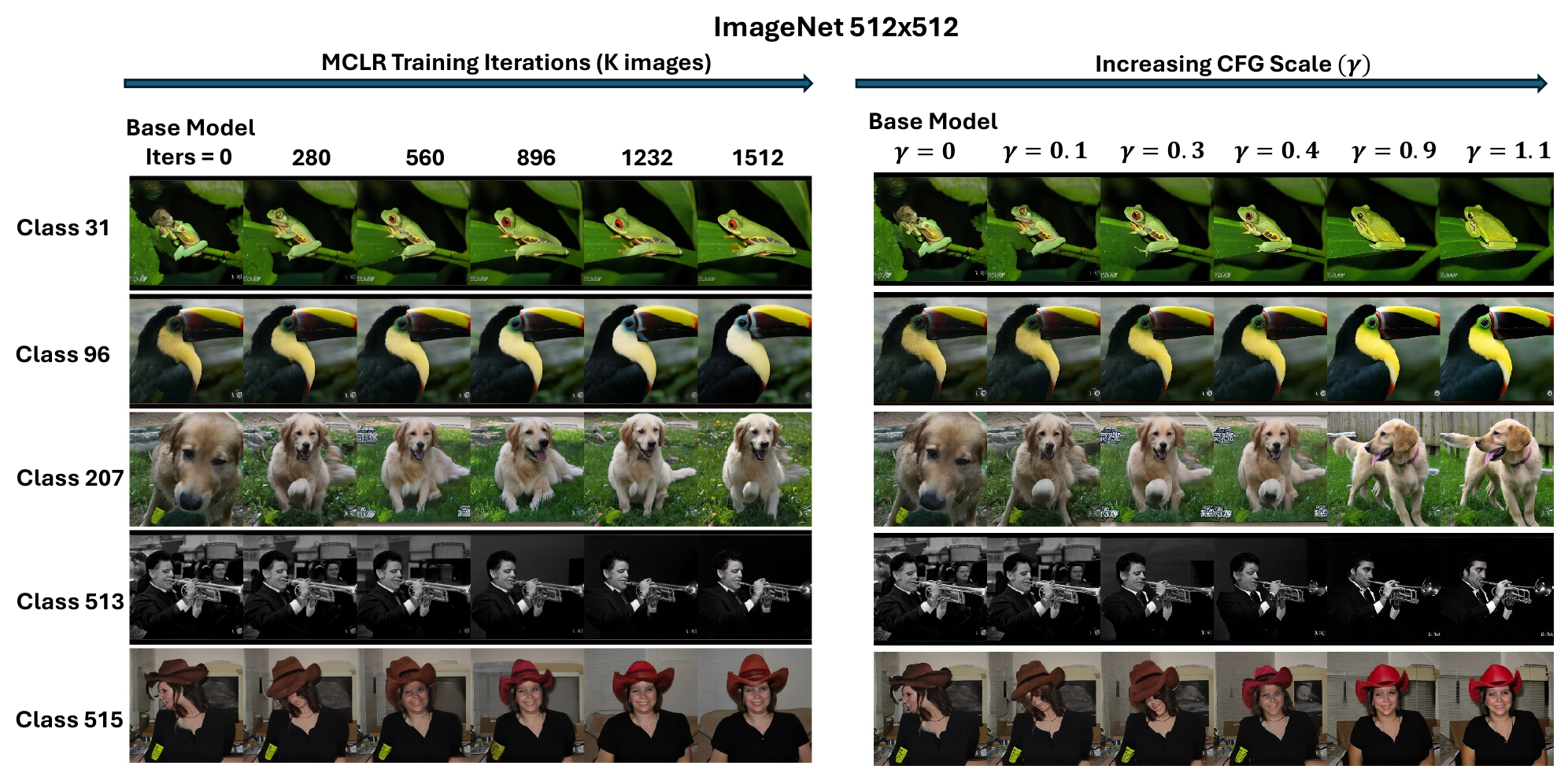}
\caption{\textbf{Comparison between CFG and MCLR.} Same as~\Cref{fig:additional_64_qua_1}, but for ImageNet-512x512 with a different initial random noise.}    
\label{fig:additional_512_qua_4}
\end{figure*}

\section{Ablation Study}
\label{ablation study appendix}
\begin{figure*}[t]
    \centering
    \includegraphics[width=0.9\linewidth]{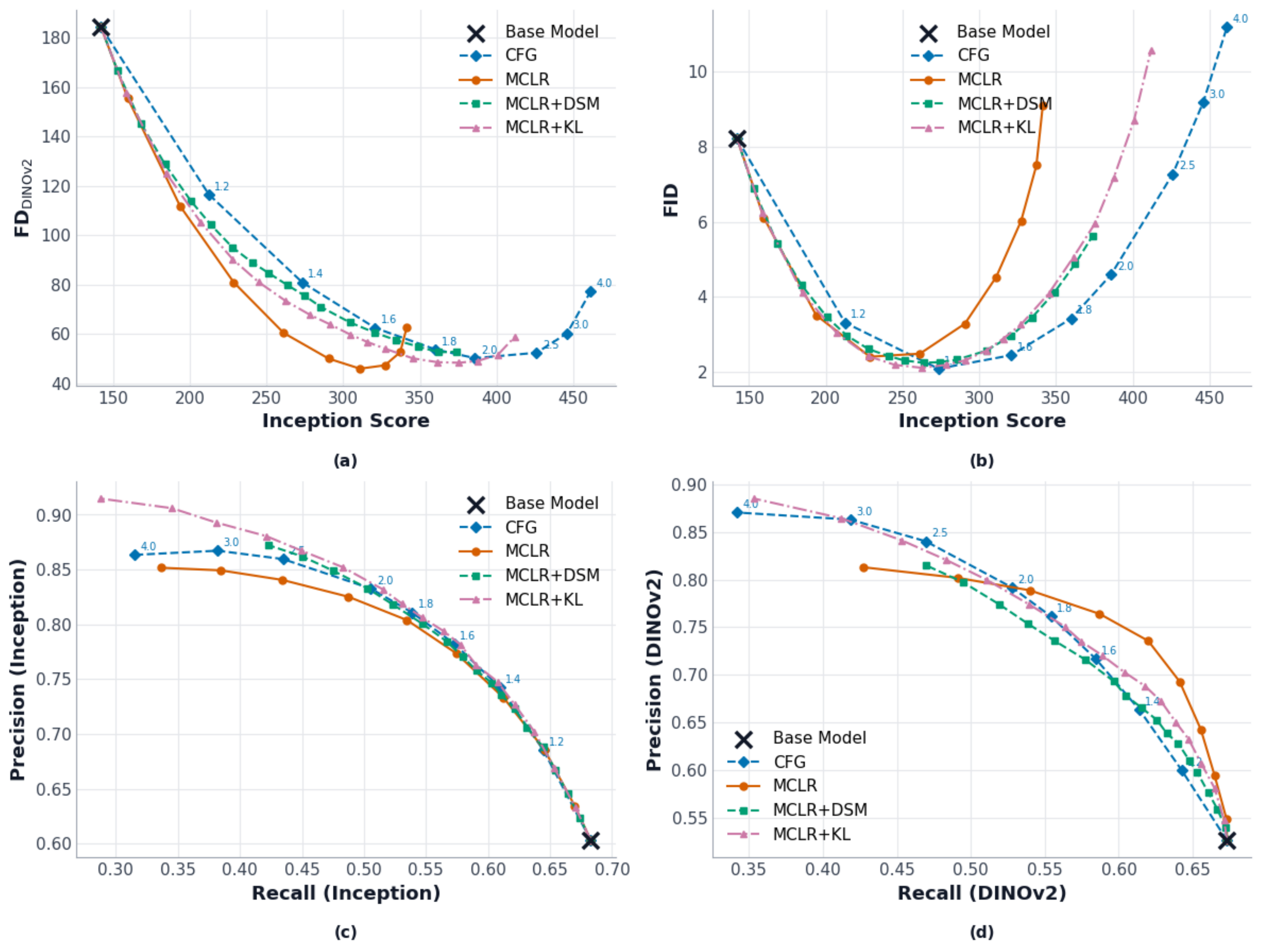}%
    \caption{\textbf{Quantitative Results for SiT-XL/2 Trained on ImageNet-256$\times$256.} (a), (b), (c), (d) show the FD--IS trade-offs and the Precision--Recall trade-offs calculated with Inception and \text{DINOv2} features respectively.
We evaluate classifier-free guidance (CFG) scales $\gamma \in \{0.1, 0.2, 0.3, 0.4, 0.5, 0.7, 0.9, 1, 1.5, 2.0, 3.0\}$. For the DSM- and KL-regularized variants, we set $\beta=1$.}
    \label{fig:SiT DSM KL ablation}
\end{figure*}

\begin{figure*}[t]
    \centering
    \includegraphics[width=1\linewidth]{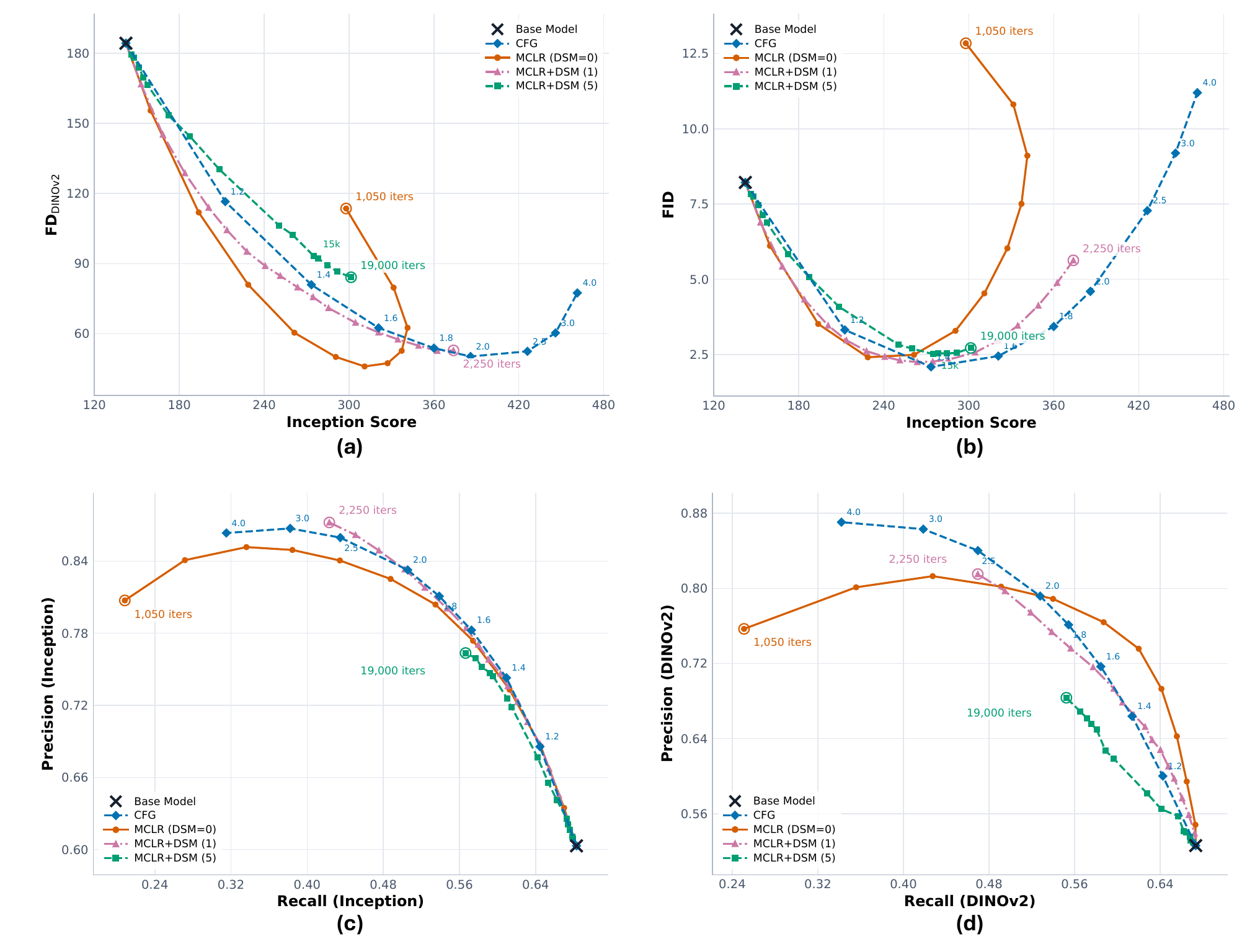}%
    \caption{\textbf{Quantitative Results for SiT-XL/2 with DSM Regularization Trained on ImageNet-256$\times$256.} (a), (b), (c), (d) show the FD--IS trade-offs and the Precision--Recall trade-offs calculated with Inception and \text{DINOv2} features respectively. We compare MCLR alone ($\beta=0$) with DSM-regularized variants using $\beta\in\{1,5\}$, and the annotations indicate the corresponding numbers of fine-tuning iterations.
We evaluate classifier-free guidance (CFG) scales $\gamma \in \{0.1, 0.2, 0.3, 0.4, 0.5, 0.7, 0.9, 1, 1.5, 2.0, 3.0\}$.}
    \label{fig:SiT DSM strength ablation}
\end{figure*}

\begin{figure*}
    \centering
    \includegraphics[width=1\linewidth]{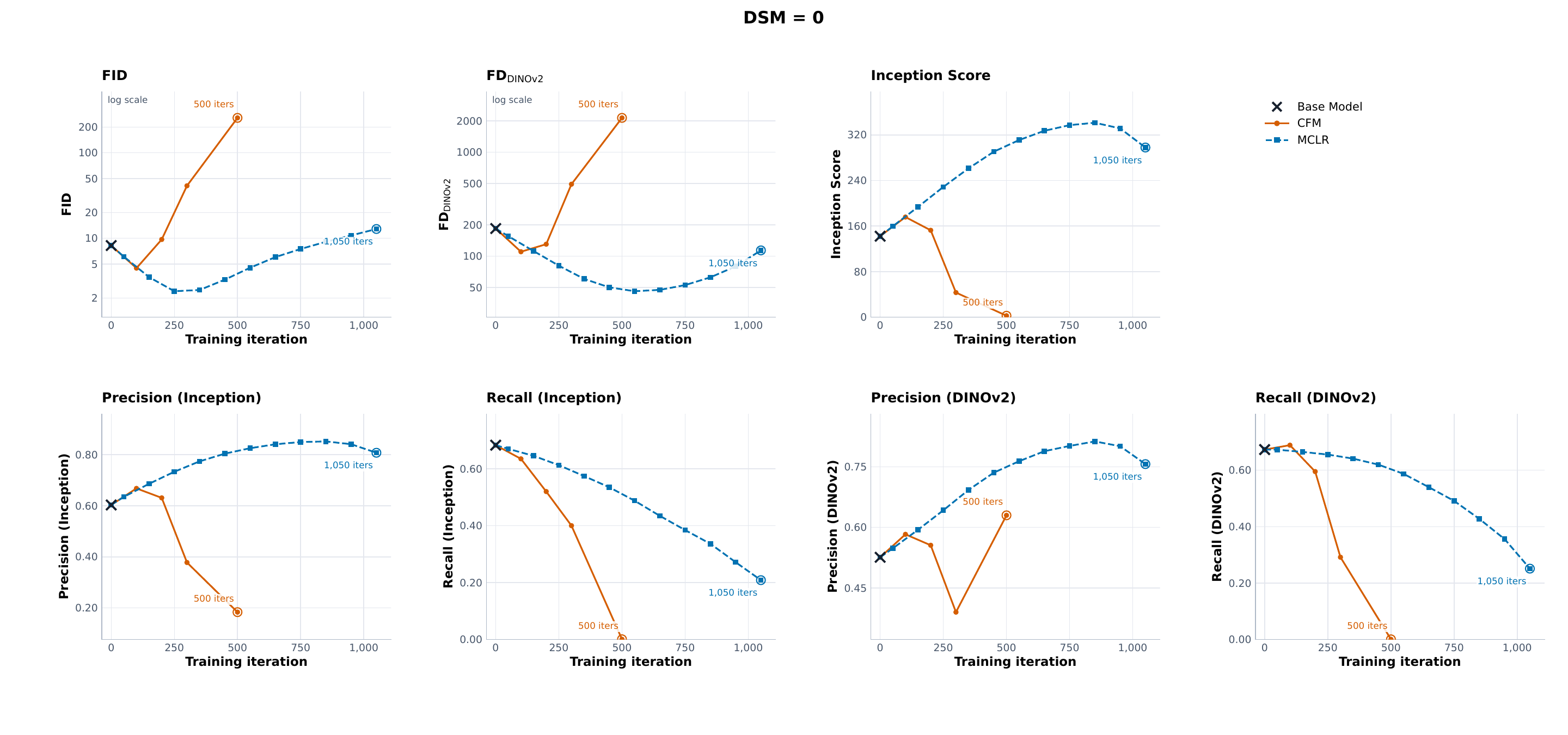}
\caption{\textbf{Training Dynamics of MCLR and CFM without DSM Regularization ($\beta=0$).} The panels report FID, FD\textsubscript{DINOv2}, Inception Score, and Inception- and DINOv2-based precision and recall over fine-tuning iterations. The black cross denotes the base model. MCLR produces larger and more sustained improvements, whereas CFM rapidly degrades after an initial improvement.}
\label{fig:mclr vs cfm reg=0}
\end{figure*}

\begin{figure*}
    \centering
    \includegraphics[width=1\linewidth]{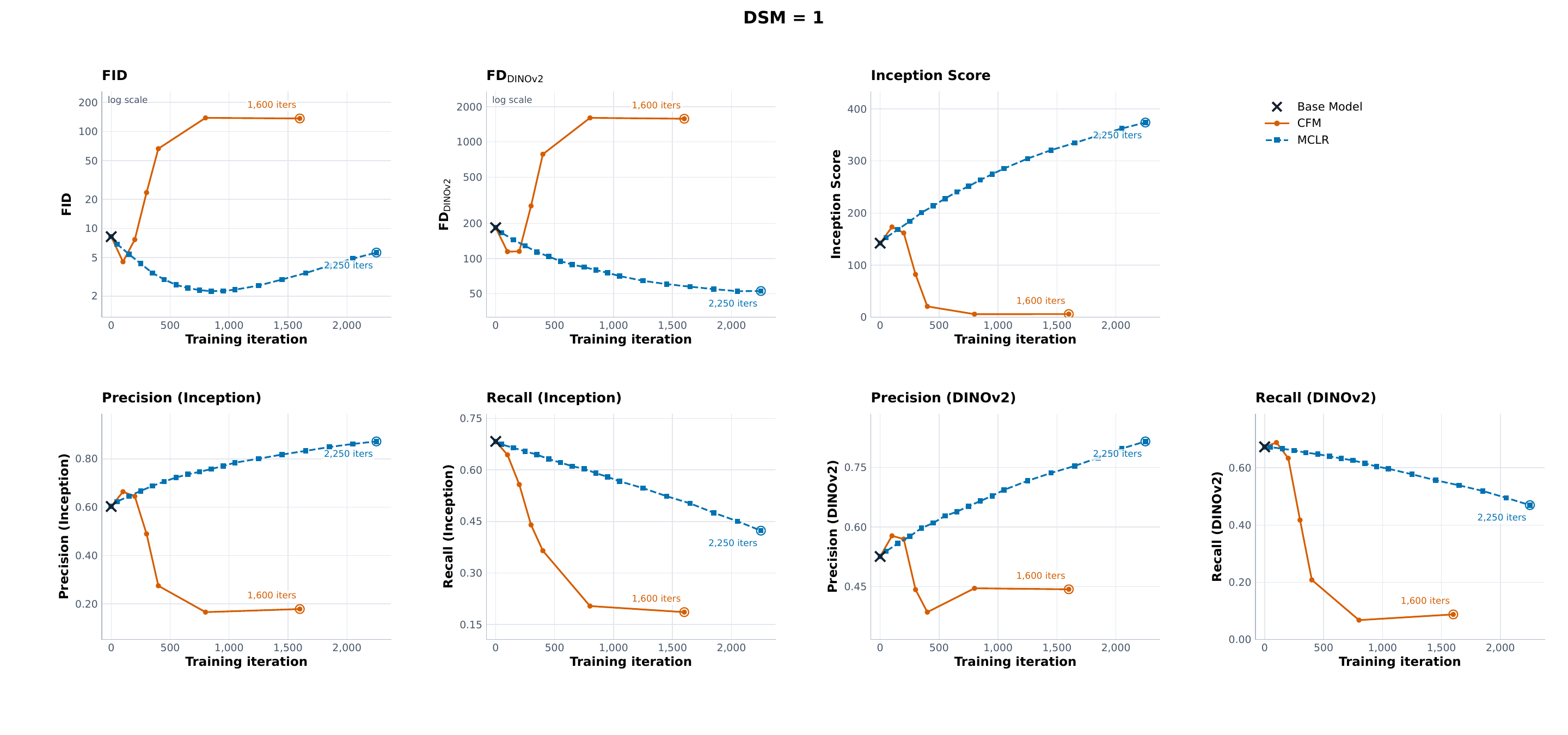}
\caption{\textbf{Training Dynamics of MCLR and CFM with DSM Regularization ($\beta=1$).} The panels report FID, FD\textsubscript{DINOv2}, Inception Score, and Inception- and DINOv2-based precision and recall over fine-tuning iterations. The black cross denotes the base model. MCLR achieves substantially better and more stable performance than CFM.}
\label{fig:mclr vs cfm reg=1}
\end{figure*}

\begin{figure*}
    \centering
    \includegraphics[width=1\linewidth]{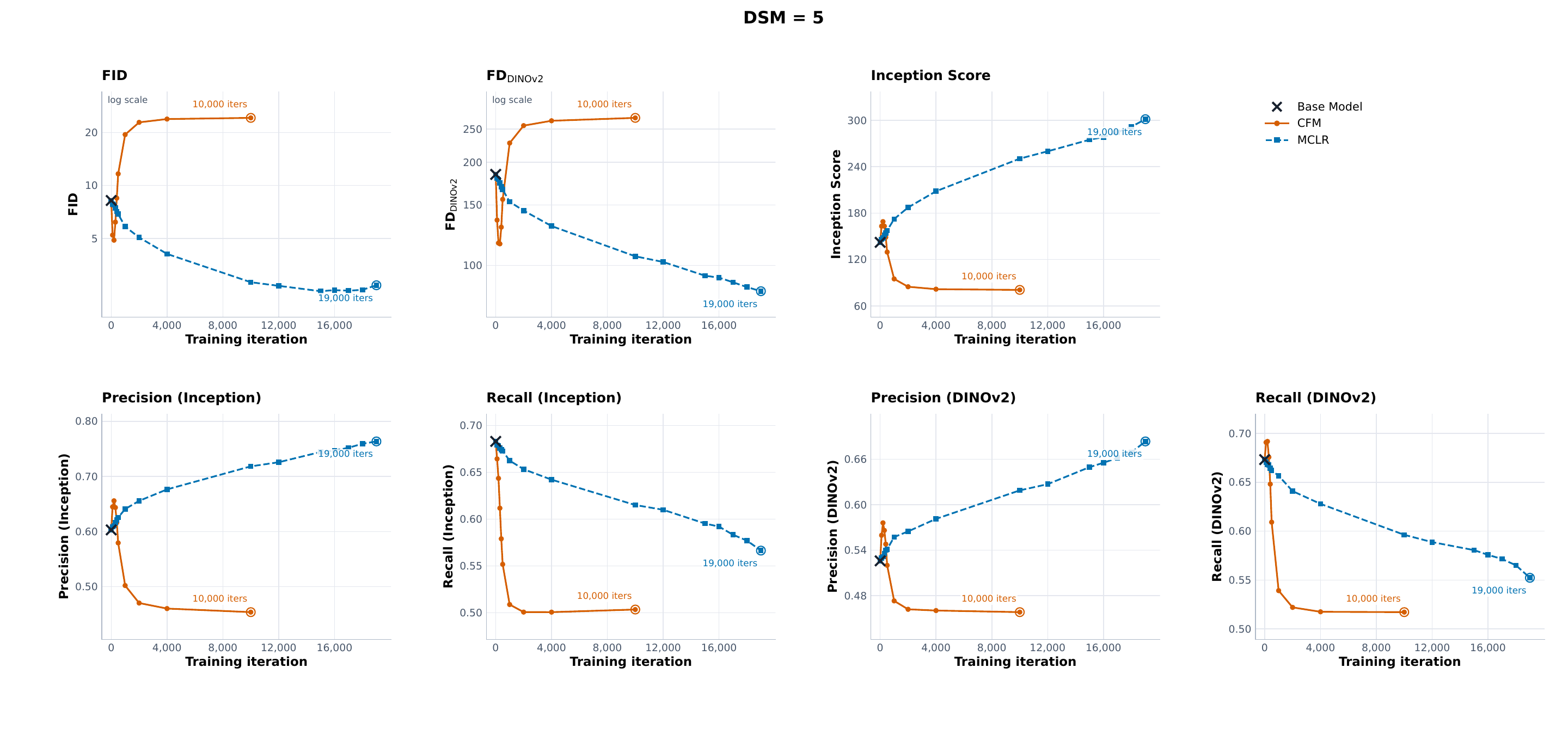}
\caption{\textbf{Training Dynamics of MCLR and CFM with DSM Regularization ($\beta=5$).} The panels report FID, FD\textsubscript{DINOv2}, Inception Score, and Inception- and DINOv2-based precision and recall over fine-tuning iterations. The black cross denotes the base model. MCLR continues to improve over a substantially longer training horizon, whereas CFM eventually deteriorates.}
\label{fig:mclr vs cfm reg=5}
\end{figure*}

\begin{figure*}[t]
    \centering
    \includegraphics[width=0.9\linewidth]{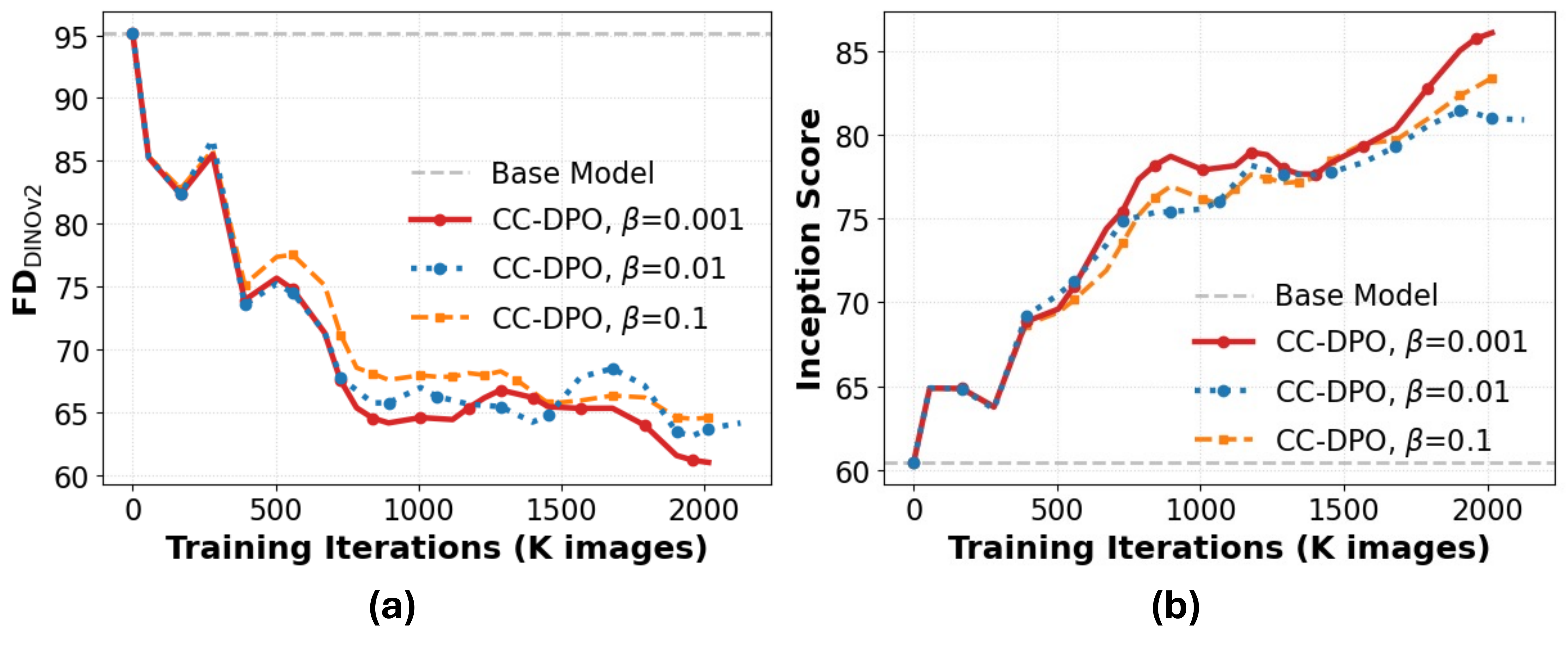}
\caption{\textbf{Effects of $\beta$ in CC-DPO.} (a,b) compare fine-tuning diffusion models using CC-DPO with different $\beta$ evaluated on the EDM2-S model trained on ImageNet-$64\times64$. Performance is reported in terms of FD\textsubscript{DINOv2} and Inception Score. CC-DPO achieves stable performance under different $\beta$, with a smaller $\beta$ leading to marginally better performance on selected metrics.}
\label{fig:ablation_DPO_beta}
\end{figure*}

\begin{figure*}[t]
    \centering
    \includegraphics[width=0.9\linewidth]{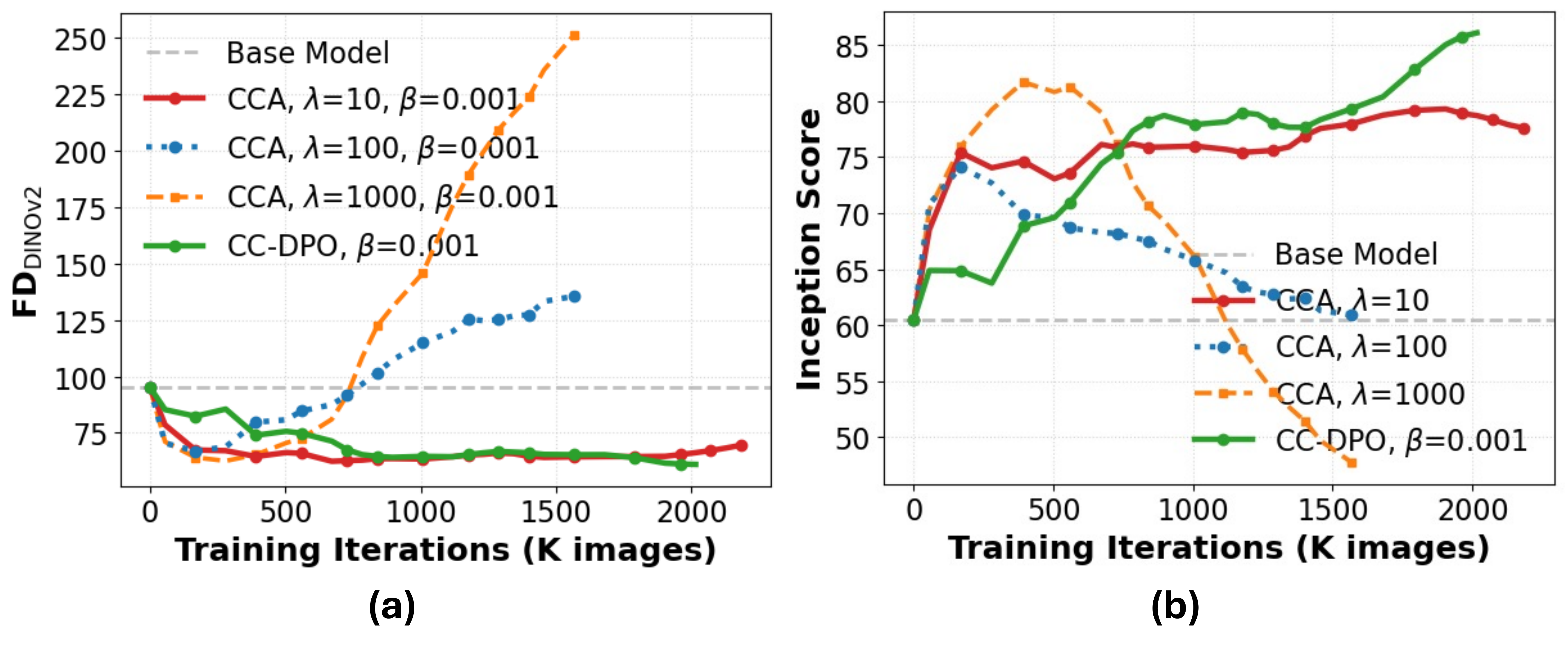}
\caption{\textbf{Effects of $\lambda$ in CCA.} (a,b) compare fine-tuning diffusion models using CCA with different $\lambda$ evaluated on the EDM2-S model trained on ImageNet-$64\times64$. Performance is reported in terms of FD\textsubscript{DINOv2} and Inception Score. Note with the same $\beta$, CC-DPO achieves comparable performance as CCA while at the same time requiring fewer hyperparameters.}
\label{fig:ablation_cca_lambda}
\end{figure*}

\subsection{Effects of DSM and KL Regularization}
\label{effect of DSM Regularization}
In this section we study the effects of combining MCLR with either (i) the standard DSM objective or (ii) a KL-style regularizer. Specifically, we fine-tune a REPA-enhanced SiT-XL/2 base model $\mathcal{D}_\text{ref}(\cdot)$ pretrained on ImageNet-$256\times256$ using the following two objectives:
\begin{equation}
\label{DSM regularized MCLR in appendix !!}
\begin{aligned}
&\beta \; \underbrace{\mathbb{E}_{\substack{\boldsymbol{c},\mb\epsilon\sim\mathcal{N}(\mb 0,\mb I)\,\\
                                \sigma\sim p(\sigma),
                               p(\boldsymbol{x}|\boldsymbol{c})}}
   \Bigl[
    w(\sigma)\bigl(
       \|\mb x-\mathcal{D}_{\mb\theta}(\mb x+\sigma\mb\epsilon;\sigma,\mb c)\|_2^2\bigr)
   \Bigr]}_{\text{DSM Objective}} \\
   &+ \underbrace{\mathbb{E}_{\substack{\boldsymbol{c}, \tilde{\boldsymbol{c}},\mb\epsilon\sim\mathcal{N}(\mb 0,\mb I)\,\\
                                \sigma\sim p(\sigma),
                               p(\boldsymbol{x}|\boldsymbol{c}),p(\boldsymbol{y}|\tilde{\boldsymbol{c}})}}
   \Bigl[
     w(\sigma)\bigl(
       \|\mb x-\mathcal{D}_{\mb\theta}(\mb x+\sigma\mb\epsilon;\sigma,\mb c)\|_2^2
       \;-  \;
       \|\mb y-\mathcal{D}_{\mb\theta}(\mb y+\sigma\mb\epsilon;\sigma,\mb c)\|_2^2
     \bigr)
   \Bigr]}_{\text{MCLR Regularization}},
\end{aligned}
\end{equation}

\begin{equation}
\label{KL regularized MCLR appendix}
\begin{aligned}
&\beta \; \underbrace{\mathbb{E}_{\substack{\boldsymbol{c},\mb\epsilon\sim\mathcal{N}(\mb 0,\mb I)\,\\
                                \sigma\sim p(\sigma),
                               p(\boldsymbol{x}|\boldsymbol{c})}}
   \Bigl[
    w(\sigma)\bigl(
       \|\mathcal{D}_\text{ref}(\mb x+\sigma\mb\epsilon;\sigma,\mb c)-\mathcal{D}_{\mb\theta}(\mb x+\sigma\mb\epsilon;\sigma,\mb c)\|_2^2\bigr)
   \Bigr]}_{\text{KL Loss}} \\
   &+ \underbrace{\mathbb{E}_{\substack{\boldsymbol{c}, \tilde{\boldsymbol{c}},\mb\epsilon\sim\mathcal{N}(\mb 0,\mb I)\,\\
                                \sigma\sim p(\sigma),
                               p(\boldsymbol{x}|\boldsymbol{c}),p(\mb y|\tilde{\mb c})}}
   \Bigl[
     w(\sigma)\bigl(
       \|\mb x-\mathcal{D}_{\mb\theta}(\mb x+\sigma\mb\epsilon;\sigma,\mb c)\|_2^2
       \;-  \;
       \|\mb y-\mathcal{D}_{\mb\theta}(\mb y+\sigma\mb\epsilon;\sigma,\mb c)\|_2^2
     \bigr)
   \Bigr]}_{\text{MCLR Regularization}},
\end{aligned}
\end{equation}
where~\eqref{KL regularized MCLR appendix} is an approximation of the KL-regularized MCLR in~\eqref{KL+MCLR objective main text}. Although the exact KL objective in~\eqref{KL+MCLR objective main text} should be defined using samples generated by the reference model, we evaluate this term on real data for simplicity.

\Cref{fig:SiT DSM KL ablation} compares MCLR alone with its DSM- and KL-regularized variants, where $\beta=1$ is used for both regularized objectives. MCLR alone achieves the best-case FD\textsubscript{DINOv2} score, whereas adding either regularizer improves the FD--IS and Precision--Recall trade-offs and yields a better best-case FID. These results suggest that the additional regularization mainly improves training stability and the overall quality--diversity trade-off, rather than the best achievable FD\textsubscript{DINOv2}.


\paragraph{Effect of Regularization Strength.} We further study the regularization strength $\beta$ for the DSM-regularized objective~\eqref{DSM regularized MCLR in appendix !!}, considering $\beta\in\{0, 1, 5\}$. As shown in~\Cref{fig:SiT DSM strength ablation}(a,b), MCLR without additional regularization ($\beta=0$) already substantially improves the reference model on both FD\textsubscript{DINOv2} and FID score, achieving the best FD\textsubscript{DINOv2} score. Adding moderate DSM regularization ($\beta=1$) further improves the best-case FID score and the Precision--Recall trade-off as shown in~\Cref{fig:SiT DSM strength ablation}(b,c). However, an excessively large weight ($\beta=5$) slows convergence and limits the gains from MCLR.

Smaller values of $\beta$ generally lead to faster convergence. The best performance is reached within fewer than $1{,}000$ iterations for $\beta=0$, compared with more than $2{,}000$ iterations for $\beta=1$ and $19{,}000$ iterations for $\beta=5$. Without DSM regularization, however, prolonged fine-tuning can lead to performance degradation; in our experiments, this behavior can be mitigated through early stopping. 

\subsection{Empirical Comparison with Contrastive Flow Matching}
\label{empirical comparison between MCLR and CFM}
We provide an empirical comparison between MCLR and Contrastive Flow Matching (CFM), discussed in~\Cref{sec:comparison-with-cfm}. Both methods augment the same standard training objective with a contrastive term, while differing in the construction of that term. This structural similarity enables a controlled comparison in which we keep the base model, learning rate, and all other training settings fixed and change only the contrastive objective. We set the contrastive weight to $1$ for both methods and vary the weight of the standard DSM/flow matching term over $\beta\in\{0,1,5\}$. As shown in~\Cref{fig:mclr vs cfm reg=0,fig:mclr vs cfm reg=1,fig:mclr vs cfm reg=5}, both methods initially improve FID and FD\textsubscript{DINOv2} across all three values of $\beta$. However, MCLR consistently yields substantially larger and more sustained improvements, whereas CFM improves only briefly before its performance rapidly deteriorates. 

Importantly, although CFM is presented as a method for accelerating diffusion training~\cite{stoica2025contrastive}, this interpretation should be treated with caution. Its early improvements in fidelity metrics are accompanied by a substantial decrease in recall, and therefore do not necessarily indicate faster learning of the target distribution. Since inference-time CFG can similarly improve fidelity at the cost of diversity, establishing genuine training acceleration would require comparisons at matched precision and recall. Therefore, improvements in FID alone are insufficient to fully substantiate the training-acceleration claim. We therefore advocate evaluating claims of training acceleration at matched fidelity--diversity operating points, rather than relying on FID alone.

\subsection{Hyperparameter Sensitivity of CC-DPO}
The hyperparameter $\beta$ in the CC-DPO objective~\eqref{DPO objective-conditional} controls the strength of the KL regularization between the base model and the fine-tuned model. We study its effect on EDM2-S trained on ImageNet-$64\times64$. As shown in~\Cref{fig:ablation_DPO_beta}, CC-DPO exhibits relatively stable performance across a wide range of $\beta$ values. In particular, smaller $\beta$ values achieve marginally better best-case FD\textsubscript{DINOv2} and Inception Score.

\subsection{Hyperparameter Sensitivity of CCA}
\label{subsec:hyperparameter for CCA}
The hyperparameter $\lambda$ in the CCA objective~\eqref{CCA objective, lower} controls the strength of the contrastive term, with larger $\lambda$ placing greater emphasis on penalizing the non-preferred sample $\mb x_l$. We study its effect on EDM2-S trained on ImageNet-$64\times64$, fixing $\beta=0.001$. As shown in~\Cref{fig:ablation_cca_lambda}, $\lambda$ has little impact on the best achievable FD\textsubscript{DINOv2}; however, smaller $\lambda$ values lead to more stable and smoother convergence. In contrast, larger $\lambda$ values can yield slightly higher best-case Inception Scores, but often cause training instability and eventual collapse when training is prolonged.

In practice, jointly tuning $\beta$ and $\lambda$ is challenging. Although CCA and CC-DPO share the same theoretical optimal solution, CC-DPO requires tuning only a single hyperparameter $\beta$. Empirically, CCA does not outperform CC-DPO under careful tuning as shown in~\Cref{fig:ablation_cca_lambda}. Therefore, we recommend CC-DPO over CCA in practice.


\section{Discussion on Related Works}
\label{More related works}
\subsection{Visual Generation without Guidance}
Our work contributes to a growing line of research~\cite{chentoward,chenvisual,tang2025diffusion} that seeks to induce CFG-like effect by modifying the standard training objective, rather than applying CFG at inference time. The most closely related approach is CCA~\cite{chentoward}, which aims to learn the gamma-powered distribution~\eqref{optimal CC-DPO solution} in autoregressive models using \textbf{Noise Contrastive Estimation (NCE)}~\cite{gutmann2010noise}. While CCA also considers DPO as a baseline and reports inferior performance, our theoretical analysis demonstrates that CC-DPO and CCA admit the same optimal solution, and are therefore fundamentally equivalent at the population level. With a correct implementation, we find that CC-DPO consistently matches or exceeds the empirical performance of CCA while requiring fewer hyperparameters and simpler optimization. We further extend CCA to diffusion models by approximating log-likelihoods via the ELBO and include this variant as a baseline in our experiments. A detailed theoretical and empirical analysis of CCA is provided in~\Cref{sec: theoretical analysis on CCA}. Another relevant direction is \textbf{Guidance-Free Training (GFT)}~\cite{chenvisual, tang2025diffusion}, which aims to directly train diffusion models to reproduce CFG-induced score functions. GFT relies on a reparameterized score function in the same functional form as CFG, hence inherits the same theoretical ambiguities associated with CFG itself. In contrast, the proposed MCLR objective provides a clearer mechanistic interpretation of CFG.

Besides these approaches, several works sought to enhance conditional generative modeling through class-wise contrastive mechanisms.~\cite{yan2024training} introduce a contrastive objective to improve the modeling of tail classes;~\cite{lee2025aligning} employ an InfoNCE-style loss~\cite{oord2018representation} to strengthen text-image alignment in diffusion models;~\cite{kadkhodaie2024feature} and~\cite{yun2025no} regularize diffusion models by encouraging class separation in feature space. In contrast to MCLR, CC-DPO and CCA, these approaches are primarily heuristic and do not provide a theoretical characterization of the conditional distribution induced by their objectives.

Lastly, beyond class-wise contrast, another complementary strategy for improving generative modeling is to contrast high-quality (“real”) data against low-quality (“synthetic”) samples. Zheng et al.~\cite{zheng2025direct} propose Direct Discriminative Optimization (DDO), which employs the same contrastive objective as CCA, but replaces class-conditioned contrast with a real–synthetic discrimination signal. We include DDO as one of our baselines.

\subsection{Inference-Time Alignment via Guidance}
Recent works~\cite{frans2025diffusion,jin2025inference,cheng2025diffusion,jiang2026rethinking} observe that CFG-style inference-time guidance can produce effects resembling those of training-time alignment methods, including reinforcement learning-based approaches. These studies provide empirical evidence that guidance may implicitly induce alignment-like behavior. 

However, existing analyses are largely heuristic and often rely on unrealistic assumptions such as the guided score corresponds to that of the gamma-powered distribution~\eqref{optimal CC-DPO solution}, an assumption that has been proven incorrect~\cite{karras2024guiding,bradley2024classifier}. As a result, a rigorous theoretical connection between CFG and alignment objectives remains incomplete.

By establishing the equivalence between CFG and a weighted MCLR objective, our work provides a formal mechanistic interpretation of CFG as an inference-time alignment algorithm. To the best of our knowledge, this is the first result that rigorously connects CFG-style guidance with likelihood-ratio-based alignment objectives.
\section{Broader Impact}
\label{broader impact}
This work aims to improve the training of conditional diffusion models by reducing reliance on inference-time guidance mechanisms such as classifier-free guidance (CFG). By
incorporating contrastive signals directly into the training
objective, the proposed method enables standard sampling
procedures to achieve CFG-like effects, potentially reducing computational overhead and simplifying deployment in
downstream applications. These improvements may benefit
a wide range of generative modeling tasks, including image
synthesis, data augmentation, and representation learning.
At the same time, the proposed approach does not introduce
new application domains or capabilities beyond those of
existing diffusion models, and thus inherits similar ethical
considerations and risks, such as the potential misuse of generated content. As with prior generative models, responsible
use, appropriate dataset curation, and adherence to existing
guidelines remain essential. Overall, this work contributes
to a better theoretical and practical understanding of conditional diffusion models, with the goal of improving their
efficiency and interpretability

\end{document}